\PassOptionsToPackage{numbers,compress}{natbib}

\documentclass{article}
\usepackage[main, final]{neurips_2026}
\usepackage[utf8]{inputenc}
\usepackage[T1]{fontenc}  
\usepackage{url}      
\usepackage{booktabs}     
\usepackage{amsfonts}      
\usepackage{nicefrac}   
\usepackage{microtype}  
\usepackage{minitoc}
\usepackage{xcolor}     
\usepackage[pagebackref=true,breaklinks=true,colorlinks=true,bookmarks=false,hyperfootnotes=false]{hyperref}
\definecolor{deepred}{HTML}{940000}
\definecolor{iphoneGreen}{HTML}{16703C}
\hypersetup{
  colorlinks=true,
  linkcolor=iphoneGreen
}
\hypersetup{urlcolor  = [rgb]{0.4,0.15,0.95}}
\hypersetup{citecolor=[rgb]{0.4,0.15,0.95}}

\usepackage{amsmath,amssymb,amsthm,mathtools}
\usepackage{enumitem}
\usepackage{graphicx}
\usepackage{subcaption}
\usepackage{caption}
\usepackage{algorithm}
\usepackage{algpseudocode}
\usepackage{tcolorbox}
\usepackage{wrapfig}
\usepackage[nameinlink,noabbrev]{cleveref}

\newtheorem{theorem}{Theorem}
\newtheorem{lemma}{Lemma}
\newtheorem{corollary}{Corollary}
\newtheorem{proposition}{Proposition}

\theoremstyle{plain}
\newtheorem{assumption}{Assumption}

\theoremstyle{plain}
\newtheorem{remark}{Remark}

\crefname{assumption}{Assumption}{assumptions}
\Crefname{assumption}{Assumption}{Assumptions}
\crefname{fact}{fact}{facts}
\Crefname{fact}{Fact}{Facts}

\DeclareMathOperator{\Var}{Var}
\newcommand{\var}[2][]{\Var_{#1}\left[#2\right]}

\newcommand{\E}{\mathbb{E}}
\newcommand{\Rbb}{\mathbb{R}}

\DeclareMathOperator{\Tr}{Tr}
\newcommand{\op}{2}
\DeclareMathOperator{\expectation}{\mathbb{E}}
\newcommand{\e}[2][]{\expectation_{#1}\left[#2\right]}
\DeclareMathOperator{\Prob}{\mathbb{P}}
\newcommand{\prob}[1]{\Prob\left(#1\right)}
\DeclareMathOperator{\cO}{\mathcal{O}}
\newcommand{\norm}[1]{\left\|#1\right\|}
\newcommand{\fnorm}[1]{\left\|#1\right\|_F}
\newcommand{\snorm}[1]{\left\|#1\right\|_2}
\newcommand{\opnorm}[1]{\left\|#1\right\|_{2}}
\newcommand{\inner}[2]{\langle #1,#2 \rangle}
\newcommand{\set}[1]{\left\lbrace{#1}\right\rbrace}
\newcommand{\Kheadline}{50}

\def\thm1{\textcolor{orange}}

\title{Denoise First, Orthogonalize Later: Understanding Momentum in Muon via Spectral Filtering}

\author{%
  \textbf{Xianliang Li}$^{1,2,*}$ \quad
  \textbf{Zihan Zhang}$^{3,2,*}$ \quad
  \textbf{Weiyang Liu}$^{4}$ \quad
  \textbf{Han Bao}$^{1,2,5,6,\dagger}$ \\[4pt]
  {\small\normalfont
    $^{1}$The Institute of Statistical Mathematics
    \quad
    $^{2}$The Graduate Institute for Advanced Studies, SOKENDAI
  } \\[2pt]
  {\small\normalfont
    $^{3}$National Institute of Informatics
    \quad
    $^{4}$The Chinese University of Hong Kong
  } \\[2pt]
  {\small\normalfont
    $^{5}$Tohoku University
    \quad
    $^{6}$RIKEN AIP
  }
}

\begin{document}
\maketitle

\begingroup
\renewcommand{\thefootnote}{}
\footnotetext{%
  $^{*}$Equal contribution.
  \quad
  $^{\dagger}$Corresponding author.
  \quad
  Project page:
  \href{https://yinleung.com/denoise-ortho/}
       {\texttt{yinleung.com/denoise-ortho}}
}
\endgroup

\renewcommand{\thefootnote}{\arabic{footnote}}
\setcounter{footnote}{0}
\doparttoc
\faketableofcontents
\begin{abstract}
Muon has recently demonstrated strong empirical performance in large language model training, but the theoretical role of momentum in Muon remains unclear. Existing analyses of Muon either remove momentum to study spectral updates in isolation, or retain momentum without explaining why it improves empirical performance. Our work bridges this gap by showing momentum in Muon acts as a spectral filter. Under a structured signal-plus-perturbation gradient model, we prove that momentum suppresses perturbations while preserving the dominant signal, thereby enlarging the spectral gap between them. This enlarged gap stabilizes the singular subspaces of the matrix passed to Muon's orthogonalization step, making the resulting update more reliable. We further show that applying momentum before orthogonalization achieves \emph{provably stronger} alignment with the signal component of the gradient than either reversing this order or simply removing momentum. Experiments across diverse tasks, including LLM pretraining, support our theoretical analysis. More broadly, our theory offers a starting point for understanding the benefits of momentum in other matrix-based optimizers.
\end{abstract}

\section{Introduction}
\label{sec:intro}

Modern neural networks, especially large language models (LLMs), typically consist of billions of parameters and require substantial computational resources for training. As the model scale continues to grow, optimizer design has become increasingly important for efficient large-scale learning. In LLM training, adaptive first-order, coordinate-wise methods such as Adam~\citep{kingma2015adam} and AdamW~\citep{loshchilov2018decoupled} remain the de~facto choice. More recently, matrix-based optimizers such as Shampoo~\citep{gupta2018shampoo}, SOAP~\citep{vyas2025soap}, and Scion~\citep{pethick2025training} have received growing interest as alternatives that exploit matrix structure in gradient updates. Among these optimizers, Muon~\citep{jordan2024muon}, which approximately orthogonalizes the momentum matrix using a Newton--Schulz iteration, has shown promising empirical performance in LLM pretraining~\citep{liu2025muonscalablellmtraining}, and recent optimizer benchmarks further identify Muon, along with matrix-based methods more broadly, as a strong  alternative to AdamW~\citep{semenov2025benchmarking,wen2025fantastic}.

A second long-standing thread in optimizer design is the role of momentum. Classical analyses establish accelerated rates for momentum on smooth and strongly convex problems~\citep{polyak1964momentum,nesterov1983accelerated,sutskever2013on}. \citet{cutkosky2020momentum} prove that momentum eliminates the need for large batches in normalized SGD, which is the vector-side analog of Muon's polar update in the sense that both retain only the direction of the gradient. The practical picture in deep learning is, however, more nuanced. \citet{wang2024marginal} show that at small learning rates and high gradient noise, momentum offers only marginal benefit, suggesting that its provable acceleration sits in the deterministic (or large-learning-rate) regime rather than in stochastic noise reduction per~se. More recent work~\citep{li2025momentum} reframes momentum in the frequency domain as a low-pass filter on the gradient stream, suggesting that the relevant mechanism may be filtering rather than acceleration. Despite these developments, neither Muon's strong empirical performance nor the role momentum plays inside it is well understood.

In practice, Muon's pretraining performance is closely tied to momentum --- see \cref{fig:end2end}, where we see a clear trend that Polar-only pipeline (i.e., Muon without momentum) underperforms in pretraining --- yet existing theory has not successfully characterized the role of momentum inside Muon.
Substantial recent work analyzes Muon-style polar updates from the aspects of its spectral property, implicit bias, and feature learning, but \emph{without momentum}~\citep{ma2026preconditioning,su2025isotropic,davis2025spectral,li2026muon,wang2025muon,kim2026sharp}.
By construction, these analyses cannot explain momentum's benefit.
A thorough review of these works is deferred to \Cref{sec:related}.
More recently,
\citet{chen2026muon} establish that Muon converges to the spectrally-constrained minimizer, though its convergence rate deteriorates with heavier momentum.
\citet{kovalev2025understanding} shows convergence of Muon to a stationary point but heavier momentum brings no further improvement.
\citet{shulgin2026beyond} analyze Muon under an inexact linear minimization oracle (LMO) that models the Newton--Schulz approximation error, yet its theoretical bound deteriorates in the heavy-momentum regime.
\citet{fan2025implicitbiasspectraldescent} establish implicit bias of Muon with momentum, but still its convergence rate blows up with too heavy momentum.
On the whole, theory tells us that heavy momentum brings no benefit --- and may even degrade convergence --- inside Muon, but this obviously contradicts what we observe in practice.

\begin{figure}[htbp]
\centering
\begin{subfigure}[t]{0.49\linewidth}
  \centering
  \includegraphics[width=\linewidth]{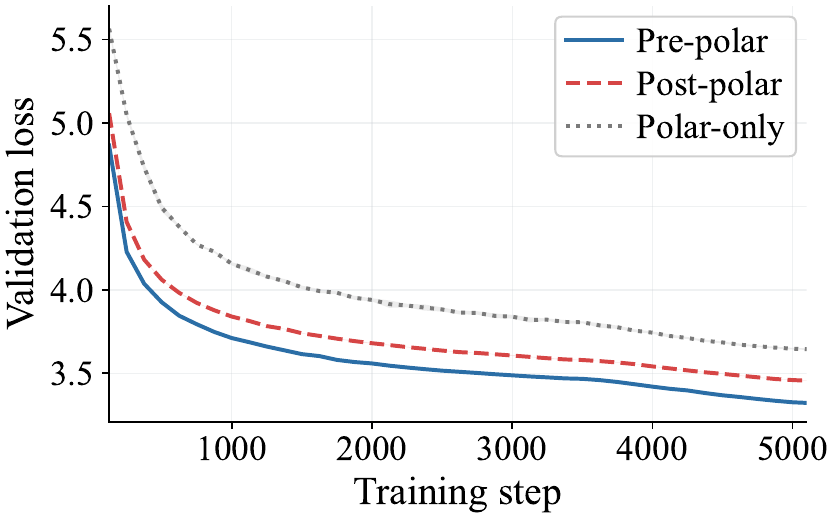}
  \caption{NanoGPT.}
  \label{fig:end2end-nanogpt}
\end{subfigure}
\hfill
\begin{subfigure}[t]{0.49\linewidth}
  \centering
  \includegraphics[width=\linewidth]{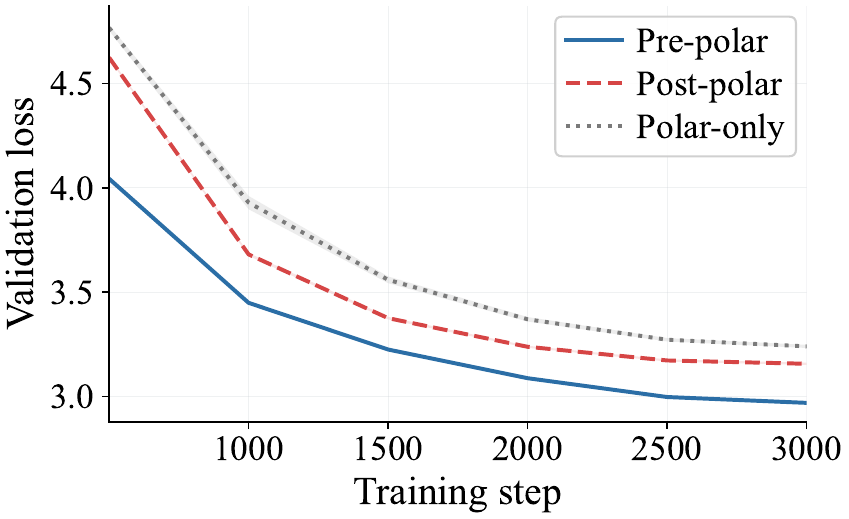}
  \caption{LLaMA 350M.}
  \label{fig:end2end-llama350m}
\end{subfigure}
\caption{End-to-end validation loss comparisons across (a) NanoGPT training and (b) LLaMA 350M training. The Muon Pre-polar pipeline outperforms Post-polar and Polar-only pipelines. The full experimental settings are in \Cref{app:exp-tasks}.}
\label{fig:end2end}
\end{figure}

The relationship between Muon's polar update and momentum is more nuanced.
Whereas Orthogonal-SGDM~\citep{tuddenham2022orthogonalising}, proposed prior to Muon, orthogonalizes each per-step gradient \emph{before} momentum smoothing, Muon applies the polar update \emph{after} momentum smoothing, significantly outperforming particularly on language model pretraining~\citep{jordan2024muon}.
Therefore, Muon's success encourages us to closely look at (i) why momentum smoothing matters practically, and (ii) why the polar factor should be put after momentum smoothing.
To this end, we compare three Muon pipelines for the weight update $\Delta W_t$ at iteration $t$, each combining a \emph{momentum buffer} $M_t=\beta M_{t-1}+(1-\beta)G_t$ and the polar factor $\cO(X)=UV^\top$ (from the thin SVD $X=U\Sigma V^\top$) on the gradient stream $\{G_t\}$ in different orders:
\begin{align}
\text{Pre-polar:}&\quad \Delta W_t\propto\cO(M_t),
\label{eq:forward}\\
\text{Post-polar:}&\quad \widetilde M_t=\beta\widetilde M_{t-1}+(1-\beta)\cO(G_t),\quad \Delta W_t\propto\widetilde M_t, \quad \text{and}
\label{eq:reversed}\\
\text{Polar-only:}&\quad \Delta W_t\propto\cO(G_t).
\label{eq:momefree}
\end{align}
Pre-polar update~\cref{eq:forward} is the standard Muon update~\citep{jordan2024muon}. Post-polar update~\cref{eq:reversed} formalizes the Orthogonal-SGDM variant. Polar-only update~\cref{eq:momefree} removes momentum entirely as a no-momentum baseline.
\Cref{fig:end2end} shows that Pre-polar dominates both Post-polar and Polar-only throughout training.
Two conclusions follow.
First, Pre-polar and Post-polar pipelines use the same two operations in opposite orders, so any gap between them must come from the ordering itself, because momentum and orthogonalization do not commute as numerical operations. Second, the gap from Pre-polar to Polar-only shows that momentum plays a role that orthogonalization alone cannot fulfill: momentum is an essential component of Muon, not a cosmetic acceleration trick.
This raises the central question of our interest:

\begin{quote}\emph{What role does momentum play on the gradient stream, and why does placing it before orthogonalization matter inside Muon?}\end{quote}

This paper supplies the missing puzzle of Muon theory. Momentum first acts as a spectral filter on the gradient stream. Inside Muon, this filtering must happen before orthogonalization. The polar factor preserves singular directions but replaces all singular values with ones, so when it acts directly on a noisy gradient it erases the amplitude gap between signal and noise directions, and we shall never distinguish signal and noise even after undergoing the subsequent momentum. When momentum acts first, it filters the gradient stream to preserve the persistent signal component and attenuate the perturbation, so the polar factor operates on a matrix whose singular-value structure already separates signal from noise. The guiding principle is simple: \emph{denoise first, then orthogonalize}.

\paragraph{Contributions.}
\begin{enumerate}[leftmargin=*]
\item \textbf{Spectral gap and singular-subspace reliability (\Cref{thm:spectral-gap} and \Cref{cor:direction}).} Under a structured signal-plus-perturbation gradient model with a coherent signal, a bounded variance mean-zero orthogonal (BVMZOS) perturbation, and the signal persistence assumption, momentum solely opens a spectral gap between the $r$-th and $(r{+}1)$-th singular values that grows with the momentum window size $T$, and aligns the top-$r$ singular subspaces of $M_t$ with those of the coherent signal. This benefits Pre-polar because the subsequent orthogonalization operates on a reliable momentum matrix rather than on the raw gradient matrix.

\item \textbf{Pre-polar dominance and quantitative separation (\Cref{thm:recovery,thm:separation}).} For any fixed gradient signal matrix and any absolutely continuous perturbation distribution, Pre-polar pipeline provably dominates both Post-polar and Polar-only pipelines in expected signal alignment once the momentum window size $T$ is sufficiently large. In the rank-1 spiked Gaussian model under low Signal-to-Noise Ratio (SNR) regime, this separation is quantitative: Polar-only alignment vanishes as the matrix size $m\to\infty$, while Pre-polar continues to recover the signal.

\item \textbf{Empirical validation across diverse tasks.} Synthetic spiked-perturbation models, CIFAR-10 and NanoGPT stationary and trajectory probes across layers and training checkpoints, and end-to-end NanoGPT and LLaMA 350M training support every prediction of \Cref{thm:spectral-gap}, \Cref{cor:direction}, and \Cref{thm:recovery}.
\end{enumerate}

\section{Preliminary and Analysis Setup}
\label{sec:setup}

\paragraph{Notation.}
Let $G_t\in\Rbb^{m\times n}$ ($m\ge n$) denote the gradient matrix at iteration $t$.
For the thin SVD decomposition $X=U\Sigma V^\top$ with orthogonal matrices $U\in\Rbb^{m\times n}$ and $V\in\Rbb^{n\times n}$, the polar factor is $\cO(X) \coloneqq UV^\top$.
We write $\sigma_k(X)$ for the $k$-th largest singular value, $\hat u_k(X)$, $\hat v_k(X)$ for the corresponding singular vectors, $\inner{A}{B}_F \coloneqq \mathrm{tr}(A^\top B)$ for the Frobenius inner product, $\fnorm{A}\coloneqq \sqrt{\operatorname{tr}(A^\top A)} = \sqrt{\sum_{i, j} A^2_{ij}}$ for the Frobenius norm, and $\snorm{\cdot}$ for the operator norm.

\paragraph{Momentum buffer.}
For the momentum coefficient $\beta\in[0,1)$, we define the \emph{momentum buffer} via the recursion
\begin{equation}
M_t=\beta M_{t-1}+(1-\beta)G_t,
\label{eq:ema}
\end{equation}
with the standard zero initialization $M_0=0$. The recursion unrolls to a convolution of the gradient stream as
$
M_t=\sum_{s=0}^{t-1} h_s\, G_{t-s} \quad \text{with the momentum kernel $h_s\coloneqq(1-\beta)\beta^s$.}
$
The form of~\cref{eq:ema} is the exponential moving average (EMA) normalization~\citep{gardner1985exponential} of the more general first-order recursion $M_t=\beta M_{t-1}+\gamma G_t$ with scalar $\gamma>0$. Because the polar factor is scale-invariant, Pre-polar update $\cO(M_t)$ depends only on the decay $\beta$, so the EMA buffer~\cref{eq:ema} and the Polyak/heavy-ball momentum ($\gamma=1$)~\citep{sutskever2013on} yield the same update after being passed to the polar factor $\cO(\cdot)$. \Cref{app:gain-invariance} records the derivation, and \Cref{app:initialization} discusses arbitrary initialization. Nesterov momentum, common in Muon implementations, applies the polar factor not to the momentum buffer $M_t$ in~\cref{eq:ema} but to a linear combination $\nu\,M_t + \kappa\,G_t$ of the $M_t$ and the current gradient. This is not covered by our analysis. We record its kernel weights and effective sample size for completeness in \Cref{app:nesterov}.

\paragraph{Effective window size.} Based on the momentum buffer~\cref{eq:ema}, the analysis uses the \emph{effective window size}, 
$
T \coloneqq \frac{1}{1-\beta},
$
of the momentum kernel $(1-\beta)\beta^s$~\citep{ghosh2026understanding}. As momentum becomes heavier, gradient smoothing is effective across longer horizon $T$.


\paragraph{Gradient model.}
\label{sec:assumptions}
Our results study the momentum buffer~\cref{eq:ema} applied to a gradient stream with the following \emph{signal-plus-perturbation} structure.

\begin{assumption}[Structured gradient decomposition]
\label{ass:decomp}
For each gradient-update step $t$,
\begin{equation}
G_t = G_t^{\mathrm{sig}} + \Xi_t,
\label{eq:decomp}
\end{equation}
where the following two conditions are satisfied:
\begin{enumerate}[label=(\alph*),leftmargin=*]
\item \textbf{(Coherent signal.)}
$G_t^{\mathrm{sig}}=\sum_{k=1}^r\lambda_k(t)u_kv_k^\top$,
where $\{u_k\}_{k=1}^r\subset\Rbb^m$ and $\{v_k\}_{k=1}^r\subset\Rbb^n$ are \emph{temporally coherent} (i.e., deterministic and time-invariant) orthonormal families, and the \emph{signed coordinate} $\lambda_k(t)$ of $G_t^{\mathrm{sig}}$ along the $k$-th signal direction $u_kv_k^\top$ is a real-valued random variable on the same probability space as the perturbation matrix $\Xi_t$ (introduced below), with finite second moment.

\item \textbf{(Bounded variance mean-zero orthogonal perturbation)}
$\{\Xi_t\}_{t\ge 0}$ is an $m\times n$ matrix-valued \emph{bounded variance mean-zero orthogonal sequence (BVMZOS)}
: there exists $\eta>0$ such that for every $t\ge 0$,
$
\e{\Xi_t} = 0$ and $ \e{\fnorm{\Xi_t}^2}\le \eta.
$
Moreover, for any $t_1, t_2\ge 0$ with $t_1\neq t_2$,
$
\e{\inner{\Xi_{t_1}}{\Xi_{t_2}}_F} = 0.
$
\end{enumerate}
\end{assumption}

\Cref{ass:decomp}(b) is sufficiently general to cover time-dependent noise and finite-variance heavy-tailed noise, going beyond the standard i.i.d. sub-Gaussian setting.
The rank-$r$ decomposition can in fact address more general cases by absorbing higher-rank yet marginal components into the perturbation $\Xi_t$ at the cutoff parameter $r$.
Empirical and theoretical evidence suggests that neural network gradients maintain effectively low rank~\citep{gur2018gradient,ba2022high}, and on real neural network gradients $r$ can be regarded as an effective rank where the spectrum has its dominant spikes~\citep{mousavi2023gradient,braun2026spectral}.

\begin{remark}[Matrix sensing example]
\label{rem:components}
While we introduce \Cref{ass:decomp} as a theoretical proxy to practically observed gradient streams in large-scale training, it is naturally satisfied by a standard matrix-sensing example.
In this setting, the signal component $G_t^\mathrm{sig}$ corresponds to the full-batch gradient lying in a time-invariant low-rank subspace, while the BVMZOS perturbation $\Xi_t$ corresponds to the mini-batch sampling noise.
\\
Consider the regression model
$
y=\inner{W^*}{X}_F+\epsilon,
$
where $W^*\in\Rbb^{m\times n}$ is a fixed but unknown matrix, $X\in\operatorname{span}(\{u_kv_k^\top\}_{k=1}^r)$ is the sensing matrix assumed to be low-rank, and $\epsilon$ is mean-zero noise.
Write
\(X_i=\sum_{k=1}^r\mu_{i,k}u_kv_k^\top\)
with some $\mu_{i,k}\in\Rbb$.
At the parameter $W=W_t$,
the least-squares loss
$
\mathcal{L}(W)  \coloneqq  \frac{1}{2N}\sum_{i=1}^N(y_i-\inner{W}{X_i}_F)^2
$
has gradient
$
\nabla\mathcal L(W_t)
= \sum_{k=1}^r\lambda_k(t)u_kv_k^\top,
$
where the signed coordinate is $\lambda_k(t) \coloneqq \frac1N\sum_{i=1}^N \mu_{i,k}(\inner{W_t}{X_i}_F-y_i)$.
Thus the full-batch gradient has exactly the coherent-signal structure. If the full-batch gradient is replaced with an unbiased mini-batch estimator,
the mini-batch sampling noise satisfies \Cref{ass:decomp}(b), provided that the sampling is independently conditional on the past.

\end{remark}

\paragraph{Signal persistence.}

In addition to the temporal coherency in \Cref{ass:decomp},
the signed coordinates $\lambda_k(t)$ are assumed to weakly drift in time.
We measure this drift relative to the \emph{root-mean-squared (RMS) coefficient scale}
$\lambda_k \coloneqq \sqrt{\frac{1}{|\mathcal{I}|}\sum_{t\in \mathcal{I}}\e{\lambda_k(t)^2}},$
where $\mathcal{I}$ is a fixed deterministic interval where all our subsequent discussion takes place,
and the expectation is taken over the randomness of the perturbation matrix $\set{\Xi_t}$.
This is a finite, deterministic scalar by \Cref{ass:decomp}(a).
Without loss of generality, we index the coherent signal directions so that $\lambda_1\ge\lambda_2\ge\cdots\ge\lambda_r>0$. We use the threshold $t\ge c_0 T$ throughout, where $c_0>0$ is a fixed constant chosen so that $\beta^{c_0 T}$ is negligible (\Cref{app:initialization}).

The quantitative statement is given as follows:

\begin{assumption}[Signal persistence]
\label{ass:persistence}
Define the momentum-filtered signed coordinate
$
\bar\lambda_k(t)  \coloneqq  (1-\beta)\sum_{\tau\ge 0}\beta^\tau\,\lambda_k(t-\tau).
$
There exists $c_{\mathrm{sig}}\in(0,1]$, irrespective of $\beta$ or $T=1/(1-\beta)$, such that for all $t\ge c_0 T$ and $k = 1,\ldots, r$,
\[
|\bar\lambda_k(t)|\ge c_{\mathrm{sig}}\lambda_k.
\]
\end{assumption}

This says the momentum-filtered coordinate $|\bar\lambda_k(t)|$ is sufficiently bounded away from zero, 
even if the unsmoothed $\lambda_k(t)$ could be close to zero.
If the signed coordinate admits a specific form such as $\lambda_k(t)=\bar\mu_k+\xi_k(t)$ with $|\bar\mu_k|>0$ and zero-mean sub-Gaussian $\xi_k(t)$, \Cref{ass:persistence} holds with high probability for large enough window $T$ (\Cref{app:persistence}).

\section{Filtering Effect of Momentum: Spectral Gap and Subspace Reliability}
\label{sec:spectral-gap}

Before investigating the interaction between momentum and the polar factor, we present our first analysis focusing on the benefit of momentum in signal recovery, throughout this section.
The full proofs of the statements in this section are deferred to \Cref{app:proof-thm1-cor-1}.

\paragraph{Amplitude recovery (in theory).}
As the effective window size $T = 1/(1-\beta)$ grows, the momentum buffer $M_t$ averages out the perturbations while preserving the coherent signal $G_t^{\mathrm{sig}}$.
\Cref{thm:spectral-gap} below quantifies this: the tail singular values decay as $(2T-1)^{-1/4}$, creating a spectral gap with a larger effective window size $T$.

\begin{theorem}[Spectral gap of the momentum buffer]
\label{thm:spectral-gap}
Under \Cref{ass:decomp,ass:persistence}, for $m\ge n$, with probability of at least $1-(2T-1)^{-\frac12}$, we have
\begin{align}
\sigma_k(M_t)
\ge
c_{\mathrm{sig}}\lambda_k
- \frac{{\sqrt{\eta}}}{(2T-1)^{\frac14}},
\quad k=1,\ldots,r, \text{~~~and~~~}
\sigma_{r+1}(M_t)
\le \frac{{\sqrt{\eta}}}{(2T-1)^{\frac14}}.
\label{eq:noise-bound}
\end{align}
\end{theorem}

Hence, the spectral gap widens with a larger effective window size $T$, or equivalently, with a heavier momentum $\beta$.
The perturbation floor $O((2T-1)^{-1/4})$ in \cref{eq:noise-bound} monotonically decays.
While this behavior is conceptually intuitive --- momentum smoothening out noise perturbations --- this spectral gap decay behavior aligns well with the practical benefit of momentum at near-1 $\beta$, in stark contrast to the standard optimization analysis whose theoretical guarantees deteriorate with large~$\beta$~\citep{chen2026muon,kovalev2025understanding,fan2025implicitbiasspectraldescent,shulgin2026beyond}.

\begin{figure}[htbp]
  \centering
  \begin{subfigure}[t]{0.32\linewidth}
    \centering
    \includegraphics[width=\linewidth]{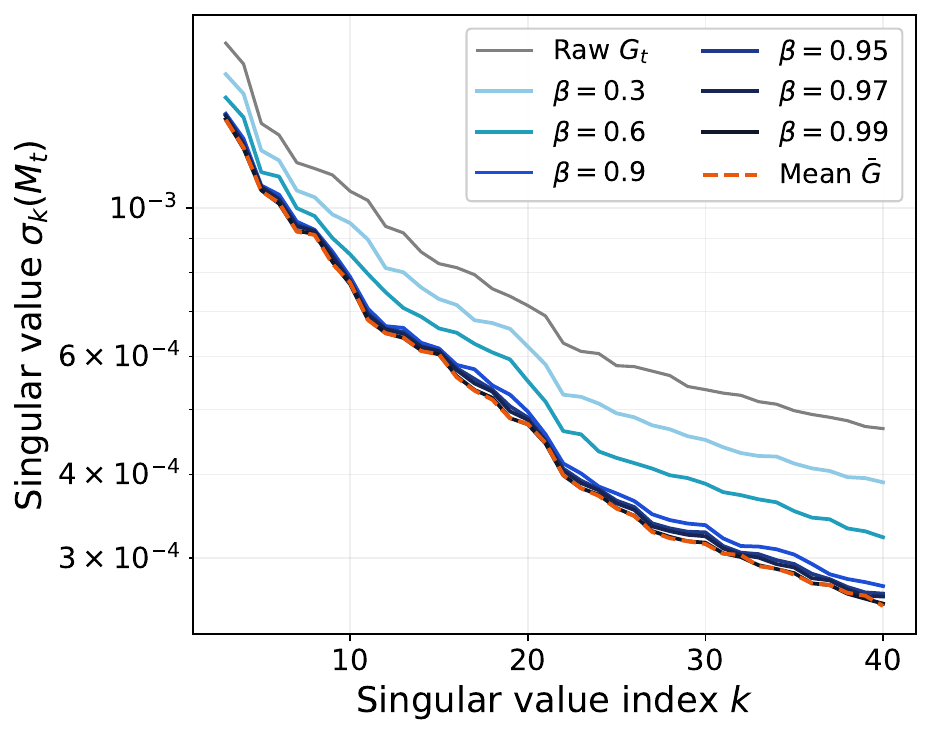}
    \caption{Filtered singular value spectra.}
    \label{fig:ch1-frozen-spectrum}
  \end{subfigure}
  \hfill
  \begin{subfigure}[t]{0.32\linewidth}
    \centering
    \includegraphics[width=\linewidth]{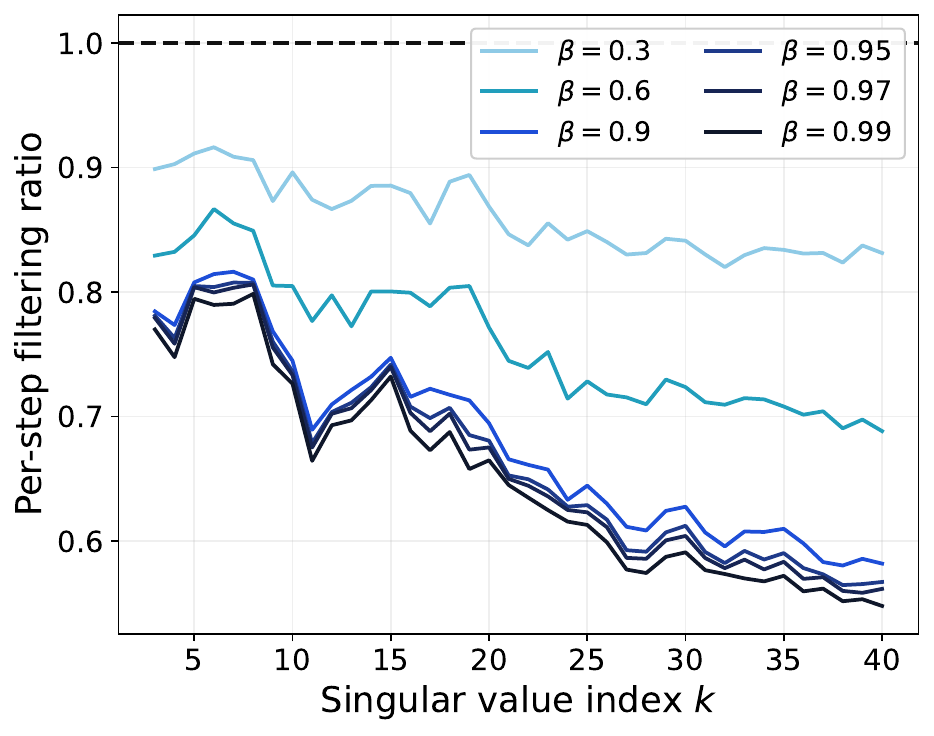}
    \caption{Per-step filtering ratio.}
    \label{fig:ch1-frozen-ratio}
  \end{subfigure}
  \hfill
  \begin{subfigure}[t]{0.32\linewidth}
    \centering
    \includegraphics[width=\linewidth]{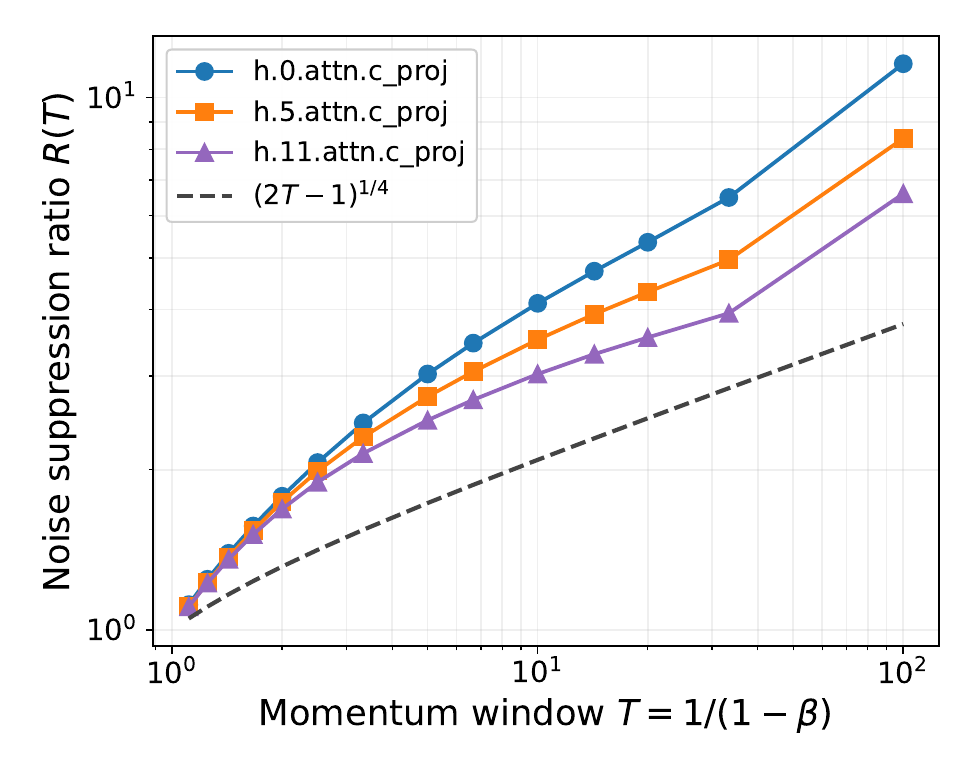}
    \caption{Noise suppression ratio $R(T)$.}
    \label{fig:ch1-frozen-edge}
  \end{subfigure}
  \caption{Spectral filtering visualization. (a) Filtered momentum singular value spectra (blue), the raw gradient spectrum (grey), and the mean-gradient spectrum (dashed orange) on layer \texttt{h.0}. (b) Per-step filtering ratio on \texttt{h.0}. (c) Noise-suppression ratio $R(T)$ on each layer \texttt{h.0}, \texttt{h.5}, and \texttt{h.11} ($K=500$) versus momentum window size $T = 1/(1-\beta)$, with the dashed $(2T-1)^{1/4}$ floor.}
  \label{fig:ch1-frozen-t1-validation}
\end{figure}

\paragraph{Amplitude recovery (in simulation).}
We next use numerical simulations to verify that the momentum buffer indeed creates the spectral gap implied by \Cref{thm:spectral-gap}, beyond the stylized gradient model.
To do so, we use the \emph{stationary probe} by collecting $K$ mini-batch gradients on a saved checkpoint with all model weights held fixed, and applying the momentum pipelines to the gradient buffer in collection order.
Holding the weights fixed removes the non-stationarity caused by parameter updates, so the probe isolates the stochastic gradient stream around a realistic training checkpoint.
In this controlled setting, the coherent signal is a well-defined, time-invariant target, just as in the stylized model.
At the same time, unlike the stylized model, the gradients are generated by a real learning task and therefore retain the structure of an actual neural-network training problem.
Here, we use the NanoGPT stationary probe at training step 3000 (out of 5100 total steps) on the attention output projection layers. The full settings appear in \Cref{app:exp-setup}, with synthetic, CIFAR-10, and additional NanoGPT layer and checkpoint results in \Cref{app:thm1-supplement}.

\Cref{thm:spectral-gap} contains two predictions: signal preservation in the \emph{head} ($k\le r$), $\sigma_k(M_t)\ge c_{\mathrm{sig}}\lambda_k - \sqrt{\eta}\,(2T-1)^{-1/4}$, which keeps the leading singular values pinned near the signal scale $c_{\mathrm{sig}}\lambda_k$, and perturbation attenuation in the \emph{tail} ($k>r$), $\sigma_{r+1}(M_t)\le \sqrt{\eta}\,(2T-1)^{-1/4}$, which drives the tail to zero at rate $T^{-1/4}$.
We validate these on real gradients via two empirical metrics.
Let $M_K^{(\beta)}$ denote the momentum buffer~\cref{eq:ema} with decay $\beta$ after collecting a buffer of $K$ gradients, and $G_K$ the last raw gradient in that buffer.
The \emph{per-step filtering ratio} $\mathrm{Filt}_k(\beta) \coloneqq \sigma_k(M_K^{(\beta)})/\sigma_k(G_K)$ is the index-wise attenuation between the raw gradient $G_K$ and the buffer $M_K^{(\beta)}$.
\Cref{thm:spectral-gap} predicts that signal preservation keeps $\mathrm{Filt}_k(\beta)$ near $1$ at the head ($k\le r$), while perturbation is suppressed as $\beta$ gets closer to $1$.

The \emph{noise-suppression ratio} $R(T) \coloneqq \|G_K - \bar G\|_{\op}/\|M_K^{(\beta)} - \bar G\|_{\op}$, with the in-buffer mean gradient $\bar G \coloneqq \frac{1}{K}\sum_{t=1}^{K} G_t$ as the reference of the coherent gradient signal, is the ratio of the raw gradient's perturbation to the perturbation remaining in the momentum buffer. Under stationarity $\bar G$ approximates the coherent signal, so subtracting it approximately removes the signal from both numerator and denominator and isolates the perturbation.
\Cref{thm:spectral-gap} predicts that this suppression grows with the momentum window size $T$, with $R(T)$ bounded below by the $(2T-1)^{1/4}$ floor.
The full derivation of this argument is deferred to \Cref{app:measurements}.

\Cref{fig:ch1-frozen-spectrum,fig:ch1-frozen-ratio} support \Cref{thm:spectral-gap} via the stationary NanoGPT probe at training step~3000 on \texttt{h.0.attn.c\_proj} layer. They show that as $\beta$ increases, the tail singular values of $M_K^{(\beta)}$ and the tail filtering ratios are suppressed more than the head, opening a gap between head and tail in both panels.
This is the spectral gap predicted by \Cref{thm:spectral-gap}, and it widens with $\beta$.
In \cref{fig:ch1-frozen-spectrum}, the momentum spectrum moves toward the mean-gradient spectrum $\sigma_k(\bar G)$ as $\beta$ increases, with the head reaching $\sigma_k(\bar G)$ ahead of the tail, consistent with \Cref{thm:spectral-gap}.
\Cref{fig:ch1-frozen-edge} shows the noise-suppression ratio $R(T)$ on three attention output projections lies above $(2T-1)^{1/4}$ guide predicted by \Cref{thm:spectral-gap}, serving as a sanity check of our theory.

\paragraph{Subspace recovery (in theory).}
\Cref{thm:spectral-gap} controls singular \emph{values} corresponding to the coherent signal.
Combined with Wedin's $\sin\Theta$ theorem~\citep{wedin1972perturbation}, the spectral gap it opens also controls the singular \emph{subspaces}: the larger the gap, the closer the top-$r$ singular subspaces of $M_t$ are to those of the coherent gradient signal.

\begin{corollary}[Singular-subspace reliability]
\label{cor:direction}
Under the conditions of \Cref{thm:spectral-gap}, with probability of at least $1-(2T-1)^{-\frac12}$, we have
\begin{align}
&\|\sin\Theta(\hat U_t,U)\|_2
\le
\left( c_{\mathrm{sig}}\lambda_r
- \frac{{\sqrt{\eta}}}{(2T-1)^{\frac14}}
\right)^{-1}
\cdot\frac{{\sqrt{\eta}}}{(2T-1)^{\frac14}}
= O(T^{-1/4})
\\
&\|\sin\Theta(\hat V_t,V)\|_2
\le
\left( c_{\mathrm{sig}}\lambda_r
- \frac{{\sqrt{\eta}}}{(2T-1)^{\frac14}}
\right)^{-1}
\cdot\frac{{\sqrt{\eta}}}{(2T-1)^{\frac14}}
= O(T^{-1/4}).
\label{eq:subspace-right}
\end{align}
where $U \coloneqq [u_1,\ldots,u_r]$, $V \coloneqq [v_1,\ldots,v_r]$, and $\hat U_t$, $\hat V_t$ denote the top-$r$ left and right singular subspaces of $M_t$, respectively. The principal angles $0\le\theta_1\le\cdots\le\theta_r\le\pi/2$ between the column spans of $\hat U_t$ and $U$ are defined by $\cos\theta_i \coloneqq \sigma_i(\hat U_t^\top U)$. Stacking them into $\Theta(\hat U_t, U) \coloneqq \mathrm{diag}(\theta_1,\ldots,\theta_r) \in\Rbb^{r\times r}$ and applying $\sin$ entrywise gives $\sin\Theta(\hat U_t,U) = \mathrm{diag}(\sin\theta_1,\ldots,\sin\theta_r)$, whose spectral norm $\|\sin\Theta(\hat U_t, U)\|_2=\sin\theta_r$ is the sine of the largest principal angle.
\end{corollary}

Each $\|\sin\Theta\|_2$ term lies in $[0,1]$ and measures how far a top-$r$ singular subspace of $M_t$ is from the corresponding signal subspace: $0$ when they coincide, $1$ when orthogonal. \Cref{cor:direction} thus says the momentum buffer's leading singular subspaces become reliable as the momentum window size grows, with the deviation vanishing at rate $T^{-1/4}$.

\paragraph{Subspace recovery (in simulation).}

\Cref{fig:ch1-frozen-c1-validation} supports \Cref{cor:direction} at each rank $r\in$ \{1,5,10\} via the stationary NanoGPT probe at training step~3000 on \texttt{h.0.attn.c\_proj} layer. The figure plots two empirical \emph{subspace alignment errors} on real gradients: the left-subspace error $\sin\Theta_U \coloneqq \sin\theta_r(M_K^{(\beta)};\,\bar G)$ and the right-subspace error $\sin\Theta_V \coloneqq \sin\theta_r(M_K^{(\beta)\top};\,\bar G^{\top})$, where $\sin\theta_r(A;\,B) \coloneqq \|\sin\Theta(U_r(A),\,U_r(B))\|_2$. Both measure the principal-angle distance between the top-$r$ singular subspaces of $M_K^{(\beta)}$ and those of $\bar G$ (full definition in \Cref{app:measurements}).

\begin{figure}[htbp]
\centering
\includegraphics[width=\linewidth]{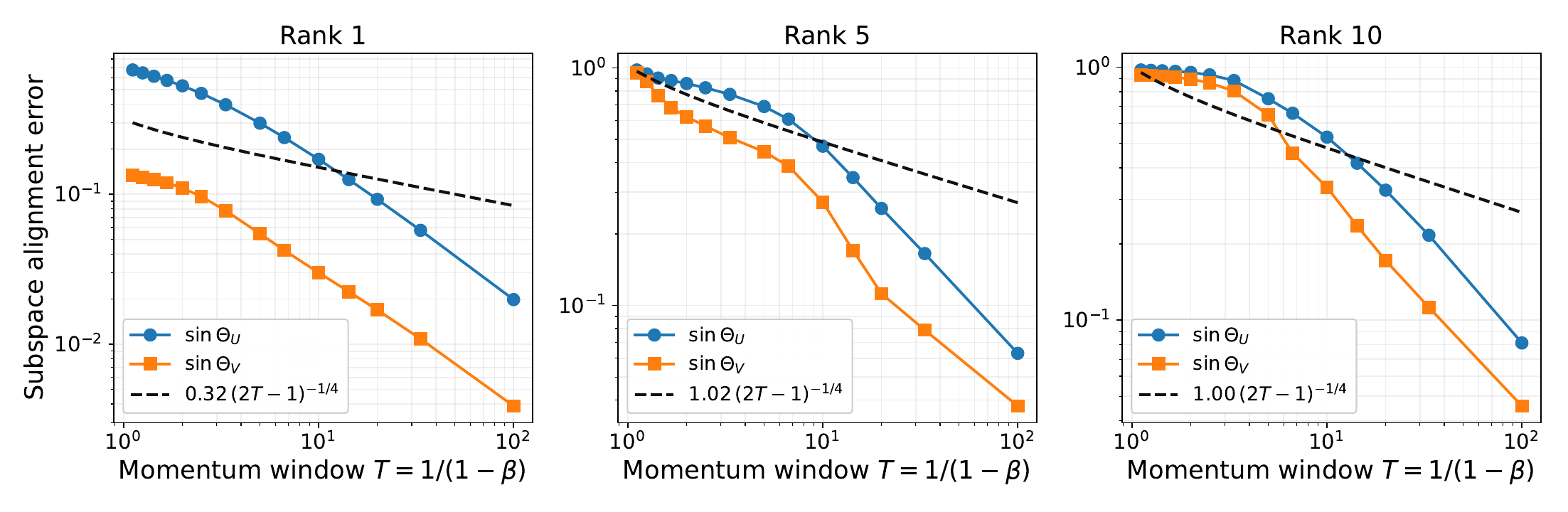}
\caption{Stationary probe subspace alignment error $\sin\Theta_U$ and $\sin\Theta_V$ at ranks $r\in\{1, 5, 10\}$ versus momentum window size $T = 1/(1-\beta)$, with the dashed $c_r\,(2T-1)^{-1/4}$ guide ($c_r$ fitted independently per panel).}
\label{fig:ch1-frozen-c1-validation}
\end{figure}

\section{Noncommutativity of Momentum and Orthogonalization}
\label{sec:spectral-gapII}

\paragraph{Signal-recovery separation (in theory).}
While \Cref{sec:spectral-gap} focuses on the denoising effect arising solely from momentum, 
this section investigates the interaction between momentum and the polar update across the variants: Pre-polar update $\cO(M_t)$, Post-polar update $\widetilde M_t$, and Polar-only $\cO(G_t)$ in \cref{eq:forward,eq:reversed,eq:momefree}.
Under the same stylized gradient model introduced in \Cref{sec:setup}, \Cref{thm:recovery} shows that Pre-polar recovers the signal matrix $G^\mathrm{sig}$ correctly as momentum filters out the BVMZOS perturbation with a sufficiently large effective window.
The expected scores of Polar-only and Post-polar remain separated
from this optimum.
Thus, our main theorem successfully demonstrates the separation between Pre- and Post-polar, explaining the practical superiority of Muon over Orthogonal-SGDM.
\Cref{thm:separation} is a quantitative version of the separation result in \Cref{thm:recovery}, proving that the signal alignment asymptotically vanishes with the matrix size $m\to\infty$ if we do not use Pre-polar.

Before stating the main theorem in this section, we introduce the following stronger assumptions in addition to \Cref{ass:decomp}, which are used in the proof of \Cref{thm:recovery}.

\begin{assumption}[Time-invariant signal component]
\label{ass:time-invariant-signal}
In addition to \Cref{ass:decomp}(a), we assume that the signal component $G^\mathrm{sig}_t$ is time-invariant, i.e., $G^\mathrm{sig}_t = G^\mathrm{sig}$ for all $t$.
\end{assumption}

\begin{assumption}[Pairwise independence of perturbations]
\label{ass:pairwise-independence}
In addition to \Cref{ass:decomp}(b), we assume that the perturbations $\{\Xi_t\}$ are pairwise independent.
As a consequence, for any measurable function $f$ and $g$, and $t_1, t_2\ge 0$ such that $t_1 \neq t_2$,
$
\e{f(\Xi_{t_1}) g(\Xi_{t_2})} = \e{f(\Xi_{t_1})} \e{g(\Xi_{t_2})}.
$
\end{assumption}

\begin{theorem}[Pre-polar recovery versus non-denoised baselines]
\label{thm:recovery}
Assuming \Cref{ass:decomp,ass:persistence,ass:time-invariant-signal} and $m\ge n$,
for every choice of fixed $(G^\mathrm{sig},\{\Xi_t\})$ such that the law of $\Xi_t$ is absolutely continuous with respect to the Lebesgue measure,
there exists a constant $C > 0$, independent of the momentum window size $T$, such that the following statements hold all at once.
\begin{enumerate}[label=(\roman*),leftmargin=*]
\item \textbf{Polar-only baseline.}
\[
\e{
\frac{\inner{\cO(G_t)}{G^\mathrm{sig}}_F}{\inner{\cO(G^\mathrm{sig})}{G^\mathrm{sig}}_F}
}
\le 1 - C.
\]

\item \textbf{Post-polar pipeline.}
\[
\e{
\frac{\inner{\tilde M_t}{G^\mathrm{sig}}_F}{\inner{\cO(G^\mathrm{sig})}{G^\mathrm{sig}}_F}
}
\le 1 - C.
\]

If we moreover assume \Cref{ass:pairwise-independence}, then for Post-polar pipeline,
\begin{equation}
\prob{
\frac{\inner{\tilde M_t}{G^\mathrm{sig}}_F}{\inner{\cO(G^\mathrm{sig})}{G^\mathrm{sig}}_F}
\ge 1 - \frac{C}{2}
}
\le
\frac{4}{C^2(2T-1)}.
\end{equation}

\setcounter{enumi}{2}
\item \textbf{Pre-polar pipeline.}
Denote $C' = \frac{2n\sqrt{\eta}}{\inner{\cO(G^\mathrm{sig})}{G^\mathrm{sig}}_F}$.
Then
\begin{equation}
\e{
\frac{\inner{\cO(M_t)}{G^\mathrm{sig}}_F}{\inner{\cO(G^\mathrm{sig})}{G^\mathrm{sig}}_F}
}
\ge 1 - \frac{C'}{\sqrt{2T-1}}.
\label{eq:forward-align}
\end{equation}

Moreover,
\begin{equation}
\prob{
\frac{\inner{\cO(M_t)}{G^\mathrm{sig}}_F}{\inner{\cO(G^\mathrm{sig})}{G^\mathrm{sig}}_F}
\ge 1 - \frac{C'}{(2T-1)^{\frac14}}
}
\ge 1 - \frac{1}{(2T-1)^{\frac14}}.
\label{eq:forward-align-prob}
\end{equation}

\item \textbf{Dominance over both non-Pre-polar baselines.}
There exists $T_0$ such that for all $T\ge T_0$,
\begin{equation}
\e{
\frac{\inner{\cO(M_t)}{G^\mathrm{sig}}_F}{\inner{\cO(G^\mathrm{sig})}{G^\mathrm{sig}}_F}
}
\ge
\max\left\{
\e{
\frac{\inner{\cO(G_t)}{G^\mathrm{sig}}_F}{\inner{\cO(G^\mathrm{sig})}{G^\mathrm{sig}}_F}
}
,
\e{
\frac{\inner{\tilde M_t}{G^\mathrm{sig}}_F}{\inner{\cO(G^\mathrm{sig})}{G^\mathrm{sig}}_F}
}
\right\} + \frac C 2.
\label{eq:ordering}
\end{equation}
If we moreover assume \Cref{ass:pairwise-independence}, then
\begin{equation}
\prob{
\frac{\inner{\cO(M_t)}{G^\mathrm{sig}}_F}{\inner{\cO(G^\mathrm{sig})}{G^\mathrm{sig}}_F}
\ge
\frac{\inner{\tilde M_t}{G^\mathrm{sig}}_F}{\inner{\cO(G^\mathrm{sig})}{G^\mathrm{sig}}_F} + \frac C 4
}
\ge 
1 - \frac{4}{C^2(2T-1)} - \frac{1}{(2T-1)^{\frac14}}.
\end{equation}
\end{enumerate}
\end{theorem}
The signal alignment score in \Cref{thm:recovery} uses a fixed
signal-dependent normalization, different from cosine similarity.
It measures the signal component carried by the update,
with dependence on both direction and magnitude.
For Post-polar, the gap can reflect both directional disagreement
and norm contraction. Thus Pre-polar pipeline strictly dominates both Polar-only and Post-polar pipelines for sufficiently large momentum windows
in expected signal alignment.
With the pairwise independence assumption (\Cref{ass:pairwise-independence}), Pre-polar pipeline also dominates Post-polar pipeline in probability.
The proof is deferred to \Cref{app:proof-thm2}.

\paragraph{Signal-recovery separation (in simulation).}
\Cref{fig:frozen-thm2} supports \Cref{thm:recovery} on the same stationary NanoGPT probe at training step~3000 used in \Cref{sec:spectral-gap}, on the \texttt{h.0.attn.c\_proj} layer.
Synthetic, CIFAR-10, and other-layer NanoGPT extensions are provided in \Cref{app:thm2-supplement}.
The figures use two empirical \emph{signal alignment} metrics on real gradients (see \Cref{app:measurements} for details): the rank-5 alignment $\mathrm{Align}_r(A;\,\bar G) \coloneqq \|U_r(\bar G)^\top A\, V_r(\bar G)\|_F / \sqrt{r} \in [0,1]$, which measures how much of $A$ lies in the top-$r$ singular subspaces of $\bar G$, and the full-rank alignment $\mathrm{Align}_{\mathrm{full}}(A;\,\bar G) \coloneqq \langle A,\, \cO(\bar G)\rangle_F / \min(m,n) \in [-1,1]$, the signed Frobenius inner product between $A$ and the polar factor $\cO(\bar G)$, normalized by $\min(m,n)=\|\cO(\bar G)\|_F^2$ so that $\cO(\bar G)$ itself scores $1$. Here $A$ is the placeholder for any of the pipelines $\cO(G_t)$, $\cO(M_t)$, or $\tilde{M}_t$.
\Cref{fig:frozen-beta-sweep} shows the full-rank signal alignment for Pre-polar, Post-polar and Polar-only pipelines versus the momentum coefficient $\beta$.
Only Pre-polar shows the monotonically improving trend in $\beta$, which is consistent with the success probability~\cref{eq:forward-align-prob} increasing in $\beta$.
\Cref{fig:frozen-align-training-r5,fig:frozen-align-training-full} show the same signal alignment at fixed $\beta=0.95$ across different checkpoints, but with rank-5 and full-rank matrices, respectively.
In both cases, Pre-polar is consistently higher than Post-polar and Polar-only pipelines across the different training checkpoints, which again supports Pre-polar's noise-smoothening mechanism.
Note that Pre-polar outperforms especially for steps 4000 and 5000, where training switches into the NanoGPT warmdown phase, i.e., linear learning-rate decay is used over the final 1450/5100 training steps.

\begin{figure}[t]
  \centering
  \begin{subfigure}[t]{0.32\linewidth}
    \centering
    \includegraphics[width=\linewidth]{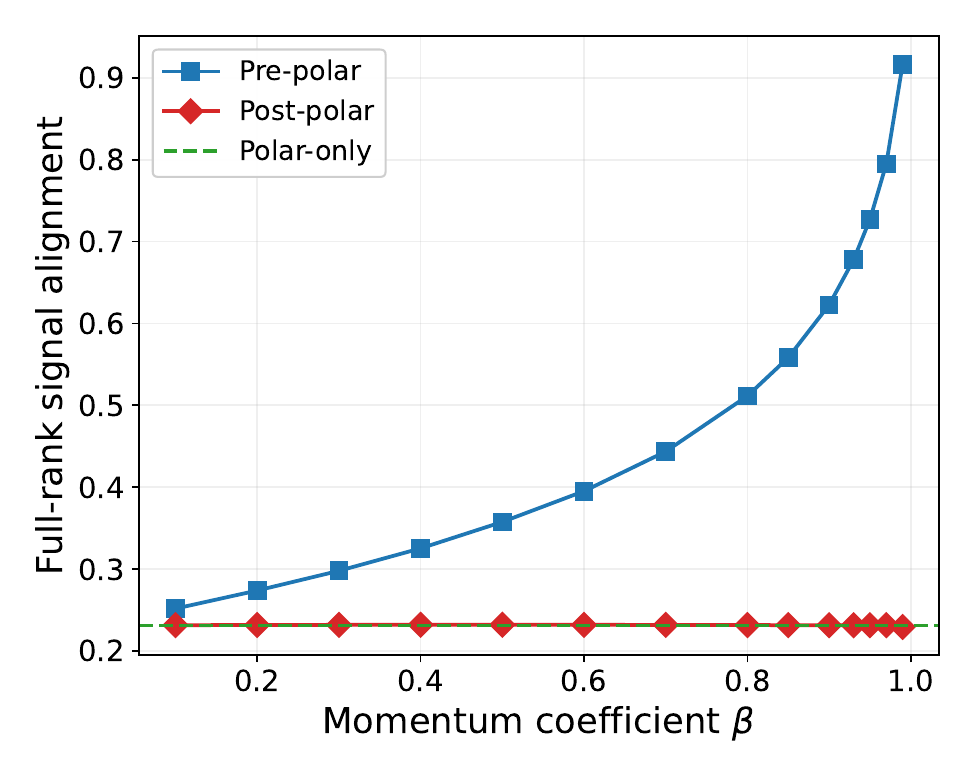}
    \caption{Full-rank alignment vs. $\beta$.}
    \label{fig:frozen-beta-sweep}
  \end{subfigure}\hfill
  \begin{subfigure}[t]{0.32\linewidth}
    \centering
    \includegraphics[width=\linewidth]{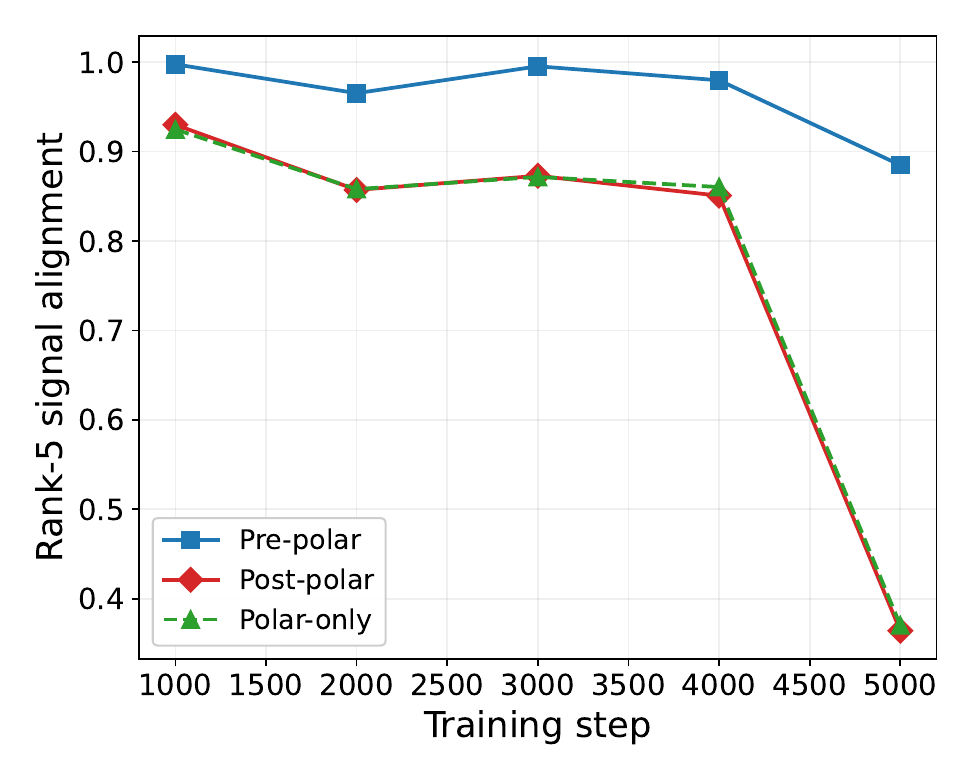}
    \caption{Rank-5 alignment over training.}
    \label{fig:frozen-align-training-r5}
  \end{subfigure}\hfill
  \begin{subfigure}[t]{0.33\linewidth}
    \centering
    \includegraphics[width=0.98\linewidth]{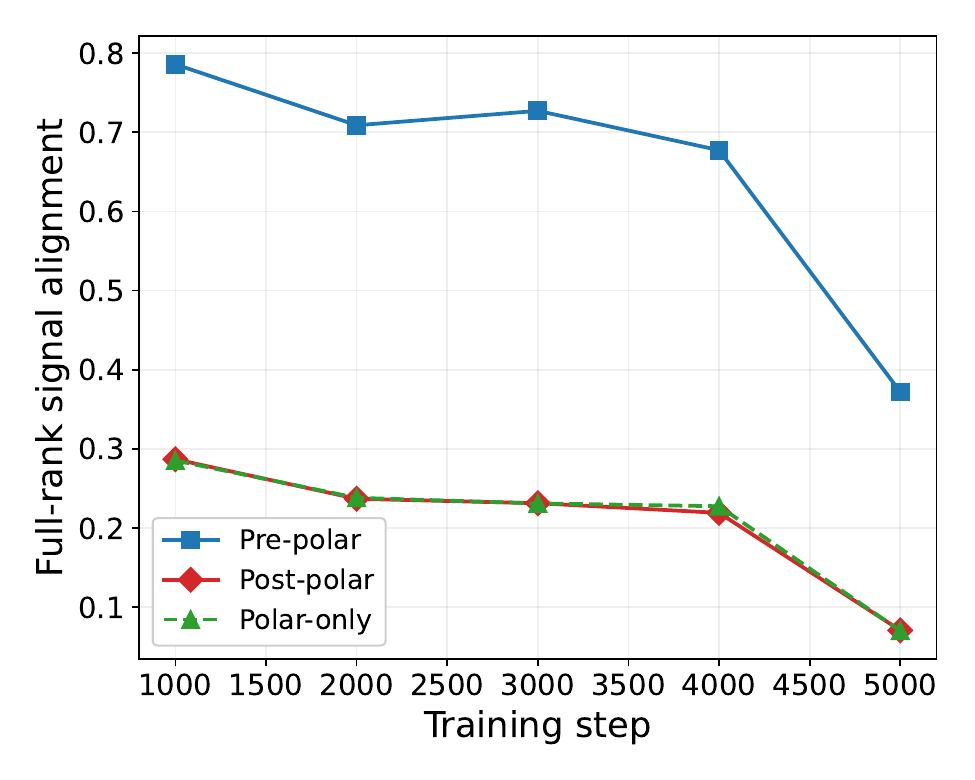}
    \caption{Full-rank alignment over training.}
    \label{fig:frozen-align-training-full}
  \end{subfigure}
  \caption{Stationary probe signal alignment for the three pipelines defined in~\cref{eq:forward}--\cref{eq:momefree} ($K=500$). (a)~$\beta$-sweep at the step-3000 checkpoint of the full-rank signal alignment. (b)~rank-5 signal alignment at $\beta=0.95$ across the five checkpoints. (c)~Full-rank signal alignment at $\beta=0.95$ across the same five checkpoints. The reference $\bar G = K^{-1}\sum_t G_t$ replaces $G^\mathrm{sig}$ on real gradients.}
  \label{fig:frozen-thm2}
\end{figure}

\paragraph{Quantitative signal-recovery separation.}
\label{sec:separation-theorem}

\Cref{thm:recovery} delineates that Pre-polar pipeline dominates Polar-only counterpart in the signal alignment,
but we do not quantitatively characterize the separation constant $C>0$.
Despite being challenging under the general coherent signal in \Cref{ass:decomp}, we show below that Polar-only alignment vanishes as the gradient matrix size $(m,n)$ grows, jointly under (i) the rank-1 spiked signal model, and (ii) the low SNR regime.

\begin{assumption}[Rank-1 spiked model]
  \label{ass:rank-1}
  For each $t$, assume that the gradient admits the signal-perturbation decomposition $G_t = G^{\mathrm{sig}} + \Xi_t = \lambda uv^\top + \Xi_t$,
  where $u\in\Rbb^m$, $v\in\Rbb^n$ are time-invariant unit vectors constituting the signal basis $G^{\mathrm{sig}}=\lambda uv^\top$, $\lambda > 0$ is the signal strength, and each element of $\Xi_t\in\Rbb^{m\times n}$ follows the standard normal $\mathcal{N}(0,\sigma_\Xi^2)$ independently.
\end{assumption}

\begin{theorem}[Pre-polar vs. Polar-only separation]
  \label{thm:separation}
  Assume $m > n$ without loss of generality.
  Under the rank-1 spiked model (\Cref{ass:rank-1}), we additionally assume the low SNR regime: $\lambda < 0.25\sigma_\Xi(\sqrt m - \sqrt n)$.
  Then, the signal alignment under Polar-only pipeline is bounded as follows:
  \[
    \E\left[\frac{\inner{\cO(G_t)}{G^\mathrm{sig}}_F}{\inner{\cO(G^\mathrm{sig})}{G^\mathrm{sig}}_F}\right]
    \le \frac1{\sqrt m} + \frac{4\lambda}{\sigma_\Xi(\sqrt m - \sqrt n)} + 4\sqrt n\exp\left(-\frac{c(\sqrt m - \sqrt n)^2}{4}\right),
  \]
  where $c>0$ is an absolute constant.
\end{theorem}

An important consequence of \Cref{thm:separation} is that the signal alignment of Polar-only pipeline becomes marginal as the matrix size $m\to\infty$ relative to the signal strength $\lambda$,
which is an extreme scenario of the low SNR regime.
In this case, the rank-1 signal $uv^\top$ is hidden in the bulk noise, and it will never be recovered.
Strikingly, Pre-polar pipeline continues to recover the target signal thanks to \Cref{thm:recovery}, as long as the effective window size $T$ is sufficiently large.
The proof is deferred to \Cref{app:proof-thm2}.

\section{Subspace Reliability and Pipeline Ordering Under Live Training}
\label{sec:experiments}

The stationary probe of \Cref{sec:spectral-gap} reflects the theoretical setting of \Cref{thm:spectral-gap,thm:recovery} by holding the model weights fixed.
To closely simulate a realistic gradient stream where the weights drift step-by-step,
we use the \emph{trajectory probe}, which maintains a sliding $K$-step gradient buffer alongside the optimizer during end-to-end NanoGPT training and applies the momentum pipelines to the buffer in collection order.
This sliding buffer creates the in-buffer mean gradient $\bar G$ (introduced in \Cref{sec:spectral-gap}), which is recomputed at each training step over the latest $K$ gradients and serves as the empirical proxy for the unobservable signal $G^\mathrm{sig}$.
Importantly, the model weights and thus the signal matrix $G^\mathrm{sig}$ can drift over time in the trajectory probe,
which serves as a sensitivity analysis of the time invariance supposed in \Cref{ass:decomp}.
We analyze the \texttt{h.0.attn.c\_proj} layer every $100$ training steps over $3$ seeds, with buffer size $K=\Kheadline$ and momentum coefficient $\beta=0.95$.
The full settings, the all-layer NanoGPT extension, and the CIFAR-10 trajectory extension appear in \Cref{app:exp-setup,app:cor1-supplement,app:thm2-supplement}, respectively.

\begin{figure}[htbp]
\centering
\includegraphics[width=\linewidth]{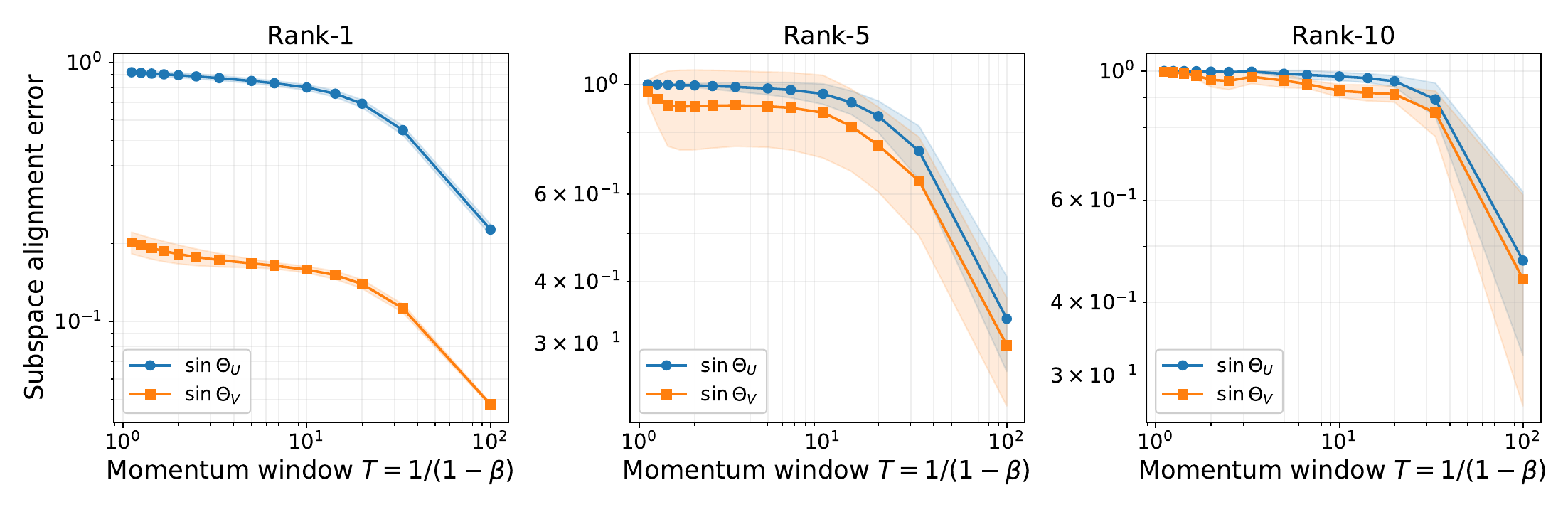}
\caption{Trajectory probe subspace alignment errors $\sin\Theta_U,\sin\Theta_V$ at ranks $r\in \{1,5,10\}$ versus the momentum window size $T=1/(1-\beta)$, with the $\beta$ grid restricted to $\beta\le 0.95$ ($T\le K/2$). Curves show seed means. Shaded bands show sample standard deviation across seeds.}
\label{fig:online-cor1}
\end{figure}

\paragraph{Subspace alignment error during training.} \Cref{fig:online-cor1} reports the subspace alignment errors $\sin\Theta_U$ and $\sin\Theta_V$ of $M_K^{(\beta)}$, which is the momentum buffer with decay $\beta$ after collecting a buffer of $K$ gradients, at ranks $r\in \{1,5,10\}$ versus the momentum window $T = 1/(1-\beta)$ on the layer \texttt{h.0.attn.c\_proj} during NanoGPT training.
The alignment errors averaged across different seeds decrease overall with $T$ at every rank, validating that the rank-$r$ subspace recovery (\Cref{cor:direction}) is robust beyond the time-invariant signal in \Cref{ass:decomp}(a).

\paragraph{Comparison of signal alignment for three pipelines.} \Cref{fig:online-thm2} reports the full-rank signal alignment (see \Cref{sec:spectral-gapII}) for the three pipelines on \texttt{h.0.attn.c\_proj} during live NanoGPT training. All three pipelines are evaluated at the last gradient of the current buffer. \Cref{fig:online-thm2-a} sweeps $\beta$ at training step~3000, and \cref{fig:online-thm2-b} tracks $\beta=0.95$ across training steps.

\begin{figure}[htbp]
  \centering
  \begin{minipage}[t]{0.49\linewidth}
    \centering
    \includegraphics[width=\linewidth]{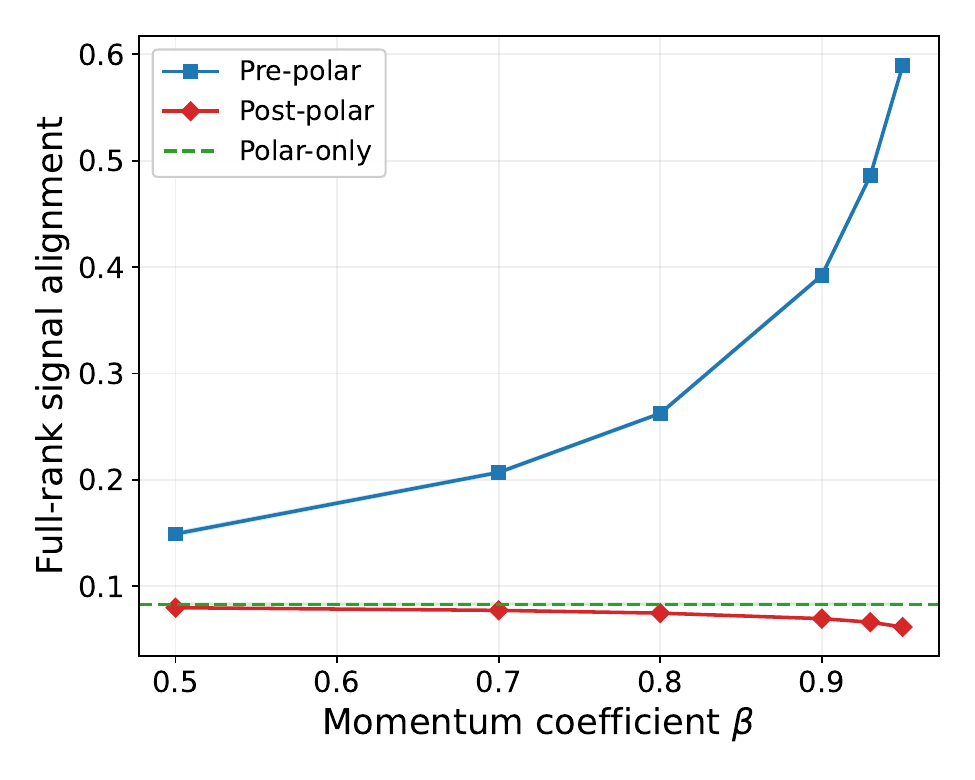}
    \subcaption{Full-rank alignment vs. $\beta$.}
    \label{fig:online-thm2-a}
  \end{minipage}\hfill
  \begin{minipage}[t]{0.49\linewidth}
    \centering
    \includegraphics[width=\linewidth]{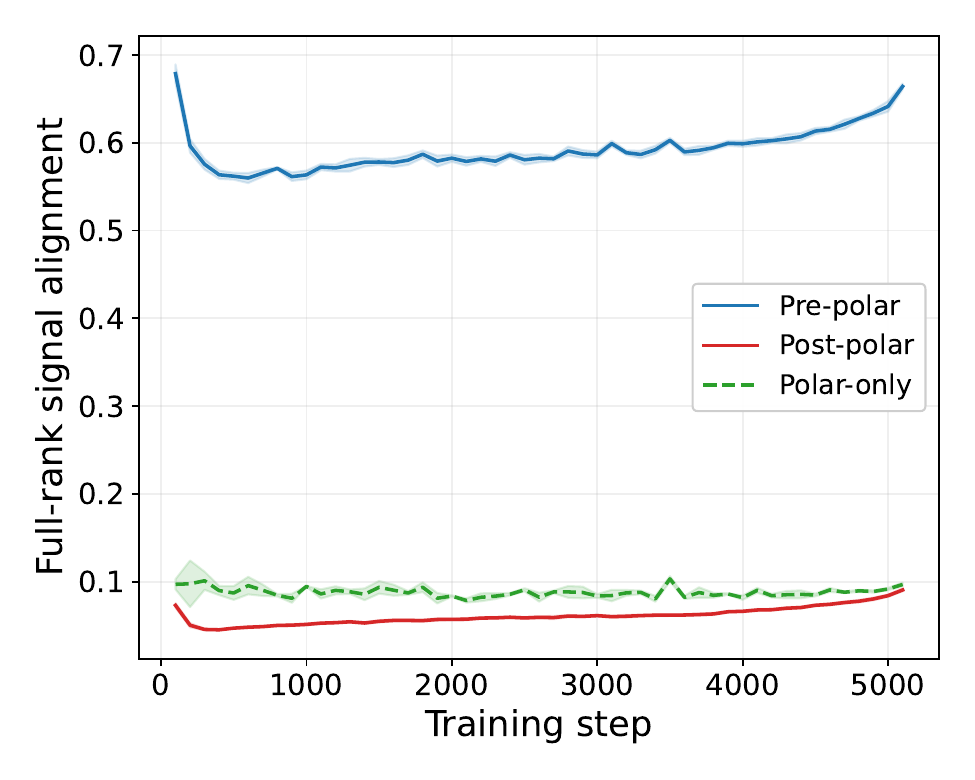}
    \subcaption{Full-rank signal alignment over training.}
    \label{fig:online-thm2-b}
  \end{minipage}
  \caption{Trajectory probe signal alignment for the three pipelines defined in~\cref{eq:forward}--\cref{eq:momefree} ($K=\Kheadline$, 3-seed mean). (a) $\beta$-sweep at step-3000 checkpoint of the full-rank signal alignment. (b) Full-rank signal alignment at $\beta=0.95$ across training steps. Curves show seed means. Shaded bands show sample standard deviation across seeds.}
  \label{fig:online-thm2}
\end{figure}

In \cref{fig:online-thm2-a}, Pre-polar full-rank alignment rises monotonically with $\beta$, while Post-polar and Polar-only stay nearly flat at low values across the entire $\beta$ range. Polar-only is $\beta$-independent by construction. The low Post-polar scores are consistent with the signal-alignment
gap established in \Cref{thm:recovery}. In \cref{fig:online-thm2-b}, Pre-polar advantage at $\beta=0.95$ persists across training steps. This trajectory result mirrors the stationary result of \cref{fig:frozen-thm2} supports that Pre-polar pipeline's dominance in signal alignment over both Post-polar and Polar-only survives the non-stationarity introduced by live training.


\section{Conclusion}
\label{sec:conclusion}

This paper explains why Muon orthogonalizes momentum rather than the raw gradient: momentum acts as a spectral filter that preserves the coherent signal while attenuating the bounded variance zero-mean orthogonal perturbation, so the polar step operates on a buffer in which the signal singular subspaces are already stabilized. Direct orthogonalization, or orthogonalizing each gradient before averaging, erases this structure and is provably worse in expected signal alignment. The same mechanism suggests a design rule that extends beyond Muon: in any matrix-aware optimizer that ends in a nonlinear spectral step, the spectral filter that separates signal from noise should come first --- \emph{denoise first, then orthogonalize}.

\paragraph{Limitations and future work.} Several extensions are left open. First, practical Muon training often increases the momentum coefficient during early training, which enlarges the effective window size across iterations. Formalizing this schedule-dependent story requires a time-varying filter analysis. Second, \Cref{thm:separation} is proved under the rank-1 spiked Gaussian model. Extending it to general rank-$r$ coherent signals or to non-Gaussian perturbations would require sharper concentration tools for the polar factor. Beyond these extensions, the spectral-gap mechanism may also guide the design of Muon variants with adaptive momentum updates.

\section*{Acknowledgment}
XL and ZZ are supported by JST SPRING (Japan Grant Number JPMJSP2104).
HB is supported by JST PRESTO (Grant Number JPMJPR24K6). 
A part of the experiments of this research was conducted using Wisteria/Aquarius in the Information Technology Center, the University of Tokyo.


\bibliographystyle{plainnat}
\bibliography{ref}

\newpage
\appendix

\crefalias{section}{appendix}
\crefalias{subsection}{appendix}
\crefalias{subsubsection}{appendix}
\onecolumn
\addcontentsline{toc}{section}{Appendix}
\renewcommand \thepart{}
\renewcommand \partname{}
\part{\Large{\centerline{Appendix}}}
\parttoc
\newpage


\section{Related Work}
\label{sec:related}

\paragraph{Matrix-based and LMO-based optimizers.}
Matrix-based optimizer has moved beyond classical coordinate-wise methods. Earlier precursors of spectral descent in restricted Boltzmann machines and deep networks include \citet{carlson2015stochastic,carlson2015preconditioned}. \citet{tuddenham2022orthogonalising} introduce Orthogonal-SGDM, which orthogonalizes each per-step gradient before the momentum averaging. Shampoo~\citep{gupta2018shampoo} and SOAP~\citep{vyas2025soap} use matrix-aware preconditioning, while Scion~\citep{pethick2025training} and Gluon~\citep{riabinin2025gluon} are LMO-based methods that operate in norm-constrained matrix geometry. Muon~\citep{jordan2024muon} applies a Newton-Schulz update to a momentum LLM training~\citep{liu2025muonscalablellmtraining}. Pion~\cite{shi2026pion} uses double-sided orthogonal transformation to design a matrix-based spectrum-preserving optimizer. A parallel line of work studies practical refinements of Muon's orthogonalization step, including the polar express that computes polar decomposition with optimal worst-case convergence~\citep{amsel2026the}, low-rank orthogonalization~\citep{he2025low}, block-wise orthogonalization~\citep{boreiko2025towards}, and row-momentum normalized preconditioning~\citep{deng2026rmnp}.

\paragraph{Theory of Muon and spectral updates.}
Recent studies have started to provide a theoretical foundation for the Muon optimizer and spectral updates. One line places Muon and related orthogonalization updates in norm-constrained geometry: \citet{chen2026muon} view Muon as a nuclear-norm Lion-K optimizer whose decoupled weight decay implicitly enforces spectral-norm constraints on the weights; \citet{kovalev2025understanding} develops a non-Euclidean trust-region interpretation of gradient orthogonalization and its momentum variant; and \citet{shulgin2026beyond} analyze Muon as an inexact LMO method, quantifying how Newton--Schulz approximation error couples with the step size and momentum coefficient. A second line studies the implicit bias of spectral descent and Muon, proving convergence to spectral-norm max-margin solutions on multiclass separable data~\citep{fan2025implicitbiasspectraldescent}. A third line asks when spectral orthogonalization should improve optimization, establishing preconditioning benefits in matrix factorization and in-context linear transformers~\citep{ma2026preconditioning}, isotropic-curvature conditions under which spectrum homogenization or orthogonalization is optimal~\citep{su2025isotropic}, and stable-rank/nuclear-rank conditions predicting when spectral updates outperform Euclidean gradient steps~\citep{davis2025spectral}. A fourth line analyzes structured tasks where per-component disparities are visible: Muon-type spectral updates balance learning across associative-memory frequencies and storage capacity~\citep{li2026muon,kim2026sharp}, heavy-tailed class distributions~\citep{wang2025muon}, imbalanced principal components~\citep{vasudeva2025muon}, and anisotropic curvature in phase retrieval~\citep{braun2026spectral}. Additionally, a recent research line~\cite{Liu2021OPT,qiu2025poet,qiu2026poetx,shi2026pion} studies how spectrum preservation can serve as a guiding principle to design optimizers. Together, these works clarify spectral-update geometry, implicit bias, and task-level scaling behavior from complementary angles.

\paragraph{Research progress on the momentum algorithm.}
Momentum is among the most extensively analyzed primitives in optimization, with theoretical study organized along several complementary aspects. In classical deterministic optimization, \citet{polyak1964momentum} show heavy-ball achieves an accelerated linear rate on smooth strongly convex quadratics, with a geometric and spectral exposition by~\citet{goh2017why}. Nesterov's accelerated gradient method attains the optimal $O(1/k^2)$ rate for smooth convex objectives~\citep{nesterov1983accelerated}. Later, \citet{sutskever2013on} extend these schemes to deep network training. In stochastic optimization, \citet{cutkosky2020momentum} prove that momentum eliminates the large-batch requirement of normalized SGD, and \citet{defazio2020momentum} reformulates SGD with momentum as primal averaging to obtain sharper non-convex convergence bounds. Inspired by signal processing theory, \citet{li2025momentum} interpret momentum in the frequency domain as a low-pass filter that amplifies low-frequency gradient components and attenuates high-frequency components, with high-frequency suppression becoming more important in late training. Together, these works examine momentum from deterministic-acceleration, stochastic-convergence, and signal-processing angles.

\paragraph{Gradient and noise structure in deep learning.}
Understanding the empirical structure of gradients and gradient noise during neural-network training has been a central concern of modern learning theory. On the gradient side, the Hessian spectrum of an overparametrized network often separates into a near-zero bulk and a few outliers, and the gradient can become largely concentrated in the corresponding top eigenspaces~\citep{sagun2017empirical,ghorbani2019investigation}. This is consistent with the observation that gradient descent often proceeds mostly within a small subspace spanned by the top Hessian eigenvectors~\citep{gur2018gradient}. In high-dimensional feature-learning regimes, the first gradient update of the first-layer weights can contain a rank-1 spike aligned with the target signal~\citep{ba2022high}. Relatedly, under structured input distributions such as spiked covariance models or multiple-index teacher models, gradient-based training can exploit the hidden low-dimensional structure and drive the learned representation toward the corresponding signal subspace~\citep{mousavi2023gradient}. On the noise side, empirical studies suggest that mini-batch gradient noise can be highly non-Gaussian and heavy-tailed~\citep{simsekli2019tail}. Orthogonally, SGD noise has been shown to be anisotropic, with its covariance aligned with the Hessian in a way that facilitates escape from sharp minima~\citep{zhu2019anisotropic}. Together, these threads establish that gradient and noise structure in modern deep learning training.

\section{Setup Conventions and Variants}
\label{app:setup-conventions}

\Cref{sec:setup} adopts the EMA normalization~\cref{eq:ema} and zero initialization $M_0=0$. We record here (i) the \emph{effective sample size} identity that underlies the perturbation variance bound used in \Cref{thm:spectral-gap,thm:recovery}; (ii) the polar-factor equivalence between EMA and Polyak/heavy-ball momentum that justifies the EMA normalization choice; (iii) the buffer's behavior under arbitrary initialization, which defines the post-transient threshold $t\ge c_0 T$ characterizing when the EMA buffer has reached its asymptotic regime; and (iv) the kernel weights and effective sample size of practical Nesterov momentum.

\subsection{Effective Sample Size of the Momentum Buffer}
\label{app:effective-sample-size}

The perturbation bounds in \Cref{thm:spectral-gap,thm:recovery} are driven by a single derived scalar quantity,
\begin{equation}
\label{eq:effective-sample-size}
N_{\mathrm{eff}}
 \coloneqq \frac{(\sum_{s\ge0}h_s)^2}{\sum_{s\ge0}h_s^2}
=\frac{1+\beta}{1-\beta}=2T-1,
\end{equation}
the \emph{effective sample size} of the momentum buffer. For a simple average of $N$ uncorrelated, mean-zero random variables, the variance scales down by a factor of $1/N$.
The bounded variance zero-mean orthogonal perturbation sequence $\{\Xi_t\}$ is pairwise uncorrelated in the Frobenius inner product ($\e{\inner{\Xi_s}{\Xi_t}_F}=0$ for $s\neq t$, see \Cref{ass:decomp}(b) and \Cref{prop:ema-concentration}), so the variance of the momentum-weighted perturbation $ S_t \coloneqq\sum_{s\ge 0} h_s \Xi_{t-s}$ is reduced by $\sum_{s\ge 0} h_s^2 = 1/(2T-1)$~\citep{busbridge2023scale} (see \cref{eq:effective-sample-size}).
Thus, the momentum buffer with decay coefficient $\beta$ provides the same variance reduction as averaging over $2T-1$ independent samples. This factor determines the perturbation bound in \Cref{thm:spectral-gap,thm:recovery}.

\subsection{Polar Factor Equivalence with Heavy-ball Momentum}
\label{app:gain-invariance}

For the decay $\beta\in[0,1)$ and scalar $\gamma>0$, define the general momentum recursion
\begin{equation}
M_t=\beta M_{t-1}+\gamma G_t,
\label{eq:standardmomentum}
\end{equation}
where $G_t$ is the incoming fresh gradient matrix.
Under zero initialization $M_0=0$, the recursion unrolls to
\[
M_t=\sum_{s=0}^{t-1} h_s\, G_{t-s},\qquad \text{where}\qquad h_s\coloneqq\gamma\beta^s.
\]
The decay $\beta$ sets how far back the buffer effectively remembers, while $\gamma$ sets the overall scale. The two standard normalizations are
\[
\text{EMA momentum~\citep{gardner1985exponential}: } \gamma=1-\beta,
\qquad
\text{Polyak (heavy-ball) momentum~\citep{sutskever2013on}: } \gamma=1.
\]
Pre-polar update~\cref{eq:forward} applies the polar factor $\cO(\cdot)$ to the buffer $M_t = \sum_{s=0}^{t-1} h_s\, G_{t-s}$. Because the polar factor is scale-invariant, $\cO(\alpha X)=\cO(X)$ for every $\alpha>0$, we may divide $M_t$ by the positive scalar $A \coloneqq \sum_{s'\ge 0}h_{s'}=\gamma/(1-\beta)$ without changing the polar factor:
\[
\cO(M_t)
=\cO\!\left(\frac{M_t}{A}\right)
=\cO\!\left(\sum_{s=0}^{t-1}\frac{h_s}{A}\,G_{t-s}\right)
=\cO\!\left(\sum_{s=0}^{t-1}(1-\beta)\beta^s\, G_{t-s}\right),
\]
where the last equality uses $h_s/A = \gamma\beta^s\cdot(1-\beta)/\gamma = (1-\beta)\beta^s$. The scalar $\gamma$ drops out, and thus every first-order recursion with the same decay $\beta$ produces the same Pre-polar update. In particular, Polyak/heavy-ball momentum ($\gamma=1$) and EMA momentum ($\gamma=1-\beta$) generate identical Pre-polar updates. The EMA normalization $\gamma=1-\beta$ adopted in~\cref{eq:ema} picks the unique $\gamma$ for which the infinite kernel sums to one ($\sum_{s\ge 0}h_s=1$), so the buffer can be read directly as a weighted average of past gradients.

\begin{remark}[Polar factor equivalence of EMA and heavy-ball momentum]
\label{rem:standard-momentum}
Let $M_t^{(\beta,\gamma)}$ denote the zero-initialized buffer generated by~\cref{eq:standardmomentum}, and let $M_t^{\mathrm{ema}(\beta)}=(1-\beta)\sum_{s=0}^{t-1}\beta^s G_{t-s}$ be the EMA-normalized buffer with the same decay. Then $M_t^{(\beta,\gamma)}=\frac{\gamma}{1-\beta}M_t^{\mathrm{ema}(\beta)}$, so $\cO(M_t^{(\beta,\gamma)})=\cO(M_t^{\mathrm{ema}(\beta)})$ for every $\gamma>0$.
\end{remark}

\subsection{Initialization and the Post-Transient Regime}
\label{app:initialization}

Under an arbitrary initialization $M_0=M_{\mathrm{init}}$, the recursion~\cref{eq:ema} unrolls to
\[
M_t = \beta^t\,M_{\mathrm{init}} + (1-\beta)\sum_{s=0}^{t-1}\beta^s G_{t-s},
\]
so the buffer carries an additional $\beta^t M_{\mathrm{init}}$ term, an exponentially decaying transient of operator-norm size at most $\beta^t\snorm{M_{\mathrm{init}}}$. We say the recursion has entered its \emph{post-transient regime} when $t\ge c_0 T$ for a constant $c_0>0$ chosen so that $\beta^{c_0 T}\snorm{M_{\mathrm{init}}}$ is negligible relative to the signal scale. 

The same threshold $t\ge c_0 T$ also governs the EMA's warmup under the default zero initialization. 
With $M_{\mathrm{init}}=0$, the finite-time buffer $M_t=(1-\beta)\sum_{s=0}^{t-1}\beta^s G_{t-s}$ carries total kernel weight $1-\beta^t$ rather than $1$,
so the buffer recovers a fraction $1-\beta^t$ of a constant signal, which would violate the persistence floor $c_{\mathrm{sig}}\lambda_k$ in \Cref{ass:persistence} during the warmup phase.
The condition $t\ge c_0 T$ ensures this $O(\beta^t)$ error is negligible, recovering the same regime as in the arbitrary-initialization case.

\subsection{Practical Nesterov: Kernel Weights and Effective Sample Size}
\label{app:nesterov}

Practical deep-learning Nesterov momentum modifies the EMA recursion~\cref{eq:ema} by passing a linear combination of the buffer $M_t$ and the current gradient $G_t$ to the polar factor, rather than $M_t$ itself. With the general first-order recursion $M_t=\beta M_{t-1}+\gamma G_t$ and scalars $\nu,\kappa$, the parameter update is formed from
\[
N_t \;=\; \nu\,M_t + \kappa\,G_t.
\]
Under zero initialization $M_0=0$, the buffer unrolls to $M_t=\sum_{s=0}^{t-1}\gamma\beta^s G_{t-s}$, so
\[
N_t \;=\; (\nu\gamma+\kappa)\,G_t \;+\; \sum_{s=1}^{t-1} \nu\gamma\,\beta^s\, G_{t-s}
\;=\; \sum_{s=0}^{t-1} h_s\,G_{t-s},
\qquad
h_0=\nu\gamma+\kappa,\quad h_s=\nu\gamma\,\beta^s\ (s\ge 1).
\]

\paragraph{Non-equivalence with plain momentum.}
\Cref{app:gain-invariance} shows that $\cO(M_t)$ depends only on the decay $\beta$, since multiplying all geometric weights $\gamma\beta^s$ by the same positive constant leaves the polar factor unchanged. For $N_t$, however, changing $\gamma$ while keeping $\kappa$ fixed changes the relative weight of the current gradient versus the geometric tail. Thus the kernel shape is not determined by $\beta$ alone. In the nondegenerate case $\nu\gamma\neq 0$ and $0<\beta<1$, matching plain-momentum geometric weights $\tilde h_s=\tilde\gamma\tilde\beta^s$ to $h_s$ for all $s\ge1$ forces $\tilde\beta=\beta$ and $\tilde\gamma=\nu\gamma$. Matching the current-gradient weight then requires $h_0=\nu\gamma+\kappa=\tilde\gamma$, hence $\kappa=0$. Therefore practical Nesterov reduces to plain momentum exactly when $\kappa=0$.

\paragraph{EMA-normalized Muon Nesterov.}
The common implementation uses $\gamma=1-\beta$, $\nu=\beta$, $\kappa=1-\beta$, giving
\[
h_0 \;=\; 1-\beta^2, \qquad h_s \;=\; (1-\beta)\beta^{s+1}\quad (s\ge 1).
\]
In the stationary/infinite-tail approximation, the kernel mass is one,
\[
\sum_{s\ge 0} h_s
\;=\; (1-\beta^2) \;+\; (1-\beta)\beta^2\sum_{s\ge 0}\beta^s
\;=\; (1-\beta^2) + \beta^2 \;=\; 1,
\]
so $N_t$ remains a weighted average of past gradients. The squared kernel mass, which controls the variance proxy of \Cref{prop:ema-concentration} through \cref{eq:effective-sample-size}, is
\[
\sum_{s\ge 0} h_s^2
\;=\; (1-\beta^2)^2 + \frac{(1-\beta)^2\beta^4}{1-\beta^2}
\;=\; \frac{(1-\beta)\,(1+2\beta-2\beta^3)}{1+\beta},
\]
so the effective sample size becomes
\[
N_{\mathrm{eff}}^{\mathrm{nest}}
\;=\; \frac{1}{\sum_{s\ge 0} h_s^2}
\;=\; \frac{1+\beta}{(1-\beta)\,(1+2\beta-2\beta^3)}
\;<\; \frac{1+\beta}{1-\beta}
\;=\; 2T-1,
\]
with strict inequality because $1+2\beta-2\beta^3 = 1+2\beta(1-\beta^2)>1$ for $\beta\in(0,1)$.

Compared with plain momentum at the same $\beta$, practical Nesterov has a strictly smaller effective sample size and therefore attenuates the momentum-filtered perturbation in \Cref{prop:ema-concentration} less aggressively, while placing more weight on the current gradient ($h_0=1-\beta^2 > 1-\beta$). This trade-off explains why enabling Nesterov can help at fixed hyperparameters, while retuned plain momentum can produce very similar training curves in practice~\citep{wen2025fantastic}.

\section{Proofs in \Cref{sec:spectral-gap}}
\label{app:proof-thm1-cor-1}

To control the perturbation terms in \Cref{thm:spectral-gap}, we first establish the following concentration estimate.

\begin{proposition}[Concentration of the momentum-filtered perturbation]
\label{prop:ema-concentration}
Let $\{\Xi_t\}_{t\in \mathbb Z}$ be an $m\times n$ matrix-valued perturbation sequence that satisfies \Cref{ass:decomp}(b)
for $t\ge 0$, and $\Xi_t = 0$ for $t<0$.
Define
\[
S_t \coloneqq (1-\beta)\sum_{s\ge0}\beta^s \Xi_{t-s},\qquad T:=\frac{1}{1-\beta}.
\]
Then for every $u > 0$, we have
\[
\prob{\fnorm{S_t} \ge \sqrt{\frac{\eta}{2T-1}} u} \le \frac{1}{u^2}.
\]
As a corollary,
\[
\prob{\opnorm{S_t} \ge \sqrt{\frac{\eta}{2T-1}} u} \le \frac{1}{u^2}.
\]
\end{proposition}

\begin{proof}[Proof of \Cref{prop:ema-concentration}]
By \Cref{ass:decomp}(b) the perturbation $\{\Xi_t\}$ is pairwise uncorrelated since orthogonal in the Frobenius inner product, $\e{\inner{\Xi_{t_1}}{\Xi_{t_2}}_F} = 0$ for $t_1 \neq t_2$.

Then we have
\[
\var[F]{S_t} = \e{\fnorm{S_t}^2} = \sum_{s=0}^\infty h_s^2 \e{\fnorm{\Xi_{t-s}}^2} \le \frac{\eta}{2T-1}
\]
by \cref{eq:effective-sample-size}.

By Chebyshev's bound we have
\[
\prob{\fnorm{S_t} \ge \sqrt{\frac{\eta}{2T-1}} u} \le \frac{1}{u^2}.
\]
The corollary is a direct consequence of the fact that the Frobenius norm is an upper bound of the operator norm.
\end{proof}

\subsection{Proof of \Cref{thm:spectral-gap}}
\label{proof:thm1}
\begin{proof}
By the decomposition~\cref{eq:decomp}, defining $G_t = 0$ for $t<0$, we have
\begin{align*}
M_t
&= (1-\beta)\sum_{s\ge0} \beta^s G_{t-s}
\\&= (1-\beta)\sum_{s\ge0} \beta^s G_{t-s}^{\mathrm{sig}} + (1-\beta)\sum_{s\ge0} \beta^s \Xi_{t-s}
\\&= \sum_{i=1}^r \bar\lambda_i(t) u_i v_i^\top + S_t.
\end{align*}
Let $M_t^{\mathrm{sig}}$ denote $\sum_{i=1}^r \bar\lambda_i(t) u_i v_i^\top$.

\noindent\textbf{Signal preservation.}
For $k=1,\ldots,r$, by the min-max theorem of singular values,
the $k$-th singular value of $M_t$ is exactly
\[
\sigma_k(M_t) = \max_{\substack{S\subset\Rbb^n\\\dim(S)=k}} \min_{\substack{x\in S\\\|x\|=1}} \|M_t x\|.
\]
We may take $S=\operatorname{span}\{v_1,\ldots,v_k\}$. Then, for every unit vector $x\in S$, the triangle inequality and the definition of the operator norm imply
\[
\|M_t x\|
\ge \|M_t^{\mathrm{sig}} x\| - \|S_t x\|
\ge \min_{1\le i\le k}|\bar\lambda_i(t)| - \opnorm{S_t}
\ge c_{\mathrm{sig}}\lambda_k - \opnorm{S_t}.
\]
Here the first inequality uses the triangle inequality. The second uses the orthonormality of $\{u_i\}$ and $\{v_i\}$ together with $\|S_t x\|\le \opnorm{S_t}$ for $\|x\|=1$. The last uses~\Cref{ass:persistence} and $\lambda_1\ge\cdots\ge\lambda_r$.
Eventually, \Cref{prop:ema-concentration} leads to the following bound:
\begin{align*}
\prob{\sigma_k(M_t) \le c_{\mathrm{sig}}\lambda_k - \sqrt{\frac{\eta}{2T-1}} u}
\le
\prob{\opnorm{S_t} > \sqrt{\frac{\eta}{2T-1}} u}
\le \frac{1}{u^2}.
\end{align*}

\noindent\textbf{Perturbation attenuation.}
Similarly, for $k=r+1,\ldots, n$, by the min-max theorem,
\[
\sigma_k(M_t) = \min_{\substack{S\subset\Rbb^n\\\dim(S)=n-k+1}} \max_{\substack{x\in S\\\|x\|=1}} \|M_t x\|.
\]
We may take $S$ be one arbitrary subspace of $\operatorname{span}\{v_1,\ldots,v_{k-1}\}^\perp$
so that $M_t$ restricted to $S$ is zero, then for every unit $x\in S$,
\[
\|M_t x\| \le \|M_t^{\mathrm{sig}} x\| + \|S_t x\|
\le 0 + \opnorm{S_t}.
\]
Then
\begin{align*}
\prob{\sigma_k(M_t) \ge \sqrt{\frac{\eta}{2T-1}} u}
\le
\prob{\opnorm{S_t} \ge \sqrt{\frac{\eta}{2T-1}} u}
\le \frac{1}{u^2}.
\end{align*}

Taking $u = (2T-1)^{\frac14}$ concludes the proof.
\end{proof}

\subsection{Proof of \Cref{cor:direction}}
\begin{proof}
Apply Wedin's $\sin\Theta$ theorem for singular subspaces \citep{wedin1972perturbation} to the decomposition
\[
M_t=M_t^{\mathrm{sig}}+R_t.
\]
The clean matrix $M_t^{\mathrm{sig}}$ has left and right signal subspaces exactly $\operatorname{span}(U)$ and $\operatorname{span}(V)$, and we have
\[
\sigma_r(M_t^{\mathrm{sig}})-\sigma_{r+1}(M_t)
\ge
c_{\mathrm{sig}}\lambda_r - \frac{\sqrt\eta}{(2T-1)^{1/4}}
\]
with probability at least $1 - (2T-1)^{-1/2}$ by \Cref{thm:spectral-gap}.
Wedin's $\sin\Theta$ theorem therefore gives,
\[
\max\bigl\{\|\sin\Theta(\hat U_t,U)\|_2,\;\|\sin\Theta(\hat V_t,V)\|_2\bigr\}
\le \frac{\snorm{R_t}}{\sigma_r(M_t^{\mathrm{sig}})-\sigma_{r+1}(M_t)}
\le \frac{\sqrt\eta/(2T-1)^{1/4}}{c_{\mathrm{sig}}\lambda_r - \sqrt\eta/(2T-1)^{1/4}},
\]
where the last inequality uses $\snorm{R_t}=\snorm{S_t}\le\sqrt\eta/(2T-1)^{1/4}$ from \Cref{prop:ema-concentration} on the same event. Both individual bounds in \Cref{cor:direction} follow.
\end{proof}

\section{Proofs in \Cref{sec:spectral-gapII}}
\label{app:proof-thm2}

Recall the notation in \Cref{ass:decomp} and \Cref{ass:persistence}.
For the ordering theorem, we additionally assume \Cref{ass:time-invariant-signal}.

The following key lemma describes how orthogonalization over noise creates a bias that prevents recovery of the signal direction.

\begin{lemma}[Quantitative gap for the polar factor]
\label{lem:polar-gap}
Let $G\in \mathbb{R}^{m\times n}$ with $m\ge n$ and $\operatorname{rank}(G)=r$. $G$ has thin SVD
\[
G = U_r \Sigma_r V_r^\top, \qquad
\Sigma_r = \operatorname{diag}(\sigma_1,\dots,\sigma_r, 0, \ldots, 0),
\qquad
\text{and }\sigma_1 \ge \cdots \ge \sigma_r > 0,
\]
where $U_r\in \mathbb R^{m\times n}$, $V_r\in \mathbb R^{n\times n}$, and $\Sigma_r\in\mathbb R^{n\times n}$.

Let \(\Xi\in\Rbb^{m\times n}\) be any random matrix.
Then we have
\[
\e{\inner{\cO(G+\Xi)}{G}_F}
\le
\inner{\cO(G)}{G}_F
-
\frac{\sigma_r}{2}
\e{\fnorm{\bigl(\cO(G+\Xi)-\cO(G)\bigr)V_r}^2}.
\]
\end{lemma}

\begin{proof}
Fix a deterministic matrix \(X\in\Rbb^{m\times n}\), and write
\[
Q  \coloneqq  \cO(X).
\]
Since the singular values of \(\cO(X)\) are all either 0 or 1, one has \(\opnorm{Q}\le 1\). Moreover,
\[
\cO(G)=U_rV_r^\top.
\]
Define
\[
B  \coloneqq  U_r^\top Q V_r \in \Rbb^{r\times r}.
\]
Then \(\opnorm{B}\le 1\), since \(U_r\) and \(V_r\) have orthonormal columns.

Now
\[
\inner{Q}{G}_F
=
\operatorname{Tr}(Q^\top U_r \Sigma_r V_r^\top)
=
\operatorname{Tr}(B^\top \Sigma_r),
\]
hence
\[
\inner{\cO(G)}{G}_F - \inner{Q}{G}_F
=
\operatorname{Tr}(\Sigma_r) - \operatorname{Tr}(B^\top \Sigma_r)
=
\sum_{i=1}^r \sigma_i (1-B_{ii}).
\]
Since \(\sigma_i \ge \sigma_r\) for all \(i\),
\[
\inner{\cO(G)}{G}_F - \inner{Q}{G}_F
\ge
\sigma_r \sum_{i=1}^r (1-B_{ii})
=
\sigma_r \bigl(r-\operatorname{Tr}(B)\bigr).
\]
On the other hand,
\[
\fnorm{(Q-\cO(G))V_r}^2
=
\fnorm{QV_r-U_r}^2
=
\fnorm{QV_r}^2 + \fnorm{U_r}^2 - 2\operatorname{Tr}(U_r^\top QV_r).
\]
Because \(\opnorm{Q}\le 1\) and \(V_r\) has orthonormal columns,
\[
\fnorm{QV_r}^2 \le r,
\qquad
\fnorm{U_r}^2 = r,
\qquad
\operatorname{Tr}(U_r^\top QV_r)=\operatorname{Tr}(B).
\]
Therefore
\[
\fnorm{(Q-\cO(G))V_r}^2
\le
2r - 2\operatorname{Tr}(B),
\]
or equivalently,
\[
r-\operatorname{Tr}(B)
\ge
\frac12 \fnorm{(Q-\cO(G))V_r}^2.
\]
Combining the last two estimates yields the pointwise bound
\[
\inner{\cO(G)}{G}_F - \inner{\cO(X)}{G}_F
\ge
\frac{\sigma_r}{2}\,
\fnorm{(\cO(X)-\cO(G))V_r}^2.
\]
Applying this with \(X=G+\Xi\) and taking expectations proves the claim.
\end{proof}

\begin{remark}
\label{rem:always-positive-gap}
Let
\[
\mathcal{N}_G
 \coloneqq 
\{X\in\Rbb^{m\times n} : \cO(X)=\cO(G)\}.
\]
We have
\[
\cO(X)  \coloneqq  X (X^\top X)^{\dagger/2},
\]
where \((\cdot)^\dagger\) denotes the Moore–Penrose pseudoinverse.
Thus,
\begin{align*}
\phantom{\iff}& \cO(X) = \cO(G)
\\ \iff &
X (X^\top X)^{\dagger/2} = U_r V_r^\top
\\ \iff &
X = U_r V_r^\top (X^\top X)^{1/2}.
\end{align*}
A necessary condition for the last equation to hold is that
\begin{equation}\label{eq:necc-set}
X X^\top = U_r V_r^\top (X^\top X) V_r U_r^\top.
\end{equation}
Thus, $\mathcal N_G$ is a subset of the solution set to \cref{eq:necc-set}. The latter is a collection of non-trivial algebraic equations (i.e. non-zero polynomial equations in the entries of $X$)
whose solution set has Lebesgue measure zero \citep[\S2.8]{BochnakCosteRoy1998}.
In particular, if $\Xi$ is a continuous random matrix, then $\prob{\cO(G+\Xi)=\cO(G)}=0$. Thus the gap in \Cref{lem:polar-gap} is positive with probability one under continuous perturbations.

In conclusion, as long as the law of $\Xi$ is absolutely continuous with respect to the Lebesgue measure, we have a strictly positive gap in \Cref{lem:polar-gap}.
\end{remark}

The following lemma is used to control the bias for Pre-polar pipeline.
\begin{lemma}
\label{lem:inner-product-perturbation-bound}
Let $G\in \mathbb{R}^{m\times n}$ with $m\ge n$ and $\operatorname{rank}(G)=r$. $G$ has thin SVD
\[
G = U_r \Sigma_r V_r^\top, \qquad
\Sigma_r = \operatorname{diag}(\sigma_1,\dots,\sigma_r, 0, \ldots, 0),
\qquad
\text{and }\sigma_1 \ge \cdots \ge \sigma_r > 0,
\]
where $U_r\in \mathbb R^{m\times n}$, $V_r\in \mathbb R^{n\times n}$, and $\Sigma_r\in\mathbb R^{n\times n}$.
For any $R \in \Rbb^{m\times n}$, we have
\[
\frac{\langle \mathcal O(G+R), G\rangle_F}{\langle \mathcal O(G), G\rangle_F}
\ge
1 - \frac{2n\opnorm{R}}{\inner{\cO(G)}{G}_F}.
\]
\end{lemma}

\begin{proof}
In this proof, we write $\|X\|_* = \sum_i \sigma_i(X)$ for the nuclear norm, the sum of all singular values.
Since $\mathcal O(G)=U_rV_r^\top$,
\[
\langle \mathcal O(G), G\rangle_F
=
\langle U_rV_r^\top, U_r\Sigma_rV_r^\top\rangle_F
=
\Tr(\Sigma_r)
=
\norm{G}_*.
\]

Now let $B=G+R$.
\[
\langle \mathcal O(B), G\rangle_F
=
\langle \mathcal O(B), B-R\rangle_F
=
\norm{B}_* - \langle \mathcal O(B), R\rangle_F.
\]
By the triangle inequality,
\[
\norm{B}_*\ge\norm{G}_* - \norm R_*.
\]
Since $R$ has rank at most $n$, we have $\norm{R}_*\le n\opnorm R$, hence
\[
\norm{B}_*\ge\norm{G}_* - n\opnorm R.
\]
Also, $\mathcal O(B)$ has at most $n$ singular values equal to $1$, so
\[
\|\mathcal O(B)\|_* \le n.
\]
Hence, by duality of nuclear and operator norms,
\[
|\langle \mathcal O(B), R\rangle_F|
\le
\|\mathcal O(B)\|_* \opnorm R
\le
n \opnorm R.
\]
Therefore
\[
\langle \mathcal O(B), G\rangle_F
\ge
\|G\|_* - 2n\opnorm R.
\]
Dividing by $\langle \mathcal O(G), G\rangle_F=\|G\|_*$ gives
\[
\frac{\langle \mathcal O(G+R), G\rangle_F}{\langle \mathcal O(G), G\rangle_F}
\ge
1 - \frac{2n\opnorm R}{\|G\|_*}.
\]
\end{proof}

\subsection{Proof of \Cref{thm:recovery}}
\label{proof:thm2}
\begin{proof}
\textbf{Part (i).}
By \Cref{lem:polar-gap} and \Cref{rem:always-positive-gap}, there exists a constant $C>0$ such that
\[
\e{
\frac{\inner{\cO(G_t)}{G^\mathrm{sig}}_F}{\inner{\cO(G^\mathrm{sig})}{G^\mathrm{sig}}_F}
}
\le
1 - C.
\]

\textbf{Part (ii).}
By the definition of $\widetilde M_t$, \Cref{lem:polar-gap}, and \Cref{rem:always-positive-gap}, for the same constant $C>0$ we have
\[
\e{
\frac{\inner{\widetilde M_t}{G^\mathrm{sig}}_F}
{\inner{\cO(G^{\mathrm{sig}})}{G^\mathrm{sig}}_F}
}
=
(1-\beta)\sum_{s=0}^t\beta^s \e{
\frac
{\inner{\cO(G_{t-s})}{G^\mathrm{sig}}_F}
{\inner{\cO(G^{\mathrm{sig}})}{G^\mathrm{sig}}_F}
}
\le
(1-\beta)\sum_{s\ge0}\beta^s (1-C)
=
1 - C.
\]
Now consider the \Cref{ass:pairwise-independence}.
Since for any $X$, $\left|\inner{\cO(X)}{G^\mathrm{sig}}_F\right| \le \inner{\cO(G^\mathrm{sig})}{G^\mathrm{sig}}_F$,
the variance
\[
\var{\inner{\cO(G_t)}{G^\mathrm{sig}}_F}\le \inner{\cO(G^\mathrm{sig})}{G^\mathrm{sig}}_F^2
\]
for every $t\ge 0$.
So we have
\begin{align*}
    \var{\inner{\widetilde M_t}{G^\mathrm{sig}}_F}
    &=
    \var{(1-\beta)\sum_{s=0}^t\beta^s \inner{\cO(G_{t-s})}{G^\mathrm{sig}}_F}
    \\&=
    (1-\beta)^2\sum_{s=0}^t\beta^{2s} \var{\inner{\cO(G_{t-s})}{G^\mathrm{sig}}_F}
    \\&\le
    (1-\beta)^2\sum_{s=0}^t\beta^{2s} \inner{\cO(G^\mathrm{sig})}{G^\mathrm{sig}}_F^2
    \\&\le
    (1-\beta)^2\cdot\frac{1}{1-\beta^2}\cdot\inner{\cO(G^\mathrm{sig})}{G^\mathrm{sig}}_F^2
    \\&=
    \frac{1-\beta}{1+\beta} \inner{\cO(G^\mathrm{sig})}{G^\mathrm{sig}}_F^2
    \\&=
    \frac{1}{2T-1} \inner{\cO(G^\mathrm{sig})}{G^\mathrm{sig}}_F^2.
\end{align*}

By Chebyshev's inequality,
\begin{align*}
&\phantom{{}={}}
\prob{
\frac{\inner{\widetilde M_t}{G^\mathrm{sig}}_F}{\inner{\cO(G^\mathrm{sig})}{G^\mathrm{sig}}_F}
\ge 1 - \frac{C}{2}
}
\\&=
\prob{
\inner{\widetilde M_t}{G^\mathrm{sig}}_F
\ge
\left(1 - \frac{C}{2}\right) \inner{\cO(G^\mathrm{sig})}{G^\mathrm{sig}}_F
}
\\&=
\prob{
\inner{\widetilde M_t}{G^\mathrm{sig}}_F - \e{\inner{\widetilde M_t}{G^\mathrm{sig}}_F}
\ge
\left(1 - \frac{C}{2}\right) \inner{\cO(G^\mathrm{sig})}{G^\mathrm{sig}}_F - \e{\inner{\widetilde M_t}{G^\mathrm{sig}}_F}
}
\\&\le
\prob{
\inner{\widetilde M_t}{G^\mathrm{sig}}_F - \e{\inner{\widetilde M_t}{G^\mathrm{sig}}_F}
\ge
\frac{C\sqrt{2T-1}}{2} \frac{\inner{\cO(G^\mathrm{sig})}{G^\mathrm{sig}}_F}{\sqrt{2T-1}}
}
\\&\le
\frac{4}{C^2(2T-1)}.
\end{align*}

\noindent\textbf{Part (iii).}
By \Cref{lem:inner-product-perturbation-bound}, we have
\[
\e{\frac{\inner{\cO(M_t)}{G^\mathrm{sig}}_F}{\inner{\cO(G^\mathrm{sig})}{G^\mathrm{sig}}_F}}
\ge
1 -
\e{\frac{2n\|M_t - G^\mathrm{sig}\|_{\op}}{\inner{\cO(G^\mathrm{sig})}{G^\mathrm{sig}}_F}}.
\]

Similar to the proof of \Cref{prop:ema-concentration}, we have
\begin{align*}
    \e{\fnorm{M_t - G^{\mathrm{sig}}}^2}
    &=
    \e{\fnorm{(1-\beta)\sum_{s=0}^t\beta^s \Xi_{t-s}}^2}
    \\&=
    (1-\beta)^2\sum_{s=0}^t\beta^{2s} \e{\fnorm{\Xi_{t-s}}^2}
    \\&\le
    \frac{1}{2T-1}\,\eta,
\end{align*}
where the second equality uses pairwise uncorrelatedness and the inequality uses $\E\fnorm{\Xi_{t-s}}^2 \le \eta$ from \Cref{ass:decomp}(b).

By Lyapunov's inequality,
\[
\e{\|M_t - G^\mathrm{sig}\|_{\op}}
\le
\e{\|M_t - G^\mathrm{sig}\|_{\op}^2}^{\frac12}
\le
\frac{\sqrt{\eta}}{(2T-1)^{\frac12}}.
\]
Therefore,
\[
\e{\frac{\inner{\cO(M_t)}{G^\mathrm{sig}}_F}{\inner{\cO(G^\mathrm{sig})}{G^\mathrm{sig}}_F}}
\ge
1 -
\frac{2n\sqrt{\eta}}{(2T-1)^{\frac12}\inner{\cO(G^\mathrm{sig})}{G^\mathrm{sig}}_F}.
\]
By taking $C' = \frac{2n\sqrt{\eta}}{\inner{\cO(G^\mathrm{sig})}{G^\mathrm{sig}}_F}$, we obtain~\cref{eq:forward-align}.

By Markov's inequality,
\begin{align*}
&\phantom{{}={}}
\prob{
\frac{\inner{\cO(M_t)}{G^\mathrm{sig}}_F}{\inner{\cO(G^\mathrm{sig})}{G^\mathrm{sig}}_F}
\le 1 - \frac{C'}{(2T-1)^{\frac14}}
}
\\&\le
\prob{
\|M_t - G^\mathrm{sig}\|_{\op}
\ge
\frac{\sqrt\eta}{(2T-1)^{\frac14}}
}
\\&\le
\e{\|M_t - G^\mathrm{sig}\|_{\op}}\cdot (2T-1)^{\frac14}
\\&\le
\frac{1}{(2T-1)^{\frac14}}.
\end{align*}

\noindent\textbf{Part (iv).}
Taking $T_0$ such that $\frac{C'}{\sqrt{2T_0-1}}\le \frac{C}{2}$ proves the first claim.
The second claim is a union bound of the probability bounds for Pre-polar and Post-polar pipelines.
\end{proof}

\subsection{Proof of \Cref{thm:separation}}
\label{proof:thm3}
\begin{proof}
  The rank-1 spiked model (\Cref{ass:rank-1}) yields
  $\inner{\cO(G^\mathrm{sig})}{G^\mathrm{sig}}_F = \lambda\|uv^\top\|_F^2 = \lambda$,
  which means that the signal alignment under Polar-only pipeline is $\E[\inner{\cO(G_t)}{uv^\top}_F]$.
  We derive the upper bound on this quantity, whose integrand is bounded as follows via the Cauchy--Schwarz inequality:
  \[
    \begin{aligned}
      \inner{\cO(G_t)}{uv^\top}_F
      &= \inner{\cO(\Xi_t)}{uv^\top}_F + \inner{\cO(G_t)-\cO(\Xi_t)}{uv^\top}_F \\
      &\le \underbrace{\inner{\cO(\Xi_t)}{uv^\top}_F}_{\eqqcolon (\clubsuit)} + \underbrace{\|\cO(G_t) - \cO(\Xi_t)\|_F}_{\eqqcolon (\diamondsuit)}.
    \end{aligned}
  \]

  To address the term $(\clubsuit)$, note that the polar factor of an element-wise Gaussian matrix $\cO(\Xi_t)$ is uniformly distributed over the Stiefel manifold $\mathrm{St}(m,n)\coloneqq\{U\in\Rbb^{m\times n} : U^\top U=I_n\}$~\citep[Theorem~2.2.1]{chikuse2003statistics}.
  This implies that $\cO(\Xi_t)$ and $V_1\cO(\Xi_t)V_2$ for the pair of orthogonal actions $(V_1,V_2)\in O(m)\times O(n)$ follow the same law.
  Thus, we can take $u=e_m^{(1)}\in\Rbb^m$ and $v=e_n^{(1)}\in\Rbb^n$ (the standard basis vector with the first coordinate being one) without loss of generality because the underlying law of $\inner{\cO(\Xi_t)}{uv^\top}_F$ and $\inner{\cO(\Xi_t)}{e_m^{(1)}e_n^{(1)\top}}_F$ remains the same with any pair of orthogonal actions $(V_1,V_2)$ satisfying $u=V_1^\top e_m^{(1)}$ and $v=V_2e_n^{(1)}$.
  Now, we can focus on $\inner{\cO(\Xi_t)}{e_m^{(1)}e_n^{(1)\top}}_F = [\cO(\Xi_t)]_{11}$ instead.
  Since $\cO(\Xi_t)\in\mathrm{St}(m,n)$, we immediately have $\sum_{i=1}^m[\cO(\Xi_t)]_{i1}^2=1$.
  By noting that the law of $\cO(\Xi_t)$ is invariant with the orthogonally actioned $V\cO(\Xi_t)$ with $V\in O(m)$, all of $[\cO(\Xi_t)]_{11}, \dots, [\cO(\Xi_t)]_{m1}$ follow the same law.
  This indicates that we have $\E[\sum_{i=1}^m[\cO(\Xi_t)]_{i1}^2] = m\E[[\cO(\Xi_t)]_{11}^2] = 1$.
  Therefore,
  \[
    \E[\inner{\cO(\Xi_t)}{uv^\top}_F]
    = \E[[\cO(\Xi_t)]_{11}]
    \le \sqrt{\E[[\cO(\Xi_t)]_{11}^2]}
    = \frac1{\sqrt m},
  \]
  where we used Jensen's inequality.

  To address the term $(\diamondsuit)$, let $\mathcal E\coloneqq\{\sigma_{\min}(\Xi_t) \ge s_*/2\}$ for $s_* \coloneqq \sigma_\Xi(\sqrt m - \sqrt n)$ denotes the ``good'' event,
  where $\sigma_{\min}(\Xi_t) > \lambda$ holds.
  We bound $(\diamondsuit)$ by viewing $G_t$ as a perturbation of $\Xi_t$ toward the rank-1 spike $\lambda uv^\top$.
  Write $G_{t,\tau}\coloneqq \Xi_t + \tau\lambda uv^\top$ ($0<\tau<1$) as an intermediate perturbation.
  Before proceeding, let us ensure the differentiability of the polar factor $\cO(\cdot)$ at $G_{t,\tau}$, for which $\sigma_{\min}(G_{t,\tau}) > 0$ is necessary by the definition of the polar factor --- in fact, to ensure the invertibility of $G_{t,\tau}^\top G_{t,\tau}$.
  By Weyl's inequality for singular values, we have
  \begin{equation}
    \sigma_{\min}(G_{t,\tau}) \ge \sigma_{\min}(\Xi_t) - \tau\lambda\snorm{uv^\top}
    = \sigma_{\min}(\Xi_t) - \tau\lambda
    \ge \sigma_{\min}(\Xi_t) - \lambda.
    \label{eq:weyl}
  \end{equation}
  Thus, $\sigma_{\min}(G_{t,\tau}) > 0$ is ensured on the ``good'' event $\mathcal E$.
  In this case, the fundamental theorem of calculus gives
  \[
    \cO(G_t) - \cO(\Xi_t)
    = \int_0^1 \frac{\mathrm{d}}{\mathrm{d}\tau}\cO(G_{t,\tau})\mathrm{d}\tau
    = \int_0^1 D\cO(G_{t,\tau})[\lambda uv^\top]\mathrm{d}\tau,
  \]
  where $D\cO(\cdot)[H]$ is the G{\^a}teaux derivative of the polar factor in the direction of $H\in\Rbb^{m\times n}$.
  Note that the polar factor is locally Lipschitz~\citep{higham2008functions}, that is, $\|D\cO(X)[H]\|_F \le \|H\|_F / \sigma_{\min}(X)$.
  Using this, we have
  \[
    \|\cO(G_t) - \cO(\Xi_t)\|_F
    \le \int_0^1\left\|D\cO(G_{t,\tau})[\lambda uv^\top]\right\|_F\mathrm{d}\tau
    \le \frac{\lambda}{\sigma_{\min}(G_{t,\tau})}
    \le \frac{\lambda}{\sigma_{\min}(\Xi_t) - \lambda},
  \]
  where Weyl's inequality~\cref{eq:weyl} is used at the last inequality.
  By combining the conditioned event $\mathcal E$ and the low SNR assumption $\lambda<s_*/4$, we finally have $\|\cO(G_t) - \cO(\Xi_t)\|_F \le 4\lambda/s_*$.
  In contrast, conditioning on the ``bad'' event $\mathcal E^\complement$, we use a more elementary bound $\|\cO(G_t) - \cO(\Xi_t)\|_F \le \|\cO(G_t)\|_F + \|\cO(\Xi_t)\|_F = 2\sqrt n$.
  To properly control, let us evaluate $\Pr(\mathcal E^\complement)$.
  By \citet[Exercise 7.3.4]{vershynin2018high}, we have the two-sided bound on the (least) singular value of a Gaussian random matrix: for any $u\ge 0$, there exists an absolute constant $c>0$ such that
  \[
    \sigma_{\min}(\Xi_t) \ge \sigma_{\Xi}(\sqrt m - \sqrt n - u)
    \quad \text{with probability at least $1-2\exp(-cu^2)$.}
  \]
  By choosing $u=(\sqrt m - \sqrt n)/2$, we have the following probability bound:
  \[
    \Pr(\mathcal E^\complement)
    = \Pr\left(\sigma_{\min}(\Xi_t) < \frac{\sigma_\Xi(\sqrt m - \sqrt n)}{2}\right) \le 2\exp\left(-\frac{c(\sqrt m - \sqrt n)^2}{4}\right).
  \]
  After all, we can evaluate the expectation of $(\diamondsuit)$ by decomposing the event into $\mathcal E$ and $\mathcal E^\complement$ as follows:
  \[
    \begin{aligned}
      \E\|\cO(G_t) - \cO(\Xi_t)\|_F
      &\le \frac{4\lambda}{\sigma_\Xi(\sqrt m - \sqrt n)} \cdot \Pr(\mathcal E) + 2\sqrt n \cdot \Pr(\mathcal E^\complement) \\
      &\le \frac{4\lambda}{\sigma_\Xi(\sqrt m - \sqrt n)} + 4\sqrt n\exp\left(-\frac{c(\sqrt m - \sqrt n)^2}{4}\right).
    \end{aligned}
  \]

  Combining the above all of $(\clubsuit)$ and $(\diamondsuit)$, we have the desired inequality.
\end{proof}

\section{Additional Discussions}
\label{app:supporting}

\subsection{Sufficient Conditions for Signal Persistence}
\label{app:persistence}

\Cref{ass:persistence} bounds the momentum-filtered signed coordinate $\bar\lambda_k(t)$ from below by $c_{\mathrm{sig}}\lambda_k$ in absolute value.
In this subsection, we demonstrate that \Cref{ass:persistence} is indeed satisfied by a specific signed coordinate $\lambda_k(t)$.

Suppose $\lambda_k(t)=\bar\mu_k+\xi_k(t)$ with $\bar\mu_k\in\Rbb$ and $\{\xi_k(t)\}_t$ independent zero-mean sub-Gaussian with parameter $\sigma_\lambda^2$. The momentum-filtered coordinate $\bar\lambda_k(t)=(1-\beta)\sum_{\tau=0}^{t}\beta^\tau\lambda_k(t-\tau)$ has mean $(1-\beta^{t+1})\bar\mu_k$, which equals $\bar\mu_k$ up to an $O(\beta^t)$ term once $t\ge c_0 T$, and sub-Gaussian parameter at most $\sigma_\lambda^2(1-\beta)^2\sum_{\tau\ge0}\beta^{2\tau}=\sigma_\lambda^2/(2T-1)$ by the weight identity~\cref{eq:effective-sample-size}. Choosing $c_0$ so that $\beta^{c_0 T}\le 1/4$, the mean has magnitude at least $\tfrac{3}{4}|\bar\mu_k|$, so for every fixed $t\ge c_0 T$ a single sub-Gaussian tail bound on the fluctuation gives
\[
\Prob\!\left(|\bar\lambda_k(t)|<|\bar\mu_k|/2\right)
\le2\exp\!\left(-\frac{\bar\mu_k^2(2T-1)}{32\sigma_\lambda^2}\right).
\]
Taking a union bound over $k=1,\ldots,r$, we have
\[
\Prob\!\left(\exists k\le r:\,|\bar\lambda_k(t)|<|\bar\mu_k|/2\right)
\le 2r\exp\!\left(-\frac{\min_{k\le r}\bar\mu_k^2\,(2T-1)}{32\sigma_\lambda^2}\right).
\]
Therefore, under the scenario $\min_k|\bar\mu_k|\gg\sigma_\lambda$, \Cref{ass:persistence} holds simultaneously for all $k\le r$ with
\[
c_{\mathrm{sig}}=\min_{k\le r}\frac{|\bar\mu_k|}{2\lambda_k}
\]
on an event of overwhelming probability at every fixed $t\ge c_0 T$.


\subsection{Sub-Gaussian Strengthening of \Cref{prop:ema-concentration}}
\label{app:prop1}

The Chebyshev-based proof of \Cref{prop:ema-concentration} as stated (under bounded Frobenius-norm second moment) appears in \Cref{app:proof-thm1-cor-1}. The result below is a \emph{separate strengthening} of \Cref{prop:ema-concentration} under a stronger hypothesis: it assumes that each bilinear projection $x^\top\Xi_t y$ is centered sub-Gaussian, and concludes with an exponential tail rather than the Chebyshev $1/u^2$ tail of \Cref{prop:ema-concentration}. We show this strengthening here for completeness. It is not used in the main theorem chain, which goes through the Chebyshev term only. Note that under the variance bound $\E\|\Xi_t\|_F^2 \le \eta$ of \Cref{ass:decomp}(b), the entrywise sub-Gaussian parameter $v$ below corresponds to $\eta = mn\,v^2$.

\begin{proposition}[Momentum-weighted bilinear sub-Gaussian concentration]
\label{prop:ema-subgaussian}
Let $\{\Xi_t\}_{t\ge 0}$ be a sequence of independent $m\times n$ random matrices. Suppose that for every fixed unit vectors $x\in\Rbb^m$ and $y\in\Rbb^n$, the bilinear projection $x^\top\Xi_t y$ is sub-Gaussian with parameter $v>0$. Define
\[
S_t  \coloneqq  (1-\beta)\sum_{s\ge0}\beta^s \Xi_{t-s}
\qquad\text{and}\qquad T \coloneqq \frac{1}{1-\beta}.
\]
Then for every fixed unit vectors $x,y$ and every $\varepsilon>0$,
\[
\Prob\!\left(|x^\top S_t y| > \varepsilon\right)
\le 2\exp\!\left(-\frac{(2T-1)\varepsilon^2}{2v^2}\right).
\]
An $\varepsilon$-net and union bound (e.g. \cite[Theorem 4.4.5]{vershynin2018high}) upgrade this to operator-norm concentration with an additional $m+n$ factor in the exponent.
\end{proposition}

\begin{proof}
Let $h_s  \coloneqq  (1-\beta)\beta^s$, so that $S_t = \sum_{s\ge 0} h_s \Xi_{t-s}$. For fixed unit vectors $x\in\Rbb^m$ and $y\in\Rbb^n$, set
\[
z_s  \coloneqq  x^\top \Xi_{t-s} y
\qquad\text{and}\qquad
W_L  \coloneqq  \sum_{s=0}^{L} h_s z_s
.
\]
Since the $z_s$ are independent and each is sub-Gaussian with parameter $v$, the moment generating function (MGF) of $W_L$ factorizes over $s$. For all $\tau\in\Rbb$,
\[
\E[e^{\tau W_L}]
= \prod_{s=0}^{L} \E[e^{\tau h_s z_s}]
\le \prod_{s=0}^{L} \exp\!\left(\frac{\tau^2 h_s^2 v^2}{2}\right)
= \exp\!\left(\frac{\tau^2 v^2}{2}\sum_{s=0}^{L} h_s^2\right)
\overset{\text{\cref{eq:effective-sample-size}}}{\le} \exp\!\left(\frac{\tau^2 v^2}{2(2T-1)}\right).
\]
A standard one-sided Chernoff bound followed by symmetrization (apply the same MGF bound to $-W_L$) gives the two-sided tail
\[
\Prob\!\left(|x^\top S_t y|>\varepsilon\right)
\le 2\exp\!\left(-\frac{(2T-1)\varepsilon^2}{2v^2}\right).
\]
An $\varepsilon$-net and union bound then upgrade this to operator-norm concentration.

\end{proof}


\section{Experimental Setups}
\label{app:exp-setup}

\subsection{Shared Computational Settings}
\label{app:shared-env}

The experiments in this paper run on two hardware environments. The end-to-end NanoGPT and LLaMA 350M~\citep{touvron2023llama} training comparison reported in \cref{fig:end2end} (Pre-polar, Post-polar, and Polar-only) runs on NVIDIA L40S GPUs. All remaining experiments, the synthetic simulations, the CIFAR-10 stationary and trajectory probes, and the NanoGPT stationary and trajectory probes, run on NVIDIA A100 40GB GPUs on the Wisteria Aquarius \footnote{\url{https://www.cc.u-tokyo.ac.jp/supercomputer/wisteria/service/}} system at the University of Tokyo.

The L40S experiments split into two end-to-end runs: a NanoGPT three-pipeline comparison on a single-node, 4-GPU PyTorch DDP setup with \texttt{bfloat16} autocast (\Cref{app:nanogpt-end2end}), and a LLaMA 350M three-pipeline comparison on a single-node, 8-GPU PyTorch DDP setup with \texttt{bfloat16} training (\Cref{app:llama350m-end2end}). For the NanoGPT experiment, we use a GPT-2-style 12-layer / width-768 / 6-head modded-NanoGPT\footnote{\url{https://github.com/KellerJordan/modded-nanogpt}} speedrun model trained on FineWeb-10B~\citep{penedo2024fineweb}. For the LLaMA experiment, we used a reduced-scale language model that preserves the corresponding design choices: a LLaMA-style 24-layer / width-1024 / 16-head causal model with a 32k LLaMA-family tokenizer trained on C4~\citep{raffel2020exploring}. The shared L40S settings for the two studies are recorded in \Cref{app:nanogpt-end2end,app:llama350m-end2end}, and the best hyperparameter selections for each Muon pipeline are reported in the same appendices.

\begin{table}[htbp]
  \centering\small
  \caption{Experimental settings for the NanoGPT three-pipeline end-to-end training comparison on NVIDIA L40S GPUs.}
  \label{tab:nanogpt-l40s}
  \begin{tabular}{@{}ll@{}}
    \toprule
    Setting & Value \\
    \midrule
    Compute setup & 4 NVIDIA L40S GPUs, single-node PyTorch DDP \\
    Training precision & \texttt{bfloat16} autocast \\
    Model family & GPT-2-style modded-NanoGPT speedrun causal language model \\
    Model scale & 12 layers, width 768, 6 heads, MLP width 3072 \\
    Core design choices & RMSNorm, RoPE, GPT-2 BPE (vocab 50304) \\
    Data pipeline & FineWeb-10B pre-tokenized binary shards \\
    Sequence length & 1024 \\
    Batch shape & 64 sequences/GPU, gradient accumulation 2, global batch 512 \\
    Training steps & 5100 \\
    Learning-rate schedule & plateau--warmdown, 1450 warmdown steps \\
    Regularization / clipping & weight decay 0, gradient clipping disabled \\
    \bottomrule
  \end{tabular}
\end{table}

\begin{table}[htbp]
  \centering\small
  \caption{Experimental settings for the LLaMA 350M three-pipeline end-to-end training comparison on NVIDIA L40S GPUs.}
  \label{tab:llama350m-l40s}
  \begin{tabular}{@{}ll@{}}
    \toprule
    Setting & Value \\
    \midrule
    Compute setup & 8 NVIDIA L40S GPUs, single-node PyTorch DDP \\
    Training precision & \texttt{bfloat16} \\
    Model family & Reduced-scale LLaMA-style causal language model \\
    Model scale & 24 layers, width 1024, 16 heads, MLP width 2736 \\
    Core design choices & RMSNorm, 32k LLaMA-family tokenizer \\
    Data pipeline & C4 dataset with document-level truncation and right-padding \\
    Sequence length & 1024 \\
    Batch shape & 128 sequences/GPU, gradient accumulation 1, global batch 1024 \\
    Training steps & 3000 \\
    Learning-rate schedule & cosine decay, 1000 warmup steps, minimum LR ratio 0.1 \\
    Regularization / clipping & weight decay 0, gradient clipping disabled \\
    \bottomrule
  \end{tabular}
\end{table}

The Wisteria runs share a common compute environment. The synthetic simulation uses a single A100 GPU for its spiked-model simulations. The CIFAR-10 stationary and trajectory probes use a single A100. The NanoGPT probe jobs use four A100s under PyTorch DDP with \texttt{bfloat16} mixed precision. Spectrum summaries, polar factors, signal and subspace alignments are computed with exact \texttt{float32} SVDs inside the \textbf{Analysis procedure}, while training runs in \texttt{bfloat16}. \Cref{tab:wisteria} records the shared Wisteria settings.

\begin{table}[htbp]
  \centering\small
  \caption{Shared experimental settings for the Wisteria Aquarius runs (the synthetic simulations, the CIFAR-10 probes, and the NanoGPT probes).}
  \label{tab:wisteria}
  \begin{tabular}{@{}ll@{}}
    \toprule
    Setting & Value \\
    \midrule
    Cluster & Wisteria Aquarius \\
    Compute GPUs & NVIDIA A100 40GB \\
    Compiler / CUDA / Python & \texttt{gcc/8.3.1}, \texttt{cuda/11.8}, \texttt{python/3.11.7} \\
    Training precision & \texttt{bfloat16} for NanoGPT, \texttt{float32} for CIFAR-10 and synthetic \\
    Probe analysis precision & exact \texttt{float32} SVD for all spectrum and polar computations \\
    Random seeds & 1337, 1338, 1339 \\
    \bottomrule
  \end{tabular}
\end{table}

\subsection{Probe Protocols}
\label{app:probe-protocols}

The CIFAR-10 and NanoGPT experiments use the same two probe protocols on a designated target weight matrix $W\in\Rbb^{m\times n}$. The two protocols differ only in how the per-step mini-batch gradients are collected. The downstream \textbf{Analysis procedure} is identical.

Let $G_t\in\Rbb^{m\times n}$ denote the mini-batch gradient of $W$ at step $t$. The probe records a gradient buffer $\{G_t\}_{t=1}^{K}$ of $K$ size. Gradients are collected immediately after backpropagation and before any optimizer state update is applied. We used two distinct momentum coefficients throughout: $\beta_{\mathrm{train}}$ is the training-time Muon momentum coefficient (used by the optimizer that generates $\{G_t\}$), and $\beta$ is the probe-side momentum coefficient that the probe analysis below applies to the gradient buffer. The two are decoupled by construction: the \textbf{Analysis procedure} operates solely on the collected raw gradient buffer $\{G_t\}$ and never accesses the training optimizer's internal state; at every probe step the three pipelines are recomputed from the same gradient buffer $\{G_t\}$. Pre-polar, Post-polar, and Polar-only comparisons are therefore fair comparisons on a common gradient stream, and training-time optimizer settings affect the results only through the trajectory that generates $\{G_t\}$. Next, we introduce the stationary probe and the trajectory probe.

\paragraph{Stationary probe.}
\begin{enumerate}[leftmargin=1.5em,itemsep=2pt,topsep=2pt]
\item Load the model from a saved checkpoint and hold every weight fixed, including the target weight $W$. No weight is updated during gradient collection.
\item For $t = 1, \dots, K$, draw one mini-batch from the dataloader in its natural sequential order (the dataloader's default ordering without shuffling), run one step of forward and backward propagation, and record the gradient $G_t$ of the target weight $W$. We refer to this as the \emph{sequential} collection order, which is the default setting used throughout the paper. For the NanoGPT stationary probe at the step-3000 checkpoint, we additionally reran the same protocol under a \emph{shuffled} collection order as a robustness check, where each mini-batch starts at a random position in the current data shard. The shuffled results are reported only in \Cref{app:thm1-supplement}.
\item Save the gradient sequence $\{G_t\}_{t=1}^{K}$ to disk in the order the mini-batches were drawn. The collection order is preserved because the downstream momentum buffers are order-dependent.
\end{enumerate}
Since the model weights do not change during gradient collection, every $G_t$ is drawn from the same gradient distribution. This protocol therefore synthetically simulates the stationary special case of the BVMZOS perturbation model of \Cref{ass:decomp}. The three pipelines are applied to the gradient buffer in its collection order, so Pre-polar and Post-polar momentum buffers see the gradient stream $\{G_t\}_{t=1}^{K}$ exactly as it was recorded. For the stationary probe at the step-3000 checkpoint we report both the sequential and shuffled collection orders. All other stationary checkpoints use the sequential order only.

\paragraph{Trajectory probe.}
During end-to-end training:
\begin{enumerate}[leftmargin=1.5em,itemsep=2pt,topsep=2pt]
\item Maintain a fixed-capacity First-In, First-Out (FIFO) buffer of the most recent $K$ target weight's gradients alongside the regular optimizer step. Each training step appends $G_t$ to the buffer and pops the oldest entry once the buffer is full.
\item At every $I$ training step (we used $I=100$ in all CIFAR-10 and NanoGPT trajectory runs), take a checkpoint of the current buffer $\{G_{t-K+1}, \dots, G_t\}$, run the \textbf{Analysis procedure} below, and save the resulting summary.
\item Continue training without modification. The probe does not alter weight updates, optimizer state, random seeds, or data ordering.
\end{enumerate}
The trajectory buffer therefore represents a sliding buffer over the live training trajectory, and the corresponding analysis quantities inherit any non-stationarity in the gradient stream.

\paragraph{Buffer-size selection.}
The buffer size $K$ obeys different constraints in the two protocols. With zero initialization, the momentum buffer (\cref{eq:ema}) truncated to $K$ steps carries kernel mass $1-\beta^K$ (see \Cref{app:initialization}), so $K \ge c_0 T$ keeps the momentum-filtered buffers (e.g., $M_K^{(\beta)}$ in \cref{eq:Mk-prepolar}) close to their asymptotic values and allows the in-buffer mean gradient $\bar G$ to serve as the reference of the coherent gradient signal.

In the stationary probe, the weights are held fixed during gradient collection, so the gradient distribution of $G_t$ does not drift across the $K$ steps. The only constraint on $K$ is the lower bound $K \ge c_0 T$ at the largest $\beta$ in the per-task grid (\Cref{app:exp-tasks}). A larger $K$ is therefore preferable.

In the trajectory probe, the weights are updated by the optimizer across the $K$ steps, so the gradient distribution of $G_t$ drifts with $t$. The lower bound $K \ge c_0 T$ still applies, and $K$ must also remain small enough to limit this drift across the steps contributing to the in-buffer mean gradient $\bar G$ (defined in the \textbf{Analysis procedure} below). The choice of $K$ therefore balances these two requirements.

\paragraph{Analysis procedure.}
Given a buffer $\{G_t\}_{t=1}^{K}$, the analysis performs the following steps:
\begin{enumerate}[leftmargin=1.5em,itemsep=2pt,topsep=2pt]
\item \emph{Signal reference.} Compute the mean gradient $\bar G \coloneqq K^{-1}\sum_{t=1}^{K} G_t$ and its exact SVD $\bar G = U\Sigma V^{\top}$. The top-$r$ left and right singular vectors $(U_r, V_r)$ define the signal subspace used as the alignment target on the CIFAR-10 and NanoGPT experiments. On the synthetic simulation the spiked model of \Cref{app:synthetic-settings} plants a known rank-$r_\star$ signal $G_t^{\mathrm{sig}}$ with singular bases $U_{\mathrm{true}}, V_{\mathrm{true}}$, so the alignment target is the planted top-$r$ subspace $(U_{\mathrm{true}}, V_{\mathrm{true}})$ directly rather than $(U_r, V_r)$ from $\bar G$ (\Cref{app:measurements}).
\item \emph{Three pipelines.} For each $\beta$ in the per-task grid (\Cref{app:exp-tasks}), starting from $M_0 = \widetilde{M}_0 = 0$, Pre-polar and Post-polar pipelines maintain two separate momentum buffers:
\begin{align}
\text{Pre-polar:}\quad & \cO\!\bigl(M_K^{(\beta)}\bigr), \quad\text{where}\quad M_K^{(\beta)} \;\coloneqq\; (1 - \beta)\sum_{s=0}^{K-1}\beta^{s}\,G_{K-s}, \label{eq:Mk-prepolar}\\
\text{Post-polar:}\quad & \widetilde M_K^{(\beta)} \;\coloneqq\; (1 - \beta)\sum_{s=0}^{K-1}\beta^{s}\,\cO(G_{K-s}), \\
\text{Polar-only:}\quad & \cO(G_K),
\end{align}
where $\cO(\cdot)$ is the polar factor introduced in \cref{eq:forward}. Pre-polar buffer $M_K^{(\beta)}$ averages the raw gradients and is orthogonalized once at the end, whereas Post-polar buffer $\widetilde M_K^{(\beta)}$ orthogonalizes each per-step gradient first and then averages the resulting polar factors. Polar-only baseline skips the momentum entirely and orthogonalizes only the final raw gradient $G_K$.
\item \emph{Spectral summaries.} Record (i) the singular-value sequences of $\bar G$ and of Pre-polar momentum buffer $M_K^{(\beta)}$, (ii) the per-step filtering ratio $\sigma_k(M_K^{(\beta)})/\sigma_k(G_K)$ at the final collection index, and (iii) the noise-suppression ratio $R(T)$ that compares the operator norm of the raw-gradient residual with that of the momentum residual at momentum window size $T = 1/(1-\beta)$ (introduced in \Cref{app:measurements}). The two ratios serve different purposes and use different denominators. The per-step filtering ratio measures the per-singular-value attenuation between $M_K^{(\beta)}$ and $G_K$. $R(T)$ uses a signal reference in both numerator and denominator and measures the operator-norm attenuation of the momentum-filtered perturbation. The signal reference is $\bar G$ on the CIFAR-10 and NanoGPT experiments and the planted signal $G_t^{\mathrm{sig}}$ on the synthetic simulation. The synthetic case additionally applies the bias correction $(1-\beta^K)\,G_t^{\mathrm{sig}}$ in the denominator to account for the zero-initialized momentum buffer (\Cref{app:measurements}).
\item \emph{Signal alignment and subspace alignment error metrics.} Signal alignment is reported with the rank-$r$ and full-rank signal alignment metrics $\mathrm{Align}_r$ and $\mathrm{Align}_{\mathrm{full}}$ of \Cref{app:measurements}. Subspace alignment error panels report the $\sin\Theta$ principal-angle distance at fixed ranks $r\in$ \{1,5,10\}.
\end{enumerate}
All SVDs inside the \textbf{Analysis procedure} are exact \texttt{float32} decompositions. Newton--Schulz iteration is used only inside the training-time Muon optimizer and never inside the probe.

\subsection{Measurements}
\label{app:measurements}

Every probe run produces the same four measurements from the gradient buffer $\{G_t\}_{t=1}^{K}$ and the three pipelines of \Cref{app:probe-protocols}. Each measurement is tied to a specific theoretical claim:
\begin{enumerate}[leftmargin=1.5em,itemsep=2pt,topsep=2pt]
\item \emph{Per-step filtering ratio} --- the ratio of the $k$-th singular value of Pre-polar momentum buffer $M_K^{(\beta)}$ (\cref{eq:Mk-prepolar}) to that of the latest collected raw gradient $G_K$,
\begin{equation*}
  \mathrm{Filt}_k(\beta) \;\coloneqq\; \frac{\sigma_k\!\bigl(M_K^{(\beta)}\bigr)}{\sigma_k(G_K)}.
\end{equation*}
Both spectra are computed from the same gradient buffer at the final collection step: $G_K$ is the last raw mini-batch gradient (equivalently, the momentum buffer at $\beta=0$), and $M_K^{(\beta)}$ is the momentum over the same window. By construction $\mathrm{Filt}_k(0)=1$. \Cref{thm:spectral-gap} predicts that when $\beta$ is close enough to $1$, the tail filtering ratios ($k > r$) are suppressed more than the head ($k \le r$), opening a spectral gap between head and tail that widens with $\beta$. The per-step filtering ratio supports \Cref{thm:spectral-gap} qualitatively.
\item \emph{Noise-suppression ratio} --- residual operator-norm ratio $R(T)$ of raw gradient vs.\ Pre-polar momentum buffer with probe-side momentum coefficient $\beta$ (associated with the effective sample size $2T-1$). Explicitly,
\begin{equation*}
  R(T) \;\coloneqq\; \frac{\bigl\|G_K - \bar G\bigr\|_{\op}}{\bigl\|M_K^{(\beta)} - \bar G\bigr\|_{\op}},
  \qquad
  \bar G \;\coloneqq\; \frac{1}{K}\sum_{t=1}^{K} G_t,
\end{equation*}
where $\bar G$ is the in-buffer approximation of $G_t^{\mathrm{sig}}$. Subtracting $\bar G$ from both numerator and denominator approximately removes the coherent signal component $G_t^{\mathrm{sig}}$ from $G_t = G_t^{\mathrm{sig}} + \Xi_t$, so the numerator $\|G_K - \bar G\|_{\op} \approx \|\Xi_K\|_{\op}$ approximates the single-step perturbation operator norm, while the denominator $\|M_K^{(\beta)} - \bar G\|_{\op} \approx \|S_K\|_{\op}$ approximates the operator norm of the momentum-filtered perturbation $S_K$ from \Cref{prop:ema-concentration}. Their ratio $R(T)$ therefore measures how much the momentum reduces the perturbation operator norm relative to a single step (see \Cref{rem:RT-floor} for the derivation). The figure conventions plot $R(T)$ against $T=1/(1-\beta)$, with the theoretical prediction $(2T-1)^{1/4}$ (derived in \Cref{rem:RT-floor}) shown as the dashed guide. The choice of $\bar G$ as the signal reference is what distinguishes $R(T)$ from the per-step filtering ratio: $\mathrm{Filt}_k(\beta)$ uses the latest single-batch gradient $G_K$ as the denominator and reports per-index attenuation, while $R(T)$ uses the $\bar G$ as the reference and reports the perturbation operator-norm reduction relative to a single step, compared against the $(2T-1)^{1/4}$ guide.\footnote{On the synthetic simulation the planted signal $G_t^{\mathrm{sig}}$ is known exactly, so we replace $\bar G$ with $G_t^{\mathrm{sig}}$ and apply a zero-init bias correction to the momentum buffer: the synthetic $R(T)$ is $\|G_K - G_t^{\mathrm{sig}}\|_{\op} / \|M_K^{(\beta)} - (1-\beta^K)\,G_t^{\mathrm{sig}}\|_{\op}$. The bias factor $(1 - \beta^K)$ accounts for the momentum buffer starting from zero, which makes $\mathbb{E}[M_K^{(\beta)}] = (1-\beta^K)\,G_t^{\mathrm{sig}}$ under $G_t = G_t^{\mathrm{sig}} + \Xi_t$. Without this correction the synthetic denominator would be inflated by an $O(\beta^K)$ transient at large $\beta$.} The noise-suppression ratio is therefore an empirical check that $R(T)$ grows with the momentum window size $T$ at a rate no slower than $(2T-1)^{1/4}$ with high probability, as predicted by \Cref{prop:ema-concentration} and \Cref{thm:spectral-gap}.


\item \emph{Subspace alignment error} --- the $\sin\theta$ subspace distance from the top-$r$ singular subspaces of the reference to those of Pre-polar momentum buffer. Explicitly,
\[
  \sin\theta_r(A;\, B) \;\coloneqq\; \bigl\| \sin\Theta\!\bigl(U_r(A),\, U_r(B)\bigr) \bigr\|_2,
\]
with the right-subspace version $\sin\theta_r(A^\top;\, B^\top) = \|\sin\Theta(V_r(A),\, V_r(B))\|_2$ reported separately. In this paper, we also define $\sin\Theta_U \coloneqq \sin\theta_r(M_K^{(\beta)};\,\bar G)$ and $\sin\Theta_V \coloneqq \sin\theta_r(M_K^{(\beta)\top};\,\bar G^{\top})$ for convenience. The gradient signal reference is $\bar G$ on the CIFAR-10 and NanoGPT experiments at rank $r\in$ \{1,5,10\}.\footnote{On the synthetic simulation the reference is the planted top-$r$ singular subspace $(U_{\mathrm{true}},V_{\mathrm{true}})$ at $r\in\{1,2,3\}$. Restricting to $r \le r_\star$ is forced by the model: the planted bases have only $r_\star$ orthonormal columns, so a query at $r > r_\star$ is either ill-defined or pads with arbitrary directions from the noise null space.} The subspace alignment error supports \Cref{cor:direction}, which predicts that $\sin\Theta_U$ and $\sin\Theta_V$ decrease as the momentum window size $T$ grows at a rate no slower than $(2T-1)^{1/4}$ with high probability.
\item \emph{Signal alignment} --- the signal alignment comparison applied to Pre-polar $=\cO(M_K^{(\beta)})$, Post-polar $=\widetilde M_K^{(\beta)}$, and Polar-only $=\cO(G_K)$, reported through the following two metrics:

\emph{Rank-$r$ signal alignment.}
\[
  \mathrm{Align}_r(A;\, B)
  \;\coloneqq\;
  \frac{\bigl\|U_r(B)^\top A\, V_r(B)\bigr\|_F}{\sqrt{r}}
  \;\in\;[0,1].
\]
Larger values indicate stronger signal alignment. \Cref{thm:recovery} predicts that Pre-polar achieves higher $\mathrm{Align}_r$ than Post-polar and Polar-only for sufficiently large momentum window size $T$. On the CIFAR-10 and NanoGPT experiments the reference is $B = \bar G$ and the rank ladder is $r\in$ \{1,5,10\}.\footnote{On the synthetic simulation $B = U_{\mathrm{true}} V_{\mathrm{true}}^\top$ and the rank ladder is $r\in\{1,2,3\}$.}

\emph{Full-rank signal alignment.}
\[
  \mathrm{Align}_{\mathrm{full}}(A;\, B)
  \;\coloneqq\;
  \frac{\langle A,\cO(B)\rangle_F}
       {\min(m,n)}
  \;\in\;[-1,1].
\]
We report $\mathrm{Align}_{\mathrm{full}}$ as a full-rank cross-check on the CIFAR-10 and NanoGPT experiments only.\footnote{The synthetic simulation drops it because the relevant signal reference is the planted rank-$r_\star$ partial polar factor $U_{\mathrm{true}} V_{\mathrm{true}}^\top$, which has Frobenius norm $\sqrt{r_\star}$ and caps the corresponding normalized full-rank signal alignment at $r_\star/\min(m,n) = 0.03$ on our 100 $\times$ 100 synthetic simulation.} Larger values indicate stronger full-rank alignment. \Cref{thm:recovery} predicts that Pre-polar achieves higher $\mathrm{Align}_{\mathrm{full}}$ than Post-polar and Polar-only for sufficiently large momentum window size $T$.
\end{enumerate}
The probe-side momentum coefficient grids are measurement- and task-specific and are listed with the corresponding task definitions in \Cref{app:exp-tasks}. The default probe-side momentum coefficient $\beta=0.95$ is shared across all experiments.

\begin{remark}[Predicted $(2T-1)^{1/4}$ guide on $R(T)$]
\label{rem:RT-floor}
The dashed $(2T-1)^{1/4}$ guide in $R(T)$ panels comes from inverting the operator-norm corollary of \Cref{prop:ema-concentration}.
\\
\emph{Bridge.} Substituting $G_t = G_t^{\mathrm{sig}} + \Xi_t$ and writing $\bar G^{\mathrm{sig}} \coloneqq K^{-1}\sum_t G_t^{\mathrm{sig}}$, $\bar\Xi \coloneqq K^{-1}\sum_t \Xi_t$,
\[
G_K - \bar G \;=\; (G_K^{\mathrm{sig}} - \bar G^{\mathrm{sig}}) + (\Xi_K - \bar\Xi),
\quad
M_K^{(\beta)} - \bar G \;=\; (M_K^{\mathrm{sig}} - \bar G^{\mathrm{sig}}) + (S_K - \bar\Xi).
\]
The BVMZOS structure gives $\E\fnorm{\bar\Xi}^2 \le \eta/K$, while \Cref{prop:ema-concentration} gives $\E\fnorm{S_K}^2 \le \eta/(2T-1)$. Since $\snorm{\cdot} \le \fnorm{\cdot}$, the typical operator norms of $\bar\Xi$ and $S_K$ scale as $\sqrt{\eta/K}$ and $\sqrt{\eta/(2T-1)}$, so $\snorm{\bar\Xi}$ is negligible against $\snorm{S_K}$ once $K \gg 2T-1$. The signal residuals $G_K^{\mathrm{sig}} - \bar G^{\mathrm{sig}}$ and $M_K^{\mathrm{sig}} - \bar G^{\mathrm{sig}}$ vanish exactly when $\lambda_k(t)$ is time-invariant, and are dominated by their perturbation counterparts in our experiments otherwise. Hence
\[
\snorm{M_K^{(\beta)} - \bar G} \approx \snorm{S_K},
\qquad
\snorm{G_K - \bar G} \approx \snorm{\Xi_K}.
\]
\\
\emph{Lower bound.} Setting $u = (2T-1)^{1/4}$ in \Cref{prop:ema-concentration} gives
\[
  \snorm{S_K} \;\le\; \sqrt{\eta}\,(2T-1)^{-1/4}
  \qquad\text{with probability at least}\quad 1 - (2T-1)^{-1/2}.
\]
The numerator $\snorm{\Xi_K}$ is $T$-independent under \Cref{ass:decomp}(b). Dividing,
\[
  R(T) \;\gtrsim\; (2T-1)^{1/4}
\]
on the same event, up to a noise-distribution-dependent constant.

\emph{Sharper rate under sub-Gaussian projections.} Under the bilinear sub-Gaussian hypothesis of \Cref{prop:ema-subgaussian} (\Cref{app:prop1}), the same bridge yields the sharper $R(T) \gtrsim \sqrt{2T-1}$, attained by the synthetic Gaussian curves of \cref{fig:app-synth-edge}. The $(2T-1)^{1/4}$ floor uses only the second-moment bound of \Cref{ass:decomp}(b) and therefore remains valid under heavy-tailed gradient noise with finite variance.

\end{remark}

\subsection{Tasks}
\label{app:exp-tasks}

We describe the five experimental tasks below, organized into probe experiments (\Cref{app:synthetic-settings,app:cifar10-probe,app:nanogpt-probe}), where the \textbf{Analysis procedure} is applied to gradients collected from a fixed or live model, and end-to-end training runs (\Cref{app:nanogpt-end2end,app:llama350m-end2end}), where Pre-polar, Post-polar, and Polar-only are compared as full training optimizers.

\subsubsection{Synthetic Simulation}
\label{app:synthetic-settings}

\begin{table}[htbp]
  \centering\small
  \caption{Synthetic task summary.}
  \begin{tabular}{@{}ll@{}}
    \toprule
    Setting & Value \\
    \midrule
    Matrix size & 100 $\times$ 100 \\
    Signal model & low-rank spike + BVMZOS perturbation \\
    \bottomrule
  \end{tabular}
\end{table}

\textbf{Signal Generator.}
The synthetic simulation uses a rank-3 spiked model: a time-invariant low-rank signal plus a perturbation satisfying \Cref{ass:decomp}(b).

\textbf{Perturbation generator.} All synthetic simulations use the perturbation $\Xi_t = b\,\epsilon\,B_t + Z_t$, where $\epsilon\in\{-1,+1\}$ is a single Rademacher sign drawn once per trajectory, $\{B_t\}$ are deterministic dense matrices with $\fnorm{B_t}=1$ and $\inner{B_s}{B_t}_F=0$ for $s\neq t$, $Z_t$ are i.i.d.\ isotropic noise with entrywise variance $\sigma_n^2/2$ (Gaussian, or rescaled finite-variance Student-$t$ with four degrees of freedom for the heavy-tailed panel of \cref{fig:app-synth-filtering}), and $b = \sigma_n\sqrt{mn/2}$. This satisfies \Cref{ass:decomp}(b): mean zero $\E[\Xi_t]=0$, bounded variance $\E\fnorm{\Xi_t}^2=\sigma_n^2 m n$, and Frobenius orthogonality $\E\inner{\Xi_s}{\Xi_t}_F = b^2\inner{B_s}{B_t}_F = 0$ for $s\neq t$.

\textbf{Measurement scope.}
We report the same four measurements as the CIFAR-10 and NanoGPT probes (\Cref{app:measurements}). The synthetic stationary probe uses $K=1000$ collected gradients per trial over 10 random-seed trials, and the signal-strength sweep (\Cref{app:exp-tasks-signal-strength}) uses $K=500$.

\textbf{Probe-side $\beta$ grids.} For each of the following probes, we used different $\beta$ grids. We use a six-point grid \{0.3,0.6,0.9,0.95,0.97,0.99\} to report the filtered singular value spectrum (\Cref{fig:app-synth-filtering}). We use a fourteen-point grid \{0.05,0.1,0.2,0.3,0.4,0.5,0.6,0.7,0.8,0.9,0.95,0.97,0.99,0.995\} for the noise-suppression ratio (\Cref{fig:app-synth-edge}) and the subspace alignment error (\Cref{fig:app-synth-direction}). The Pre-polar / Post-polar / Polar-only $\beta$-sweep (\Cref{fig:frozen-beta-sweep}) uses the same fourteen-point grid. We use the eight-point grid \{0.5,0.7,0.8,0.9,0.93,0.95,0.97,0.99\} for the all-layer 3$\times$4 stationary grid (\Cref{fig:app-nanogpt-ordering-grid}) and the per-layer / per-checkpoint numerical summaries (\Cref{tab:frozen-numerical,tab:frozen-breakdown,tab:frozen-checkpoint}). For simulations that do not sweep $\beta$, we use $\beta=0.95$ as the default probe-side momentum coefficient.

\subsubsection{CIFAR-10 Probe}
\label{app:cifar10-probe}

The target weight matrix to be probed is the convolution layer \texttt{layer2.0.conv1} of a ResNet-18, whose weight tensor is matricized to shape $128 \times 576$. The stationary and trajectory probes follow the general protocol of \Cref{app:probe-protocols}.

\textbf{Stationary probe.}
The CIFAR-10 stationary probe collects mini-batch gradients after warmup on \texttt{layer2.0.conv1} from a ResNet-18 trained on CIFAR-10. In the language of \Cref{ass:decomp}, the signal proxy is the gradient mean over the collection buffer after warmup and the perturbation proxy is the residual around that mean. $K$ is chosen large enough to fully warm the momentum buffer at the largest probe-side $\beta=0.99$ ($T=100$). At $K=2000$, the kernel mass $1-\beta^K$ is within $2\times 10^{-9}$ of one per \Cref{app:initialization}.

\textbf{Trajectory probe.}
The CIFAR-10 trajectory probe maintains a sliding buffer of recent target-layer gradients during continued training after the warmup phase. The analysis interval, stored ordering metric, and direction-cut ranks follow the cross-experiment conventions of \Cref{app:probe-protocols,app:measurements}. We use $K=100$ (The sensitivity analysis with respect to $K$ is provided in \Cref{app:k-sweep}).

\textbf{Probe-side $\beta$ grids.} The four CIFAR-10 measurements use the following grids. We use a six-point grid \{0.3,0.6,0.9,0.95,0.97,0.99\} for the stationary per-step filtering ratio (\Cref{fig:app-cifar-spectrum-ratio}). We use a fourteen-point grid \{0.1,0.2,0.3,0.4,0.5,0.6,0.7,0.8,0.85,0.9,0.93,0.95,0.97,0.99\} for the stationary noise-suppression ratio (\Cref{fig:app-cifar-edge}) and the stationary subspace alignment error (\Cref{fig:app-cifar-direction}). We use an eight-point grid \{0.5,0.7,0.8,0.9,0.93,0.95,0.97,0.99\} for the stationary signal alignment ordering (\Cref{fig:app-cifar-ordering}). On the trajectory probe at $K=100$, we use the fourteen-point grid \{0.1,0.2,0.3,0.4,0.5,0.6,0.7,0.8,0.85,0.9,0.93,0.95,0.97,0.99\} for the noise-suppression ratio and the subspace alignment error (\Cref{fig:app-cifar-online-direction}), and the eight-point grid \{0.5,0.7,0.8,0.9,0.93,0.95,0.97,0.99\} for the ordering panel (\Cref{fig:app-cifar-online-ordering}). For experiments that do not sweep $\beta$, we use $\beta=0.95$ as the default probe-side momentum coefficient.

\subsubsection{NanoGPT Probe}
\label{app:nanogpt-probe}

The NanoGPT training code is adapted from modded-NanoGPT. All NanoGPT probe runs, the stationary probe across three representative attention layers, and the trajectory probe across every attention and MLP output projection, share a single training configuration, described below, except for the L40S end-to-end run configuration shown in \Cref{app:nanogpt-end2end} separately.

\textbf{Training.}
The shared NanoGPT training configuration is a 12-layer, 768-width GPT-2 style model with 6 heads, rotary position embeddings (RoPE)~\citep{su2024roformer}, RMSNorm, scaled the squared-ReLU activations, FineWeb-10B training data, global batch size 512 (device batch size 64, sequence length 1024), zero warmup, 1450 warmdown iterations, and a 5100-step budget. The Muon optimizer runs with momentum coefficient $\beta_{\mathrm{train}} = 0.8$, Nesterov momentum enabled, learning rate $0.05$, and five Newton--Schulz steps per update. 

The trajectory probe is applied to all 24 output projections of the model: the 12 attention output projections \texttt{h.\{0,\ldots,11\}.attn.c\_proj} (each 768$\times$768) and the 12 MLP output projections \texttt{h.\{0,\ldots,11\}.mlp.c\_proj} (each 768$\times$3072), all with buffer size $K=\Kheadline$ and analysis every 100 training steps. The stationary probe additionally covers three representative depths \texttt{h.0.attn.c\_proj}, \texttt{h.5.attn.c\_proj}, and \texttt{h.11.attn.c\_proj} at the saved checkpoints $\{$1000, 2000, 3000, 4000, 5000$\}$.

\textbf{Stationary probe.}
The NanoGPT stationary probe reuses stored checkpoints at training steps 1000, 2000, 3000, 4000, and 5000 and applies the protocol of \Cref{app:probe-protocols} to three representative attention output projections: \texttt{h.0.attn.c\_proj}, \texttt{h.5.attn.c\_proj}, and \texttt{h.11.attn.c\_proj}. The default protocol uses the sequential collection order of \Cref{app:probe-protocols}, with collection buffer of $K=500$ and $K=1000$ gradients. The stationary probe holds the weights fixed during collection, so $K$ is chosen large enough to warm the momentum buffer at the largest probe-side $\beta=0.99$ ($T=100$). At the step-3000 checkpoint we additionally rerun the same two windows under the shuffled collection order as the robustness cross-check of \Cref{app:thm1-supplement}. \Cref{fig:frozen-thm2} uses \texttt{h.0} as the representative slice, while the appendix keeps \texttt{h.5} and \texttt{h.11} as cross-layer checks.

\textbf{Trajectory probe.}
The NanoGPT trajectory probe is reported at two scopes: representative-layer evidence on a single attention output projection (\Cref{sec:experiments}), and all-layer evidence spanning every attention and MLP output projection (\Cref{app:thm2-supplement}). Both scopes share the same 3-seed run group, buffer size $K=\Kheadline$, analysis interval, and probe-side $\beta$ grids, differing only in the target layer. The representative-layer slice uses \texttt{h.0.attn.c\_proj}, matching the stationary-probe representative layer (\Cref{fig:ch1-frozen-c1-validation}). \Cref{app:thm2-supplement} aggregates over all 24 projection targets. The $K$-sweep in \Cref{app:k-sweep} verifies that Pre-polar dominance over Post-polar and Polar-only, and the Pre-polar to Polar-only alignment ratio, hold consistently for $K$ ranging from $2.5T$ to $10T$.

The probe stores per-step filtered singular-value spectra, $\sin\Theta$ subspace alignment errors at ranks $r\in$ \{1,5,10\}, and Pre-polar / Post-polar / Polar-only full-rank alignment values. Three training seeds \{1337, 1338, 1339\} are run for both instances. We report 3-seed mean and sample standard deviation.

\textbf{Probe-side $\beta$ grids.} The NanoGPT panels use the following grids. We use a six-point grid \{0.3,0.6,0.9,0.95,0.97,0.99\} for the stationary per-step filtering ratio (\Cref{fig:ch1-frozen-ratio}). We use a fourteen-point grid \{0.1,0.2,0.3,0.4,0.5,0.6,0.7,0.8,0.85,0.9,0.93,0.95,0.97,0.99\} for the stationary noise-suppression ratio (\Cref{fig:ch1-frozen-edge}) and the stationary subspace alignment error (\Cref{fig:ch1-frozen-c1-validation}). We use an eight-point grid \{0.5,0.7,0.8,0.9,0.93,0.95,0.97,0.99\} for the stationary Pre-polar / Post-polar / Polar-only signal-alignment comparison (\Cref{fig:frozen-beta-sweep}). For the trajectory subspace alignment error (\Cref{fig:online-cor1}) we use a twelve-point grid \{0.1,0.2,0.3,0.4,0.5,0.6,0.7,0.8,0.85,0.9,0.93,0.95\}, and for the trajectory Pre-polar / Post-polar / Polar-only signal-alignment comparison (\Cref{fig:online-thm2-a}) the six-point grid \{0.5,0.7,0.8,0.9,0.93,0.95\}. We display the full fourteen-point grid for the all-layer trajectory subspace alignment error (\Cref{fig:app-nanogpt-online-direction-attn,fig:app-nanogpt-online-direction-mlp}). For experiments that do not sweep $\beta$, we use $\beta=0.95$ as the default probe-side momentum coefficient.

\subsubsection{NanoGPT End-to-End Training}
\label{app:nanogpt-end2end}

The NanoGPT experiment is the first of the two end-to-end comparisons in \cref{fig:end2end}. The shared L40S training recipe (architecture, data, batch shape, schedule, precision, DDP layout, training steps) is given in \cref{tab:nanogpt-l40s}. Hidden two-dimensional matrix parameters in the transformer blocks are updated by the selected Muon pipeline (with five Newton--Schulz iterations per update), while token embeddings, the language modeling head, and scalar tensors are updated by an auxiliary fused AdamW optimizer with learning rate 0.008, betas (0.8, 0.95), and $\epsilon=10^{-10}$. Pre-polar and Post-polar additionally use the Muon momentum warmup that linearly interpolates from a fixed initial momentum 0.85 to the pipeline's best $\beta$ over the first 300 steps and enable Nesterov momentum. Polar-only carries no momentum buffer, skips the momentum warmup, and turns off Nesterov momentum. The pipeline-specific best Muon learning rate and $\beta$ are selected by the sweep below.

\paragraph{Hyperparameter sweep and best configurations.}
Before the seed-averaging stage, we ran a single-seed (\texttt{seed}=0) sweep at the full 5100-step training budget to pick the best Muon learning rate and momentum coefficient per pipeline. The main sweep covers three Muon learning rates \{0.025,\,0.05,\,0.075\} crossed with three momentum coefficients \{0.8, 0.9,\,0.95\} for Pre-polar and Post-polar pipelines (nine $(\mathrm{lr},\beta)$ cells per pipeline), and six learning rates \{0.01,\,0.015,\,0.02,\,0.025,\,0.05,\,0.075\} for Polar-only pipeline. The selected configurations reported in \cref{tab:nanogpt-end2end-best} are those that minimize the FineWeb-10B validation loss at step 5100 within the swept grid. The selected configurations were then rerun with three independent seeds (1-3) under identical settings, and the per-step three-seed mean and sample standard deviation feed the NanoGPT panel of \cref{fig:end2end}.

\begin{table}[htbp]
  \centering\small
  \caption{NanoGPT end-to-end best configurations selected from the single-seed (\texttt{seed}=0) sweep and used for the three-seed final comparison.}
  \label{tab:nanogpt-end2end-best}
  \begin{tabular}{@{}lcccc@{}}
    \toprule
    Variant & Pipeline & Muon LR & $\beta$ & Seeds \\
    \midrule
    Pre-polar & momentum $\rightarrow$ NS & 0.025 & 0.8 & 3 \\
    Post-polar & NS $\rightarrow$ momentum & 0.05 & 0.95 & 3 \\
    Polar-only & NS only & 0.02 & -- & 3 \\
    \bottomrule
  \end{tabular}
\end{table}

\subsubsection{LLaMA 350M End-to-End Training}
\label{app:llama350m-end2end}

The LLaMA 350M experiment is the second of the two end-to-end comparisons in \cref{fig:end2end}. The shared L40S training recipe (architecture, data, batch shape, schedule, precision, DDP layout, training steps) is given in \cref{tab:llama350m-l40s}. Hidden two-dimensional matrix parameters in the transformer blocks are updated by the selected Muon pipeline (with five Newton--Schulz iterations per update), while embeddings, the output head, normalization parameters, and other non-Muon tensors are updated by an auxiliary AdamW optimizer with learning rate $3\times10^{-4}$, betas (0.9,0.95), and $\epsilon=10^{-10}$. Pre-polar and Post-polar hold the Muon momentum coefficient constant at the pipeline's best $\beta$ throughout training and enable Nesterov momentum. No momentum warmup is applied here, in contrast to the NanoGPT recipe of \Cref{app:nanogpt-end2end}. Polar-only carries no momentum buffer and turns off Nesterov momentum. Validation cross-entropy is computed every 500 steps over $\sim$$10$M C4 tokens, and the reported final loss / perplexity in \cref{tab:llama350m-results} is the value at step 3000. The pipeline-specific best Muon learning rate and $\beta$ are selected by the sweep below.

\paragraph{Hyperparameter sweep and best configurations.}
Before the seed-averaging stage, we ran a single-seed (\texttt{seed}=0) sweep at the full 3000-step training budget to pick the best Muon learning rate and momentum coefficient per pipeline. The main sweep covers six Muon learning rates \{0.005,\,0.01,\,0.02,\,0.025,\,0.03,\,0.035\} crossed with four momentum coefficients \{0.8,\,0.9,\,0.95,\,0.975\} for Pre-polar and Post-polar pipelines, and the same six learning rates for Polar-only pipeline. The selected configurations reported in \cref{tab:llama350m-results} are those that minimize the C4 validation loss at step 3000 within the swept grid. The selected configurations were then rerun with three independent seeds (1-3). The per-step three-seed mean and sample standard deviation feed the LLaMA 350M panel of \cref{fig:end2end}.

\begin{table}[htbp]
  \centering\small
  \caption{LLaMA 350M end-to-end best configurations selected from the single-seed (\texttt{seed}=0) sweep and seed-averaged final validation loss / perplexity at step 3000. Final loss and perplexity are reported as mean $\pm$ sample standard deviation over completed runs.}
  \label{tab:llama350m-results}
  \begin{tabular}{@{}lccccc@{}}
    \toprule
    Variant & Pipeline & Muon LR & $\beta$ & Seeds & Final loss / PPL \\
    \midrule
    Pre-polar  & momentum $\rightarrow$ NS & 0.020 & 0.95 & 3 & 2.9696 $\pm$ 0.0002 / 19.48 \\
    Post-polar & NS $\rightarrow$ momentum & 0.020 & 0.95 & 3 & 3.1568 $\pm$ 0.0058 / 23.50 \\
    Polar-only & NS only                   & 0.035 & --     & 3 & 3.2402 $\pm$ 0.0076 / 25.54 \\
    \bottomrule
  \end{tabular}
\end{table}

\subsection{Signal-Strength Sensitivity}
\label{app:exp-tasks-signal-strength}

The signal-strength sweep of \Cref{app:signal-strength} reuses the synthetic, CIFAR-10, and NanoGPT probe protocols defined above with three task-specific configurations. All three sweeps share the probe-side momentum coefficient $\beta=0.95$ and the single-pass momentum over the stored collection order described in \Cref{app:probe-protocols}. The metric definitions are those of \Cref{app:measurements}.

\textbf{Synthetic $\lambda$ sweep (\Cref{fig:app-synth-signal-strength}).} The synthetic generator is the rank-3 spiked model shared with \cref{fig:app-synth-ordering}: $m=n=100$, planted singular values $\lambda \cdot [1, 8/12, 5/12]$ (so that $\lambda$ acts directly as the leading signal singular value), $\sigma_n=1$, BVMZOS perturbation, $K=500$ collection steps per trial, and 10 random-seed trials. The signal subspace bases $(U,V)$ are random orthogonal matrices fixed across $\lambda$ for reproducibility (seed 123). The reported metric is rank-3 subspace alignment against the planted $(U,V)$ subspaces, plotted with $\pm 1\sigma$ trial-variance bands. The swept grid is $\lambda \in $ \{0.5, 1.0, 1.5, 2.0, 3.0, 4.0, 5.0, 6.0, 8.0, 10.0, 12.0, 15.0, 18.0, 20.0\}.

\textbf{CIFAR-10 batch sweep (\Cref{fig:app-cifar-batch}).} The probe target weight matrix is layer \texttt{layer2.0.conv1} of the same ResNet-18 used in the default CIFAR-10 stationary probe. The network is warmed up for 500 steps, then frozen. Six independent probes are run with mini-batch sizes $b \in$ \{16, 32, 64, 128, 256, 512\}, each collecting $K=200$ stationary mini-batch gradients on the same frozen network. Per-batch the $K$ gradients are processed through the three pipelines at $\beta=0.95$ in collection order. The four reported panels are rank-1, rank-5, rank-10 subspace alignment against the empirical mean-gradient subspace, and full-rank signal alignment.

\textbf{NanoGPT batch sweep (\Cref{fig:app-nanogpt-batch}).} The probe target weight matrix is layer \texttt{h.0.attn.c\_proj} (matrix shape 768 $\times$ 768) of the NanoGPT step-3000 checkpoint. Six independent probes are run with mini-batch sizes $b \in $ \{16, 32, 64, 128, 256, 512\} at sequence length 1024, each collecting $K=500$ stationary mini-batch gradients with single-step accumulation. Per-batch the $K$ gradients are processed through the three pipelines at $\beta=0.95$ in collection order. The four reported panels are rank-1 subspace alignment, rank-5 and rank-10 subspace alignment, and full-rank signal alignment.

\section{Experimental Results of Spectral Gaps}
\label{app:thm1-supplement}

This section extends the validation of \Cref{thm:spectral-gap} from the NanoGPT probe of \Cref{sec:spectral-gap} to synthetic and CIFAR-10 stationary probes and to the remaining NanoGPT attention output projections and saved checkpoints. The full settings are shown in \Cref{app:synthetic-settings,app:cifar10-probe,app:nanogpt-probe}. All panels follow the curve, marker, and reference conventions of \cref{fig:ch1-frozen-t1-validation}. Across all panels, the same qualitative pattern repeats: the head of the momentum buffer's singular value spectrum tracks the signal singular values, the tail collapses toward the perturbation floor $(2T-1)^{-1/4}$ of \Cref{thm:spectral-gap} as $\beta$ grows, and the noise-suppression ratio $R(T)$ sits above the $(2T-1)^{1/4}$ rate predicted by \Cref{thm:spectral-gap}.

\paragraph{Synthetic and CIFAR-10 Stationary Probes.}
\Cref{fig:app-synth-filtering} shows the filtered singular value spectrum on the rank-3 spiked model under Gaussian-noise (left panel) and heavy-tailed Student-$t$-noise (right panel) BVMZOS perturbations, with planted singular values $\sigma_1=12$, $\sigma_2=8$, and $\sigma_3=5$ (black diamonds at $k=1,2,3$). As $\beta$ grows, the top-$3$ singular values of the momentum buffer concentrate near the true signal, while singular values at indices $k\ge 4$ are more suppressed, supporting \Cref{thm:spectral-gap}.
\Cref{fig:app-cifar-spectrum-ratio} reproduces this pattern on the CIFAR-10 stationary probe at \texttt{layer2.0.conv1} of a ResNet-18 (matricized to 128 $\times$ 576), warmup step 500, $K=2000$, with the filtered singular value spectra in~\cref{fig:app-cifar-spectrum} and the per-step filtering ratio $\mathrm{Filt}_k(\beta)$ (\Cref{app:measurements}) in~\cref{fig:app-cifar-ratio}. As $\beta$ grows, the per-step filtering ratios decrease, with the head indices suppressed less than the tail.
\Cref{fig:app-edge-pair} shows the noise-suppression ratio $R(T)$ on synthetic and CIFAR-10 stationary probe experiments. On synthetic (\Cref{fig:app-synth-edge}), where the planted $G_t^{\mathrm{sig}}$ replaces $\bar G$ in the denominator with a zero-init bias correction (\Cref{app:measurements}), $R(T)$ grows above the $(2T-1)^{1/4}$ rate under the BVMZOS perturbation. On CIFAR-10 (\Cref{fig:app-cifar-edge}), the empirical curve grows at the similar rate above the same floor across the full $\beta$ sweep, supporting \Cref{thm:spectral-gap}.

\begin{figure}[htbp]
  \centering
  \includegraphics[width=\linewidth]{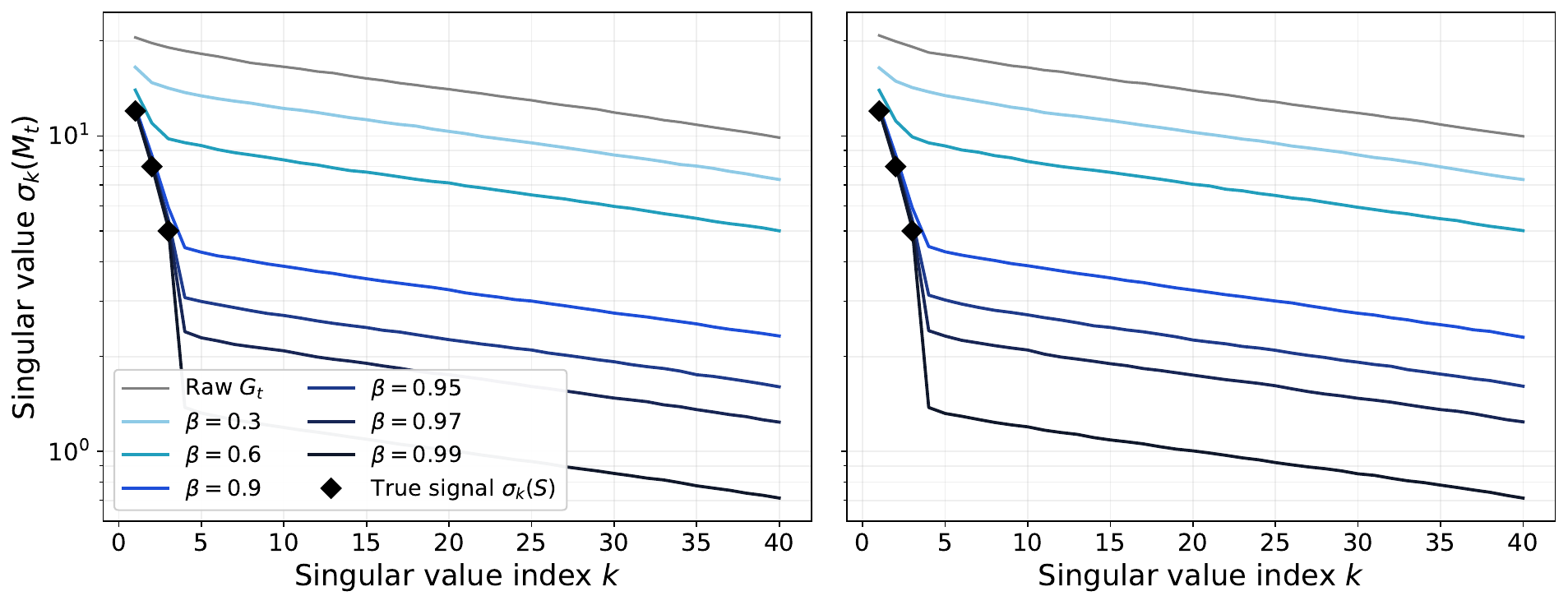}
  \caption{Synthetic stationary filtered singular value spectra under the rank-3 spiked model ($m=n=100$, $\sigma_n=1$, $K=1000$, 10-trial mean) with a BVMZOS perturbation. (Left). The rank-3 spiked model with the Gaussian noise. (Right) The rank-3 spiked model with the heavy-tailed Student-$t$ noise. The black diamonds at $k=1,2,3$ mark the planted signal singular values $\sigma_k\in\{12,8,5\}$. The experimental details are described in \Cref{app:synthetic-settings}.}
  \label{fig:app-synth-filtering}
\end{figure}

\begin{figure}[htbp]
  \centering
  \begin{minipage}[t]{0.48\linewidth}
    \centering
    \includegraphics[width=\linewidth]{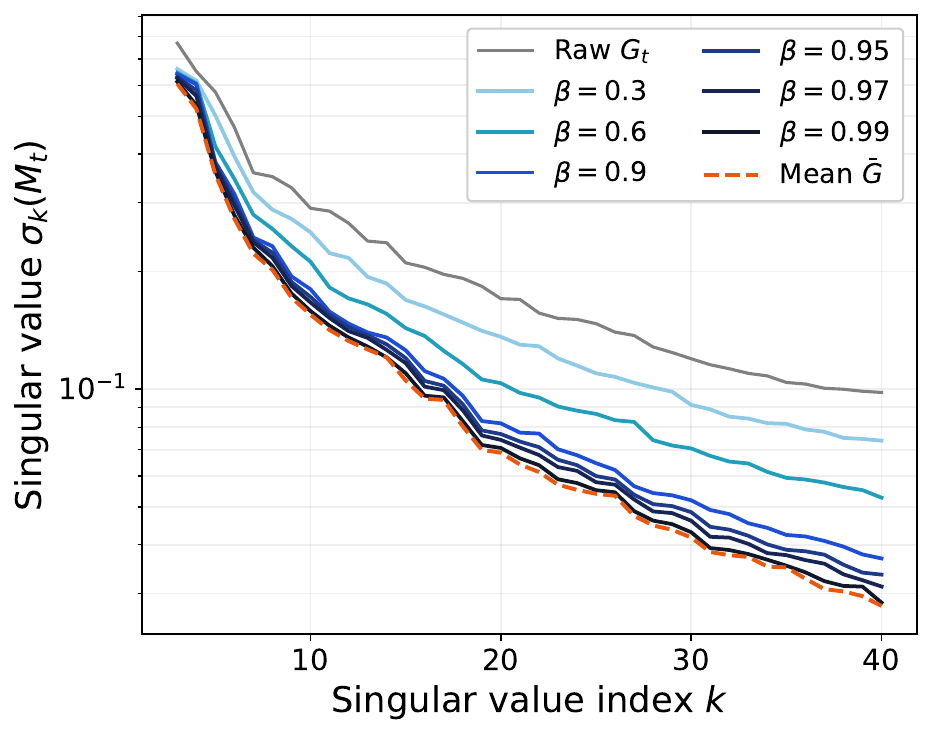}
    \subcaption{Filtered singular value spectra.}
    \label{fig:app-cifar-spectrum}
  \end{minipage}\hfill
  \begin{minipage}[t]{0.48\linewidth}
    \centering
    \includegraphics[width=\linewidth]{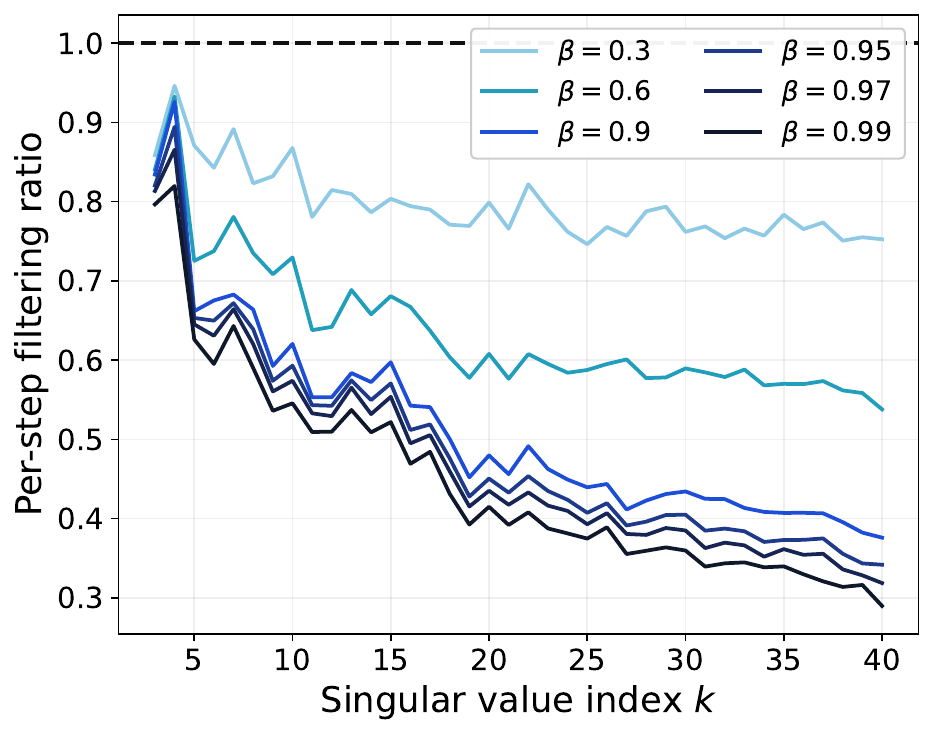}
    \subcaption{Per-step filtering ratio.}
    \label{fig:app-cifar-ratio}
  \end{minipage}
  \caption{CIFAR-10 stationary probe at \texttt{layer2.0.conv1} (128 $\times$ 576), warmup step 500, $K=2000$. Index range $k\in\{3,\dots,40\}$. Curve and reference conventions follow \cref{fig:ch1-frozen-spectrum,fig:ch1-frozen-ratio}.}
  \label{fig:app-cifar-spectrum-ratio}
\end{figure}

\begin{figure}[htbp]
  \centering
  \begin{minipage}[t]{0.48\linewidth}
    \centering
    \includegraphics[width=\linewidth]{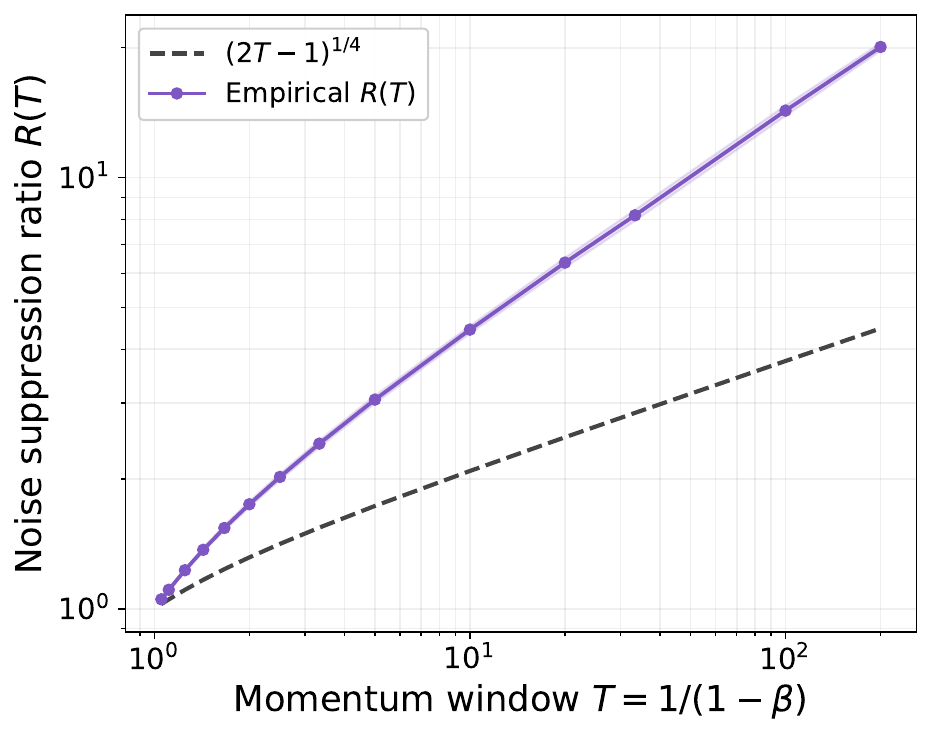}
    \subcaption{Synthetic noise-suppression ratio.}
    \label{fig:app-synth-edge}
  \end{minipage}\hfill
  \begin{minipage}[t]{0.48\linewidth}
    \centering
    \includegraphics[width=\linewidth]{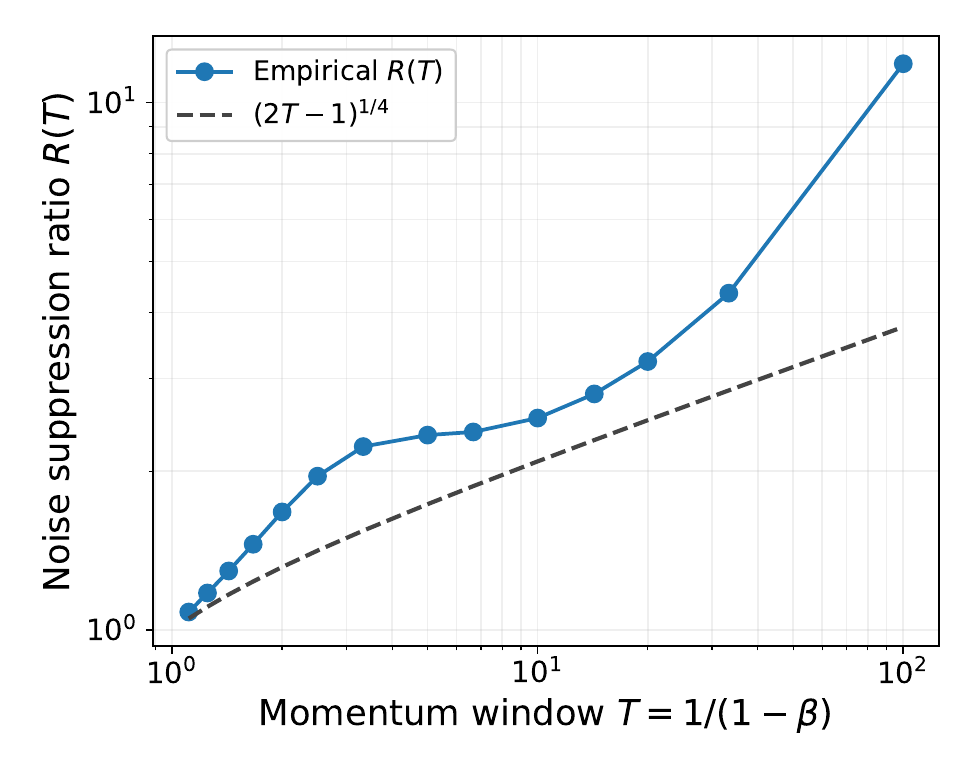}
    \subcaption{CIFAR-10 stationary noise-suppression ratio.}
    \label{fig:app-cifar-edge}
  \end{minipage}
    \caption{Noise-suppression ratio $R(T)$ (\Cref{app:measurements}) on (a) synthetic rank-3 spiked gradients ($m=n=100$, $\sigma_n=1$, 1000 steps, $10$ trials) under a BVMZOS perturbation and (b) CIFAR-10 stationary gradients at \texttt{layer2.0.conv1}. The noise-suppression ratio $R(T)$ in the synthetic simulation uses the planted signal $G_t^{\mathrm{sig}}$ in place of $\bar G$ with a zero-init bias correction (\Cref{app:measurements}). Dashed line: $(2T-1)^{1/4}$ floor.}
  \label{fig:app-edge-pair}
\end{figure}

\paragraph{NanoGPT Stationary Probes Across Layers and Checkpoints.}
\Cref{fig:app-nanogpt-spectra,fig:app-nanogpt-ratio} extend \cref{fig:ch1-frozen-spectrum,fig:ch1-frozen-ratio}, respectively, to a 3 $\times$ 5 grid over three attention output projections (\texttt{h.0}, \texttt{h.5}, \texttt{h.11}) as rows and five saved training checkpoints (steps 1000, 2000, 3000, 4000, and 5000) as columns, $K=500$ per cell. At every cell, as $\beta$ grows, the buffer spectrum moves toward the mean-gradient spectrum $\sigma_k(\bar G)$ with the head reaching $\sigma_k(\bar G)$ ahead of the tail, and the per-step filtering ratios decrease with the head suppressed less than the tail. This pattern, where the head filtering ratios are suppressed less than the tail, is consistent with the spectral gap predicted by \Cref{thm:spectral-gap}.
\Cref{fig:app-nanogpt-edge} extends \cref{fig:ch1-frozen-edge} to the early- and late-training checkpoints (steps 1000 and 5000) on the same three attention output projections. At both checkpoints, $R(T)$ grows at the rate above the dashed $(2T-1)^{1/4}$ floor. The depth ordering of \cref{fig:ch1-frozen-edge} (\texttt{h.0} furthest above the floor, \texttt{h.11} closest to it) is preserved across training.

\begin{figure}[htbp]
  \centering
  \includegraphics[width=\linewidth]{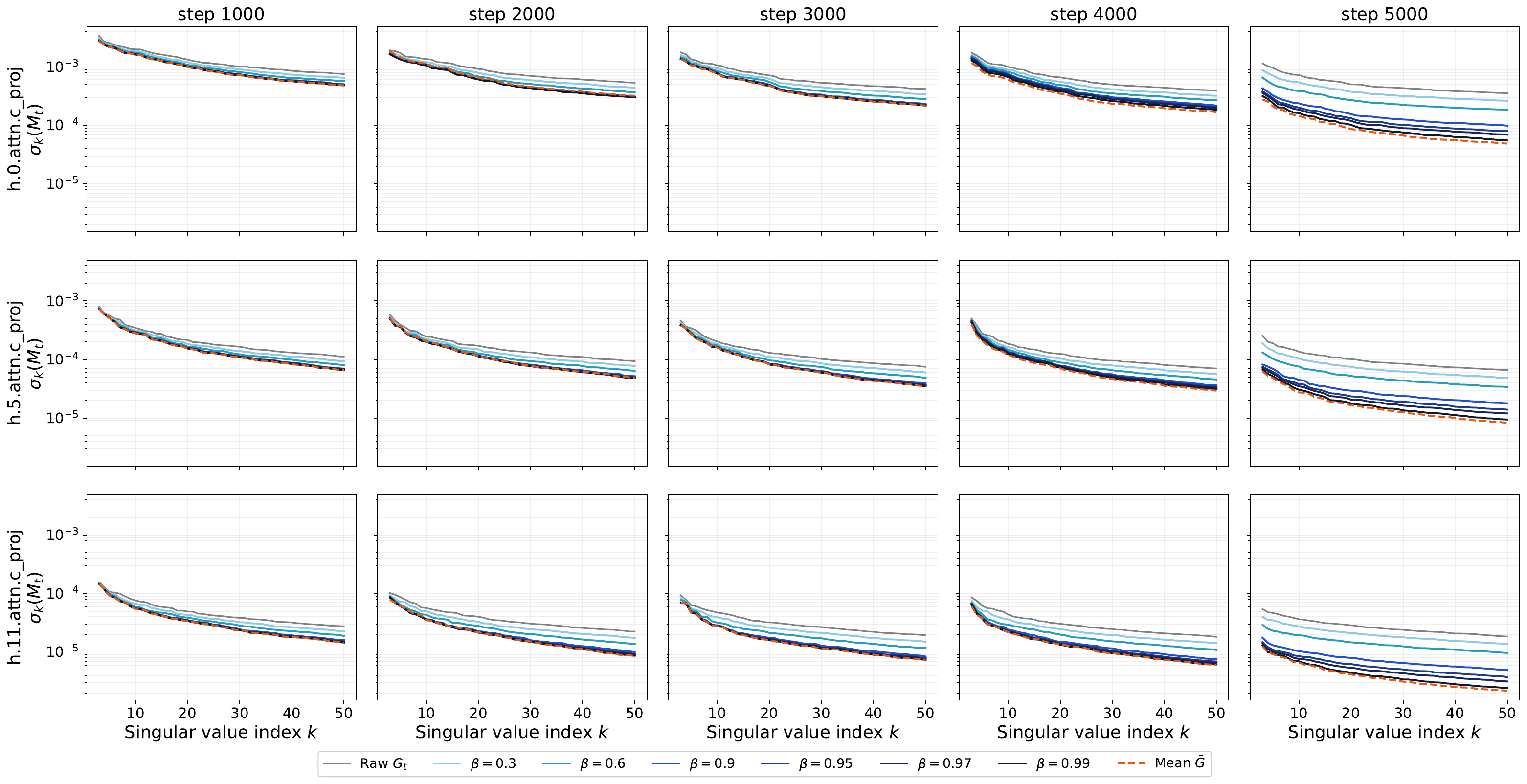}
  \caption{Stationary NanoGPT filtered singular value spectra over attention output projections \texttt{h.0}, \texttt{h.5}, \texttt{h.11} (rows) and training checkpoints 1000, 2000, 3000, 4000, and 5000 (columns), $K=500$. Mean-gradient spectrum $\sigma_k(\bar G)$ shown in dashed orange. Axes are shared across all fifteen cells.}
  \label{fig:app-nanogpt-spectra}
\end{figure}

\begin{figure}[htbp]
  \centering
  \includegraphics[width=\linewidth]{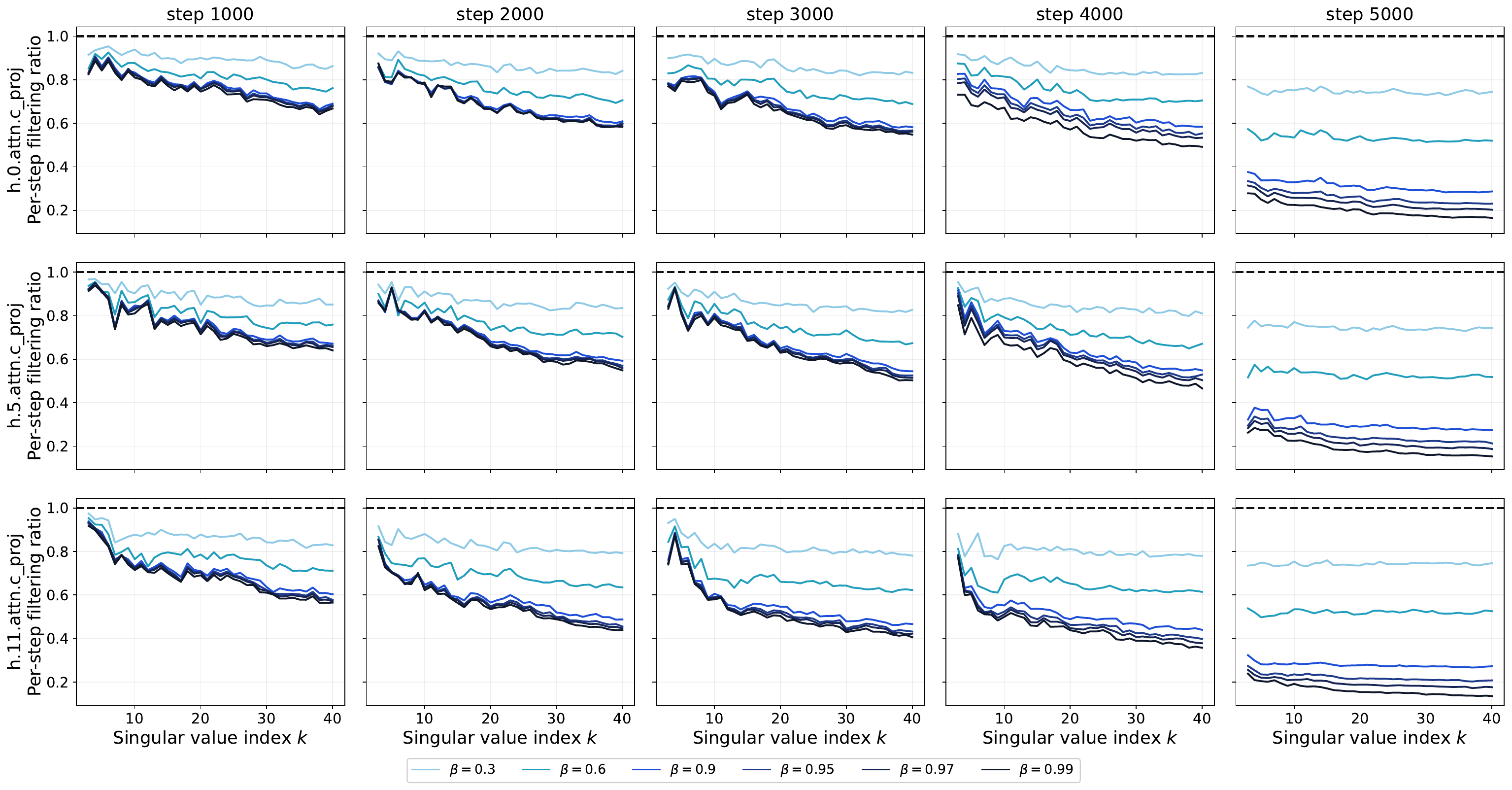}
  \caption{Stationary NanoGPT per-step filtering ratio $\mathrm{Filt}_k(\beta) = \sigma_k(M_K^{(\beta)})/\sigma_k(G_K)$ over the same $(\text{layer}, \text{step})$ grid as \cref{fig:app-nanogpt-spectra}. Dashed reference at $y=1$. Axes are shared across all fifteen cells.}
  \label{fig:app-nanogpt-ratio}
\end{figure}

\begin{figure}[htbp]
  \centering
  \includegraphics[width=\linewidth]{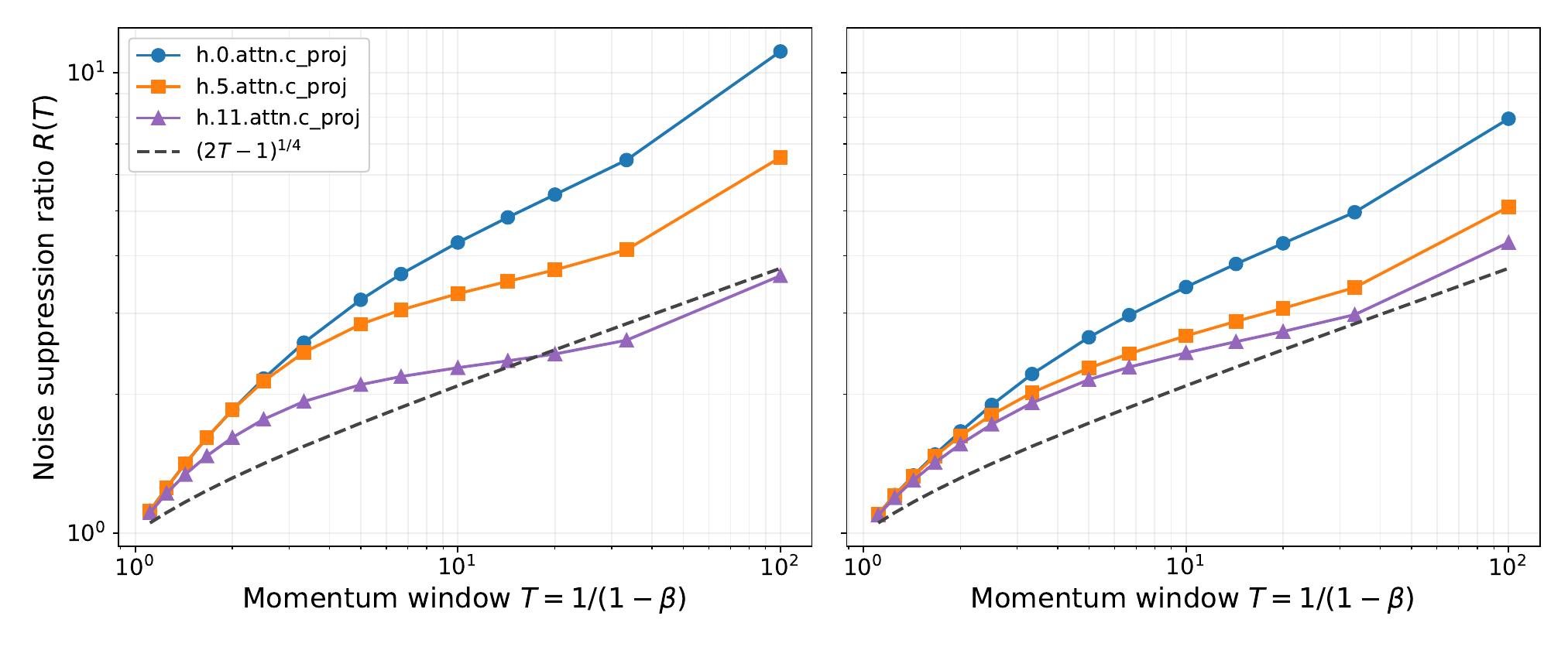}
  \caption{Stationary NanoGPT noise-suppression ratio $R(T)$ (\Cref{app:measurements}) at training checkpoints (a) step 1000 and (b) step 5000. Three attention output projections per panel. Dashed line: $(2T-1)^{1/4}$ floor.}
  \label{fig:app-nanogpt-edge}
\end{figure}


\section{Experimental Results of Subspace Alignment}
\label{app:cor1-supplement}

In this section, we extend the subspace alignment experiments (theoretically suggested by \Cref{cor:direction}) provided in \cref{fig:ch1-frozen-c1-validation,fig:online-cor1} for the representative NanoGPT layer, by additionally running the same experiments on synthetic and CIFAR-10 stationary probes, the remaining stationary NanoGPT layers, and the trajectory probes on CIFAR-10 and NanoGPT. All panels follow the curve, marker, and reference conventions of \cref{fig:ch1-frozen-c1-validation} for stationary probe and \cref{fig:online-cor1} for trajectory probe. Each dashed reference is the rate of $c_r\,(2T-1)^{-1/4}$ guide with a fitted $c_r$, which corresponds to the high-probability \Cref{cor:direction} bound. The signal reference is the planted top-$r$ singular subspace $(U_{\mathrm{true}},V_{\mathrm{true}})$ on the synthetic simulation (where ground truth is available) and the empirical mean gradient $\bar G$ on CIFAR-10 and NanoGPT. Across all panels, $\sin\Theta_U$ and $\sin\Theta_V$ decrease as the momentum window size $T$ grows.

\paragraph{Synthetic and CIFAR-10 Stationary Probes.}
\Cref{fig:app-synth-direction} shows the subspace alignment errors $\sin\Theta_U$ and $\sin\Theta_V$ on the rank-3 spiked model shared with \cref{fig:app-synth-filtering} at ranks $r\in\{1,2,3\}$, measured against the planted top-$r$ singular subspace $(U_{\mathrm{true}},V_{\mathrm{true}})$. Both errors decrease as $\beta$ grows at all ranks, supporting \Cref{cor:direction}.
\Cref{fig:app-cifar-direction} shows the subspace alignment errors $\sin\Theta_U$ and $\sin\Theta_V$ on the CIFAR-10 stationary probe at \texttt{layer2.0.conv1} of a ResNet-18 (matricized to 128 $\times$ 576), warmup step 500, $K=2000$ at ranks $r\in$ \{1,5,10\}. Both errors decrease as $\beta$ grows, supporting \Cref{cor:direction}.

\begin{figure}[htbp]
  \centering
  \includegraphics[width=0.85\linewidth]{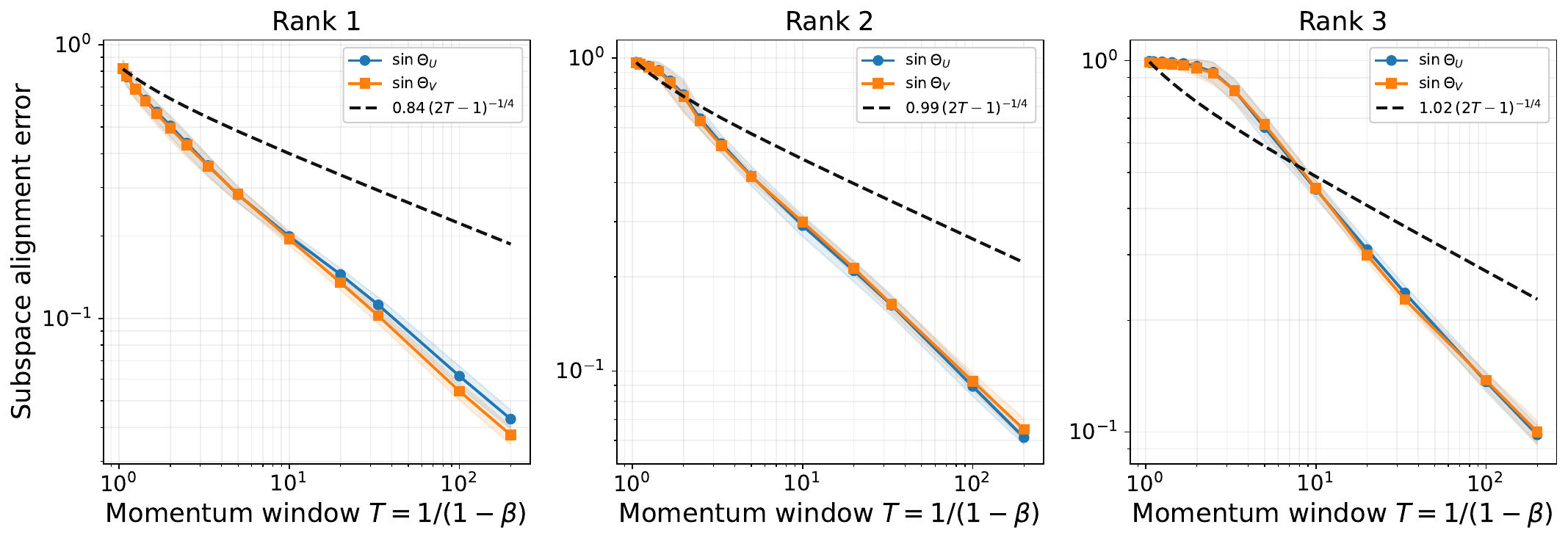}
  \caption{Subspace alignment error on the synthetic rank-3 spiked model under a BVMZOS perturbation. Panels report ranks $r\in\{1,2,3\}$. $\sin\Theta_U$ (blue) and $\sin\Theta_V$ (orange) are computed against the planted top-$r$ singular subspace $(U_{\mathrm{true}},V_{\mathrm{true}})$. Dashed line: fitted $c_r\,(2T-1)^{-1/4}$ guide. Shaded bands: $\pm 1$ trial standard deviation across 10 random-seed trials.}
  \label{fig:app-synth-direction}
\end{figure}

\begin{figure}[htbp]
  \centering
  \includegraphics[width=0.96\linewidth]{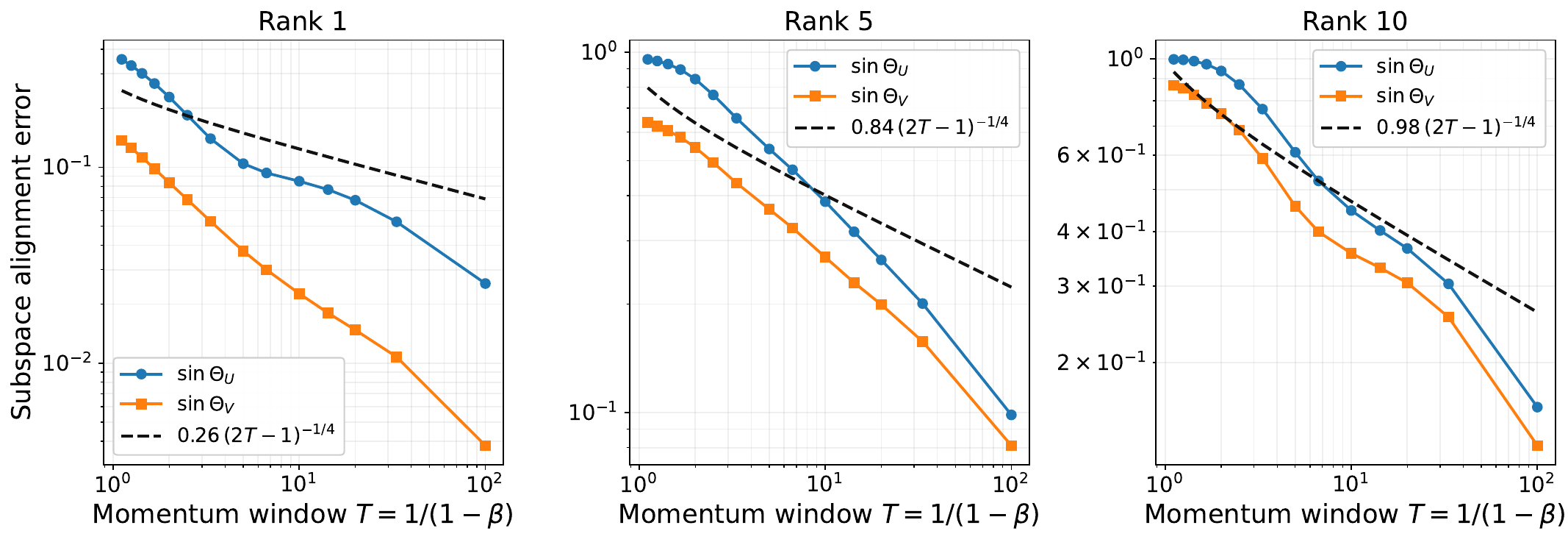}
  \caption{CIFAR-10 stationary subspace alignment error on \texttt{layer2.0.conv1} (128 $\times$ 576), warmup step 500, $K=2000$, at ranks $r\in$ \{1,5,10\}. Curve and reference conventions follow \cref{fig:ch1-frozen-c1-validation}.}
  \label{fig:app-cifar-direction}
\end{figure}

\paragraph{NanoGPT Stationary Probes Across Layers.}
\Cref{fig:app-nanogpt-direction} extends \cref{fig:ch1-frozen-c1-validation} to the two remaining attention output projections \texttt{h.5.attn.c\_proj} and \texttt{h.11.attn.c\_proj} at step~3000, $K=500$ at ranks $r\in$ \{1,5,10\}.

\begin{figure}[htbp]
  \centering
  \includegraphics[width=\linewidth]{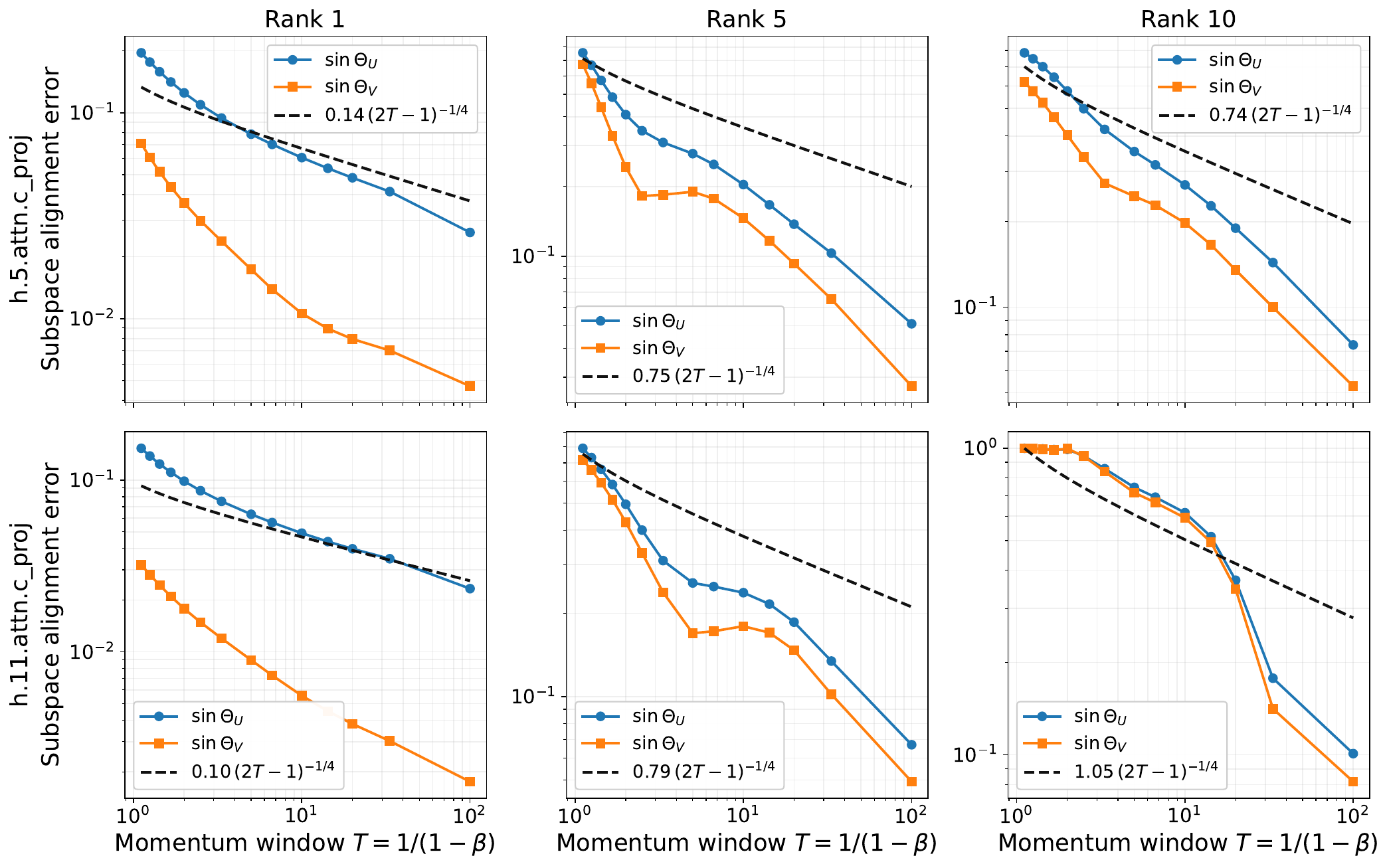}
  \caption{Stationary NanoGPT subspace alignment error on attention output projections \texttt{h.5.attn.c\_proj} (top) and \texttt{h.11.attn.c\_proj} (bottom) at checkpoint step~3000, $K=500$, at ranks $r\in$ \{1,5,10\} (columns). The representative \texttt{h.0.attn.c\_proj} panel is in the main text as \cref{fig:ch1-frozen-c1-validation}.}
  \label{fig:app-nanogpt-direction}
\end{figure}

\paragraph{CIFAR-10 Trajectory Probes.}
\Cref{fig:app-cifar-online-direction} shows the subspace alignment errors $\sin\Theta_U$ and $\sin\Theta_V$ on the CIFAR-10 trajectory probe at \texttt{layer2.0.conv1}, $K=100$, $I=100$, final analysis step $t=1500$, six-seed mean (seeds 42--47). Both errors decrease as $\beta$ grows at all ranks, confirming that the rank-$r$ subspace reliability bound survives the local-in-time relaxation of \Cref{ass:decomp}(a) and supporting \Cref{cor:direction}.

\begin{figure}[htbp]
  \centering
  \includegraphics[width=0.96\linewidth]{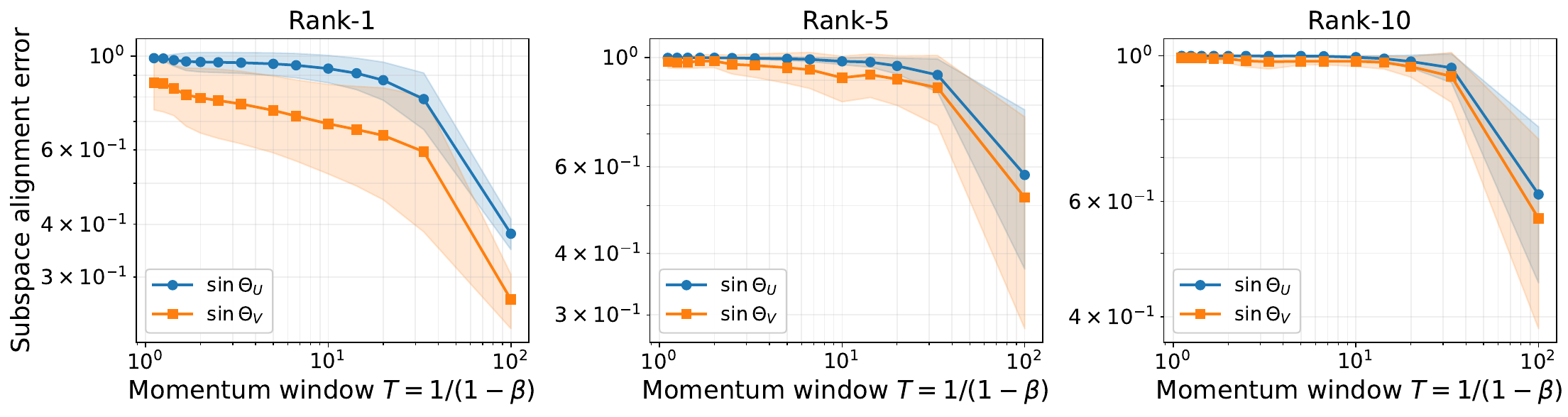}
  \caption{CIFAR-10 trajectory subspace alignment error at training step 1500 on \texttt{layer2.0.conv1}, six-seed mean (seeds 42--47, $K=100$), at ranks $r\in$ \{1,5,10\}. Solid lines: six-seed mean. Shaded bands: sample standard deviation across seeds.}
  \label{fig:app-cifar-online-direction}
\end{figure}

\paragraph{NanoGPT Trajectory Probes Across Layers.}
\Cref{fig:app-nanogpt-online-direction-attn,fig:app-nanogpt-online-direction-mlp} extend \cref{fig:online-cor1} to the two remaining attention output projections and the three MLP output projections at training step~3000, $K=\Kheadline$, $I=100$, three-seed mean with sample standard deviation bands across seeds.

\begin{figure}[htbp]
  \centering
  \includegraphics[width=\linewidth]{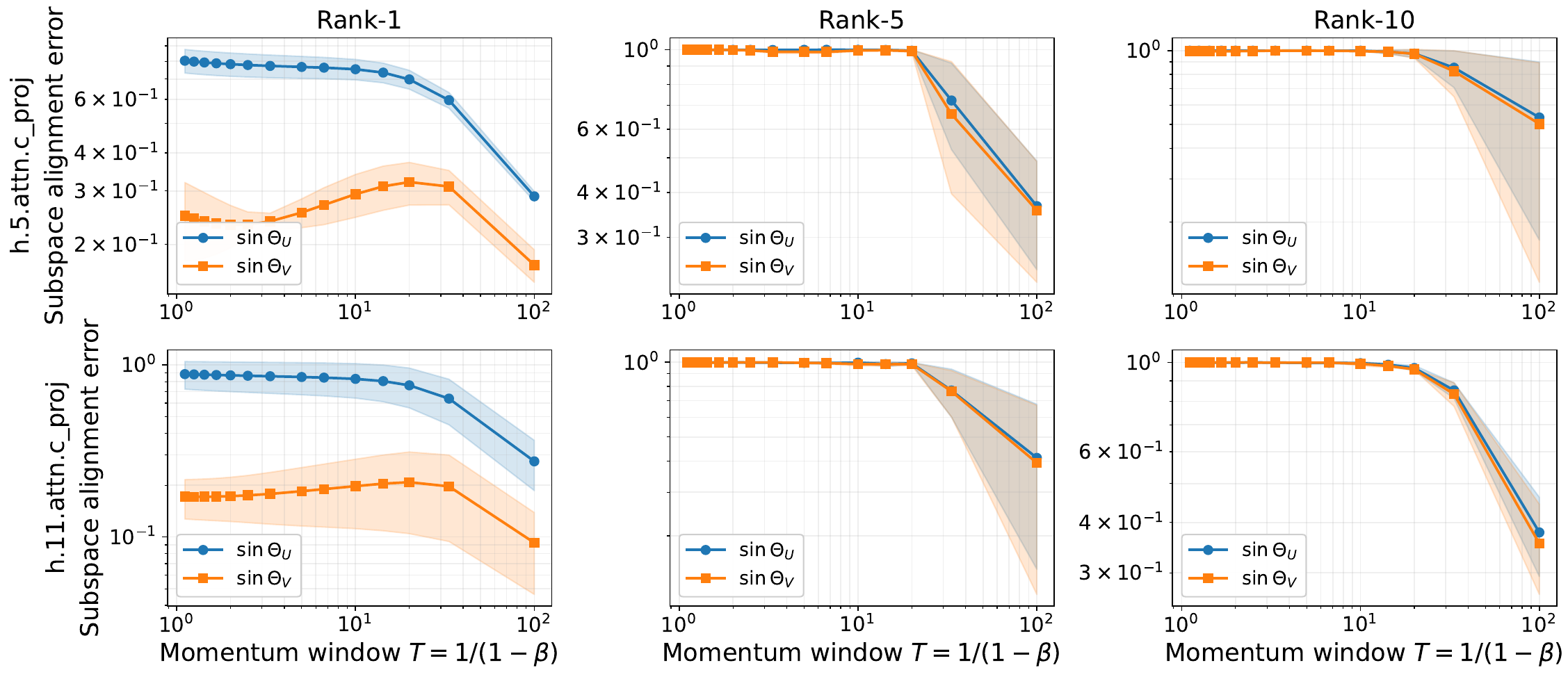}
  \caption{NanoGPT trajectory subspace alignment error on attention output projections \texttt{h.5.attn.c\_proj} (top) and \texttt{h.11.attn.c\_proj} (bottom) at training step~3000, trajectory buffer $K=\Kheadline$, at ranks $r\in$ \{1,5,10\} (columns). 3-seed mean with $\pm 1$ sample-standard-deviation bands.}
  \label{fig:app-nanogpt-online-direction-attn}
\end{figure}

\begin{figure}[htbp]
  \centering
  \includegraphics[width=\linewidth]{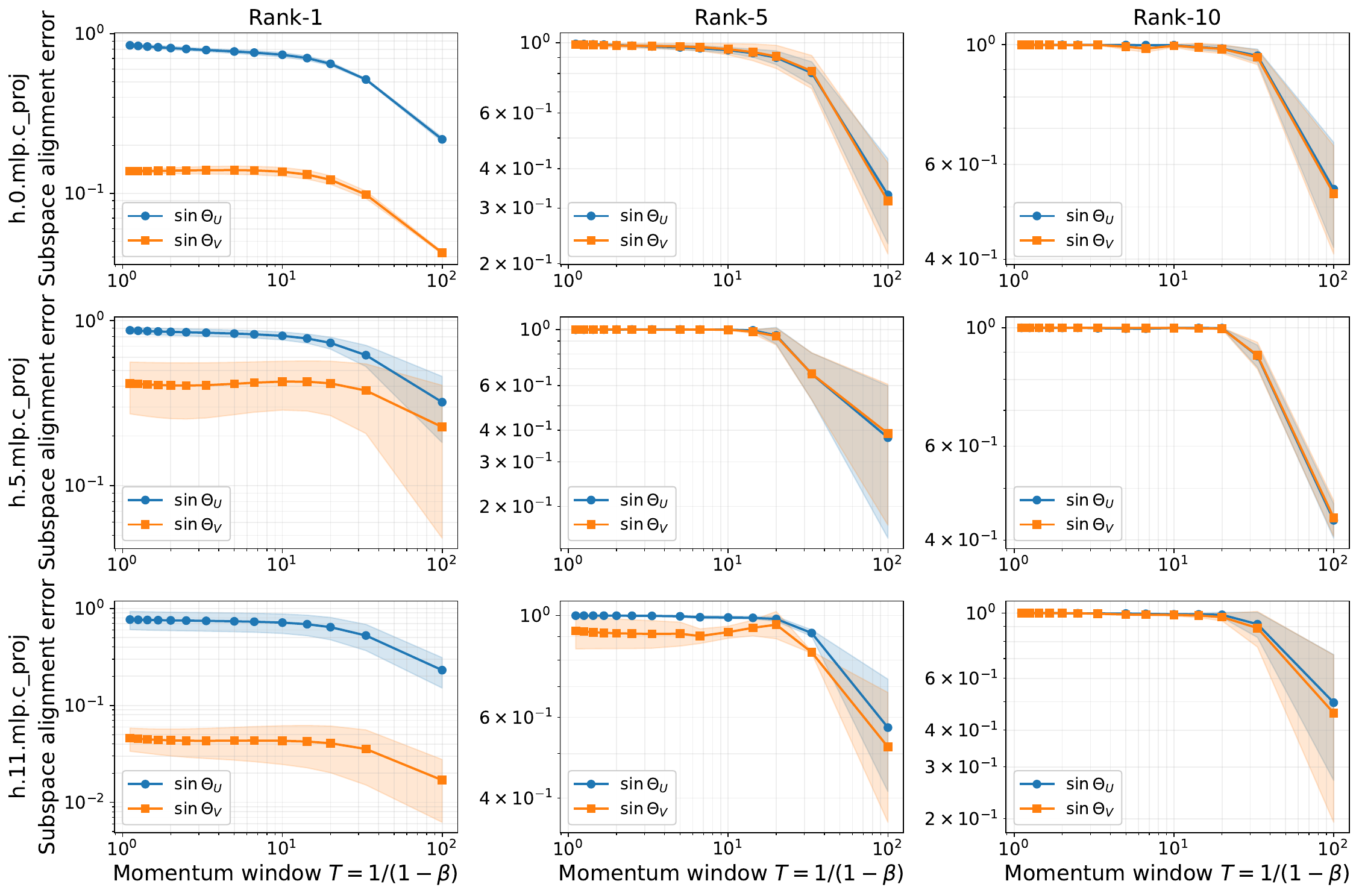}
  \caption{NanoGPT trajectory subspace alignment error on MLP output projections \texttt{h.0.mlp.c\_proj} (top), \texttt{h.5.mlp.c\_proj} (middle), \texttt{h.11.mlp.c\_proj} (rows) at training step~3000, trajectory buffer $K=\Kheadline$, at ranks $r\in$ \{1,5,10\} (columns). Same plotting conventions as \cref{fig:app-nanogpt-online-direction-attn}.}
  \label{fig:app-nanogpt-online-direction-mlp}
\end{figure}


\section{Experimental Results of Signal Alignment Ordering}
\label{app:thm2-supplement}

In this section, we extend the signal alignment experiments (theoretically suggested by \Cref{thm:recovery}) provided in \cref{fig:frozen-thm2,fig:online-thm2} for the representative NanoGPT layer, by additionally running the same experiments on synthetic and CIFAR-10 stationary probes, the remaining stationary NanoGPT layers and checkpoints, and the trajectory probes on CIFAR-10 and NanoGPT. All panels follow the curve, marker, and reference conventions of \cref{fig:frozen-thm2} for stationary probe and \cref{fig:online-thm2} for trajectory probe. The signal reference is the planted top-$r$ singular subspace $(U_{\mathrm{true}},V_{\mathrm{true}})$ on the synthetic simulation (where ground truth is available) and the empirical mean gradient $\bar G$ on CIFAR-10 and NanoGPT. \Cref{thm:recovery} predicts only Pre-polar dominance over both non-denoised pipelines. The relative position of Post-polar and Polar-only is layer- and rank-dependent on real gradients and is not predicted by the theorem. Across all panels, the Pre-polar signal alignment rises monotonically with $\beta$ and dominates the Post-polar and Polar-only signal alignments, while the Post-polar and Polar-only signal alignments stay nearly flat at low values.

\paragraph{Synthetic and CIFAR-10 Stationary Probes.}
\Cref{fig:app-synth-ordering} shows the signal alignment $\mathrm{Align}_r$ on the rank-3 spiked model shared with \cref{fig:app-synth-filtering} at ranks $r\in\{1,2,3\}$. The Pre-polar signal alignment rises monotonically with $\beta$ at every rank, while the Post-polar signal alignment declines or is nearly flat in $\beta$ and the Polar-only signal alignment is $\beta$-independent by construction, supporting \Cref{thm:recovery}.
\Cref{fig:app-cifar-ordering} supports \Cref{thm:recovery} on the CIFAR-10 stationary probe at \texttt{layer2.0.conv1} of a ResNet-18 (matricized to 128 $\times$ 576), warmup step 500, $K=2000$ at ranks $r\in$ \{1,5,10\} and the full-rank signal alignment. The Pre-polar signal alignment dominates both the signal alignments of non-denoised pipelines at every panel. The relative position of the Post-polar signal alignment and the Polar-only signal alignment varies across ranks on this layer.

\begin{figure}[htbp]
  \centering
  \includegraphics[width=0.95\linewidth]{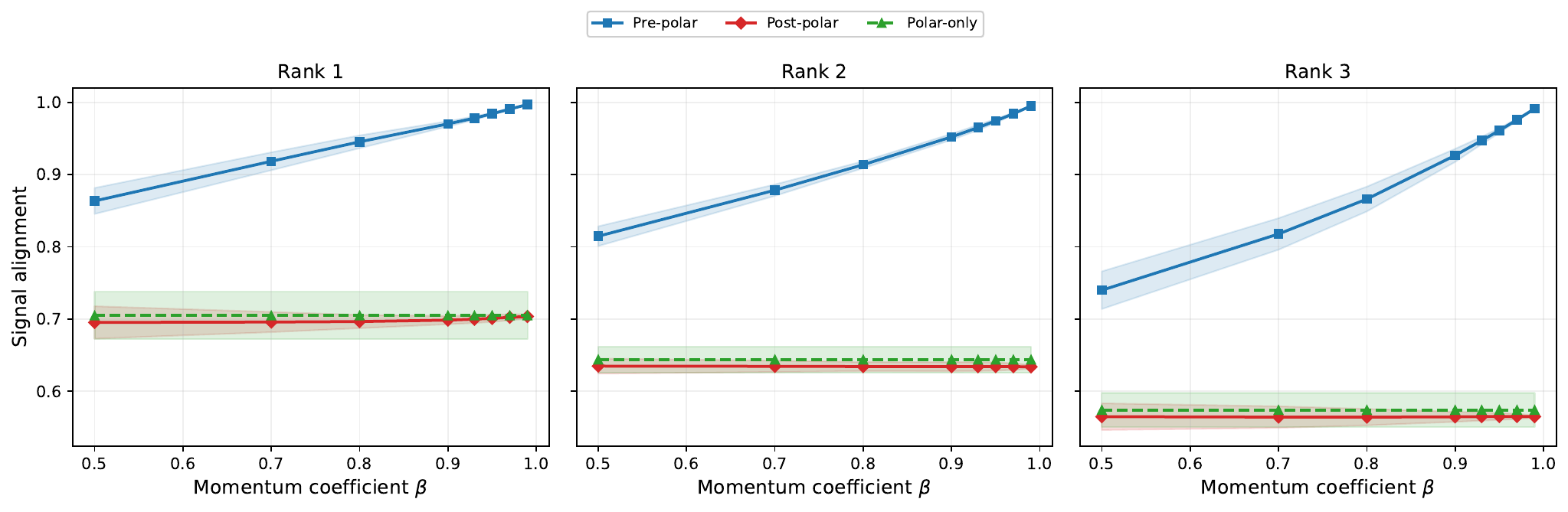}
  \caption{Synthetic signal alignment versus momentum coefficient $\beta$ on the rank-3 spiked model ($m=n=100$, $\sigma_n=1$, $K=1000$, 10 random-seed trials) under a BVMZOS perturbation. Panels report ranks $r\in\{1,2,3\}$ against the planted top-$r$ singular subspace $(U_{\mathrm{true}},V_{\mathrm{true}})$. Curve and reference conventions follow \cref{fig:frozen-thm2}. Shaded bands show trial standard deviation across the 10 trials.}
  \label{fig:app-synth-ordering}
\end{figure}

\begin{figure}[htbp]
  \centering
  \includegraphics[width=0.8\linewidth]{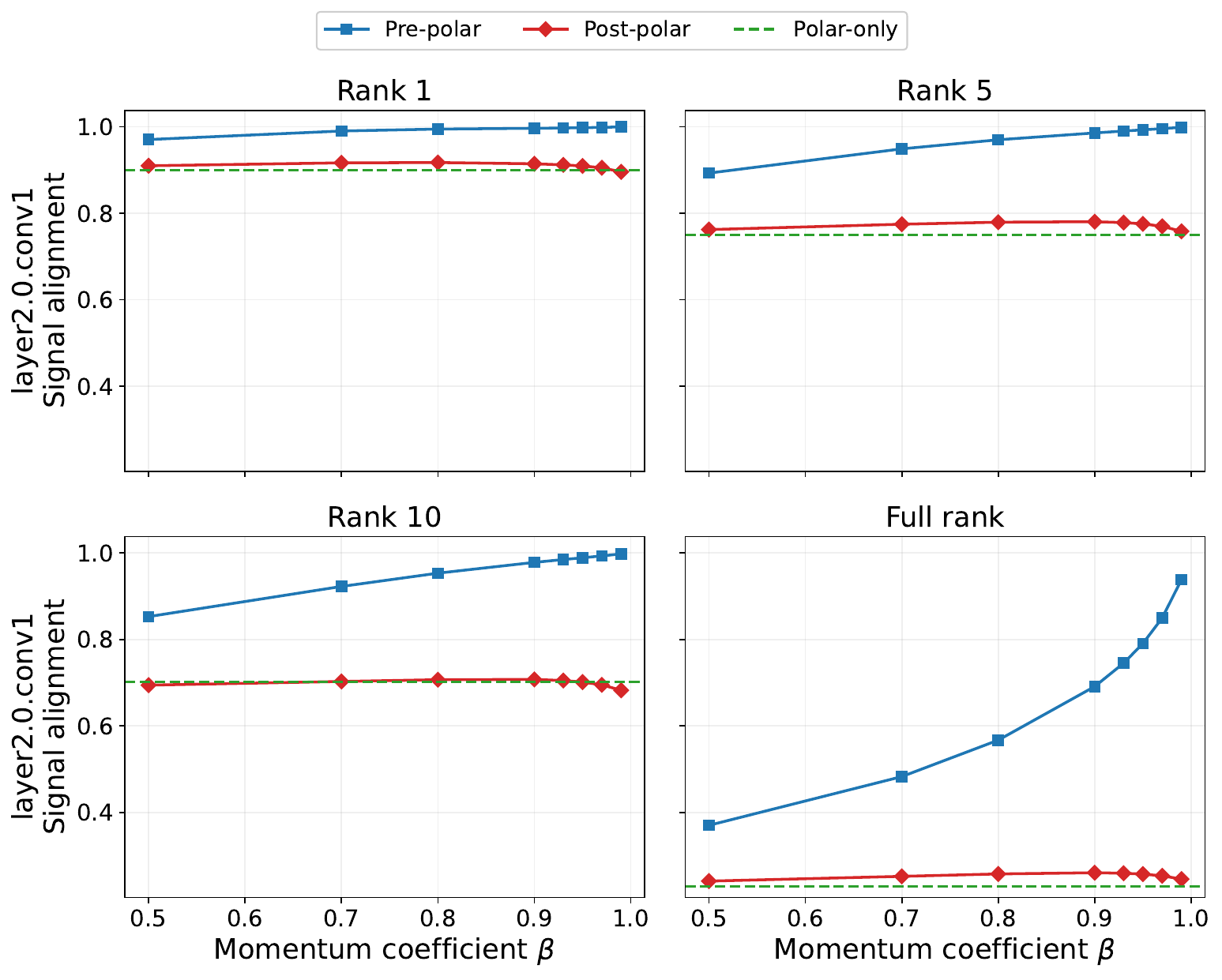}
  \caption{CIFAR-10 stationary signal alignment on \texttt{layer2.0.conv1} (128 $\times$ 576), warmup step 500, $K=2000$, at ranks $r\in$ \{1,5,10\} and the full-rank signal alignment. Curve and reference conventions follow \cref{fig:frozen-thm2}.}
  \label{fig:app-cifar-ordering}
\end{figure}

\paragraph{NanoGPT Stationary Probes Across Layers and Checkpoints.}
\Cref{fig:app-nanogpt-ordering-grid} extends \cref{fig:frozen-thm2} to a $3\times 4$ grid over three attention output projections (\texttt{h.0}, \texttt{h.5}, \texttt{h.11}) as rows and four ranks (rank-1, rank-5, rank-10, full-rank) as columns at step 3000, $K=500$. Pre-polar dominates both non-denoised pipelines at every panel. Rank-5 and rank-10 are the stable subspace ranks. Rank-1 is unstable on layers with a small $\sigma_1/\sigma_2$ gap.
\Cref{tab:frozen-numerical,tab:frozen-breakdown,tab:frozen-checkpoint} report the corresponding numerical summaries at $\beta=0.95$ across the three attention output projections, the four (collection order, $K$) settings, and the five checkpoints. The Pre-polar vs. Post-polar gap is positive at every collection-order, layer, and checkpoint, and grows from step 1000 to step 5000.

\begin{figure}[htbp]
  \centering
  \includegraphics[width=\linewidth]{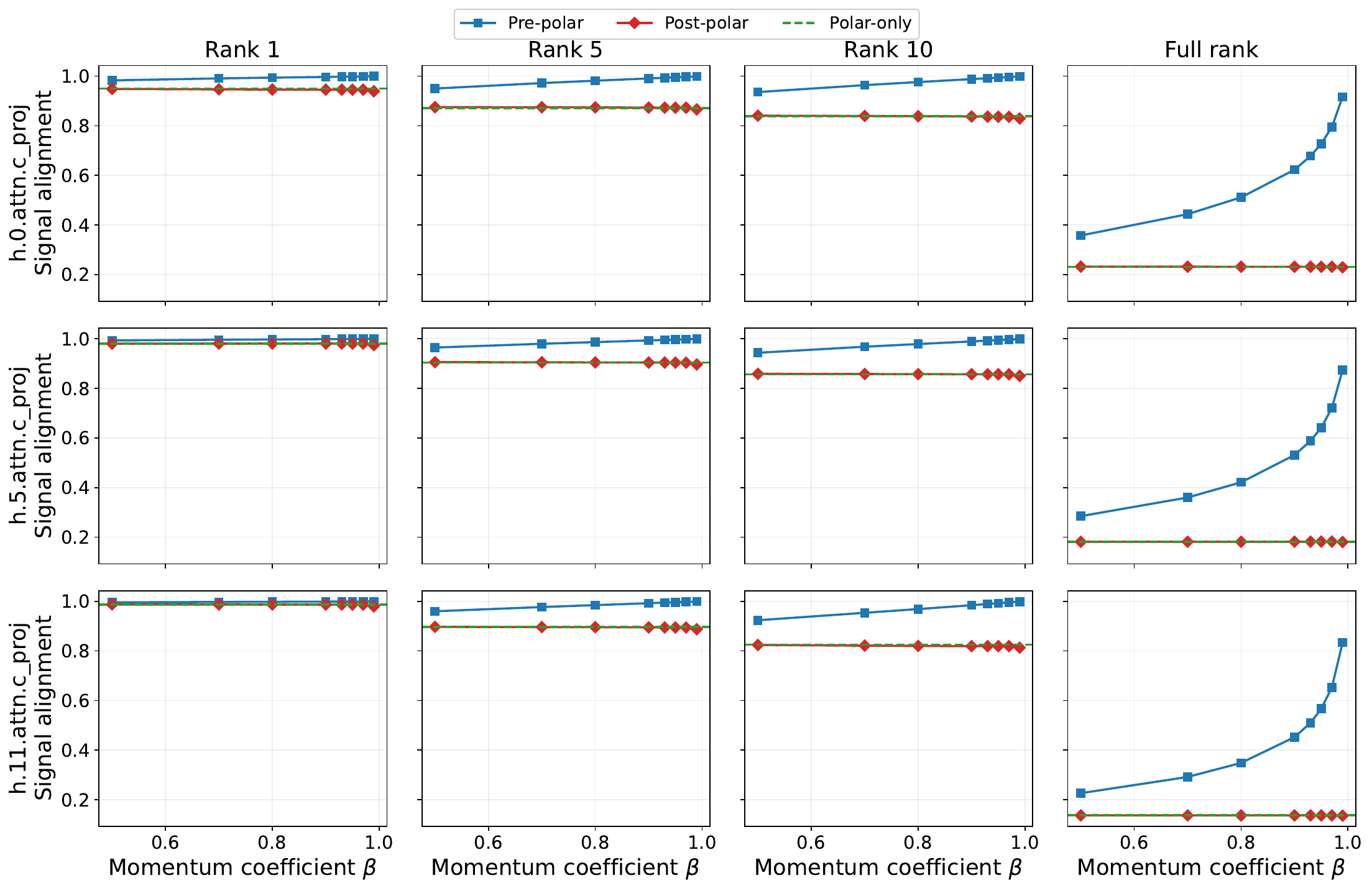}
  \caption{Stationary NanoGPT signal alignment over attention output projections \texttt{h.0} (top), \texttt{h.5} (middle), \texttt{h.11} (rows) and ranks $r\in$ \{1,5,10\} plus full rank signal alignment (columns) at checkpoint step 3000, $K=500$. Curve and reference conventions follow \cref{fig:frozen-thm2}. All twelve cells share the same 8-point $\beta$ grid \{0.5, 0.7, 0.8, 0.9, 0.93, 0.95, 0.97, 0.99\}.}
  \label{fig:app-nanogpt-ordering-grid}
\end{figure}

\begin{table}[htbp]
    \centering\small
    \caption{Fixed-rank stationary probe averages at step 3000 across the three attention output projections, at $\beta=0.95$. Gap columns report Pre-polar vs. Post-polar signal alignment at the indicated rank.}
    \label{tab:frozen-numerical}
    \begin{tabular}{@{}lccc@{}}
        \toprule
        Collection order & Gap r1 & Gap r5 & Gap r10 \\
        \midrule
        sequential, 500 gradients  & 0.0277 & 0.1057 & 0.1562 \\
        sequential, 1000 gradients & 0.0228 & 0.0963 & 0.1432 \\
        shuffled, 500 gradients    & 0.0296 & 0.1102 & 0.1619 \\
        shuffled, 1000 gradients   & 0.0291 & 0.1091 & 0.1595 \\
        \bottomrule
    \end{tabular}
\end{table}

\begin{table}[htbp]
    \centering\small
    \caption{Stationary rank-5 Pre-polar and Post-polar signal alignments and their gap per attention output projection at step 3000, $\beta=0.95$, one row per (collection order, layer) pair under collection orders $\{$sequential, shuffled$\}$ at $K\in\{500, 1000\}$ gradients.}
    \label{tab:frozen-breakdown}
    \begin{tabular}{@{}llccc@{}}
        \toprule
        Collection order & Layer & Pre r5 & Post r5 & Gap r5 \\
        \midrule
        sequential, 500 gradients & h.0 & 0.9953 & 0.8728 & 0.1224 \\
        sequential, 500 gradients & h.5 & 0.9966 & 0.9038 & 0.0928 \\
        sequential, 500 gradients & h.11 & 0.9963 & 0.8945 & 0.1018 \\
        sequential, 1000 gradients & h.0 & 0.9845 & 0.8771 & 0.1075 \\
        sequential, 1000 gradients & h.5 & 0.9905 & 0.9056 & 0.0849 \\
        sequential, 1000 gradients & h.11 & 0.9946 & 0.8981 & 0.0965 \\
        shuffled, 500 gradients & h.0 & 0.9903 & 0.8631 & 0.1272 \\
        shuffled, 500 gradients & h.5 & 0.9936 & 0.8962 & 0.0974 \\
        shuffled, 500 gradients & h.11 & 0.9952 & 0.8893 & 0.1059 \\
        shuffled, 1000 gradients & h.0 & 0.9918 & 0.8641 & 0.1277 \\
        shuffled, 1000 gradients & h.5 & 0.9949 & 0.8990 & 0.0959 \\
        shuffled, 1000 gradients & h.11 & 0.9953 & 0.8916 & 0.1036 \\
        \bottomrule
    \end{tabular}
\end{table}

\begin{table}[htbp]
    \centering\small
    \caption{Stationary Pre-polar vs. Post-polar signal alignment gap averaged across the three attention output projections at $\beta=0.95$, sequential collection order, $K=500$, at each of the five checkpoints.}
    \label{tab:frozen-checkpoint}
    \begin{tabular}{@{}lccc@{}}
        \toprule
        Checkpoint step & Gap r1 & Gap r5 & Gap r10 \\
        \midrule
        1000 & 0.0150 & 0.0613 & 0.0897 \\
        2000 & 0.0245 & 0.0959 & 0.1378 \\
        3000 & 0.0277 & 0.1057 & 0.1562 \\
        4000 & 0.0360 & 0.1197 & 0.1771 \\
        5000 & 0.3546 & 0.4921 & 0.5291 \\
        \bottomrule
    \end{tabular}
\end{table}

\paragraph{CIFAR-10 Trajectory Probes.} \Cref{fig:app-cifar-online-ordering} supports \Cref{thm:recovery} on the CIFAR-10 trajectory probe at \texttt{layer2.0.conv1} at the final analysis checkpoint (training step 1500, $K=100$). Panels report rank-5 signal alignment and full-rank signal alignment, both against $\cO(\bar G)$ of the same buffer. The Pre-polar signal alignment grows monotonically with $\beta$ on both ranks, while the Post-polar signal alignment declines or is nearly flat in $\beta$ and the Polar-only signal alignment is $\beta$-independent by construction.
\Cref{fig:app-cifar-online-history} extends the trajectory ordering across all 15 analysis checkpoints at $\beta=0.95$ ($T=20$, $I=100$). The Pre-polar vs. Post-polar full-rank alignment gap stays at 0.480 $\pm$ 0.011 across the trajectory and Pre-polar vs. Polar-only gap at 0.474 $\pm$ 0.012. Pre-polar advantage at $\beta=0.95$ persists across training steps, mirroring the NanoGPT trajectory result of \cref{fig:online-thm2-b} and validating that Pre-polar signal alignment's dominance over both the signal alignments of Post-polar and Polar-only baselines survives the non-stationarity introduced by live training.

\begin{figure}[htbp]
  \centering
  \includegraphics[width=0.96\linewidth]{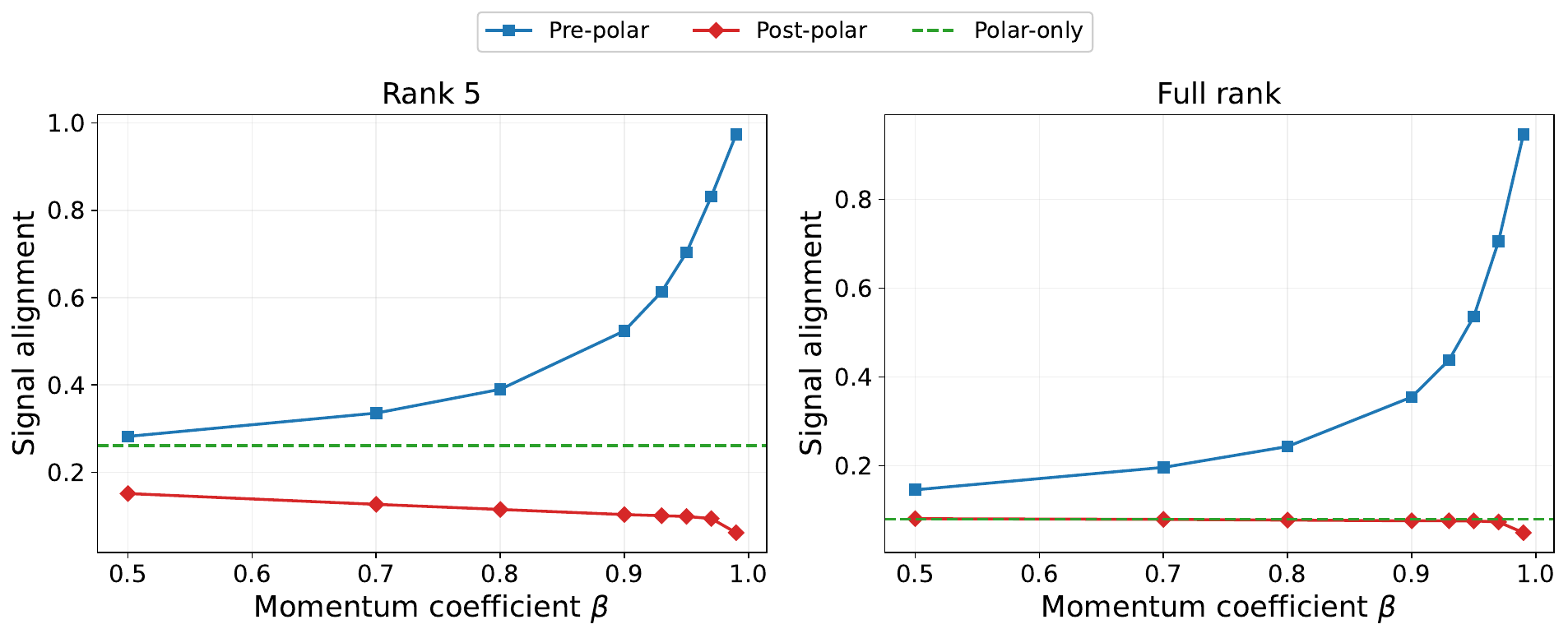}
  \caption{CIFAR-10 trajectory signal alignment versus $\beta$ at training step 1500, $K=100$, on \texttt{layer2.0.conv1}, at the rank-5 (left) and full-rank alignment (right) panels. Curve and reference conventions follow \cref{fig:frozen-thm2}.}
  \label{fig:app-cifar-online-ordering}
\end{figure}

\begin{figure}[htbp]
  \centering
  \includegraphics[width=0.5\linewidth]{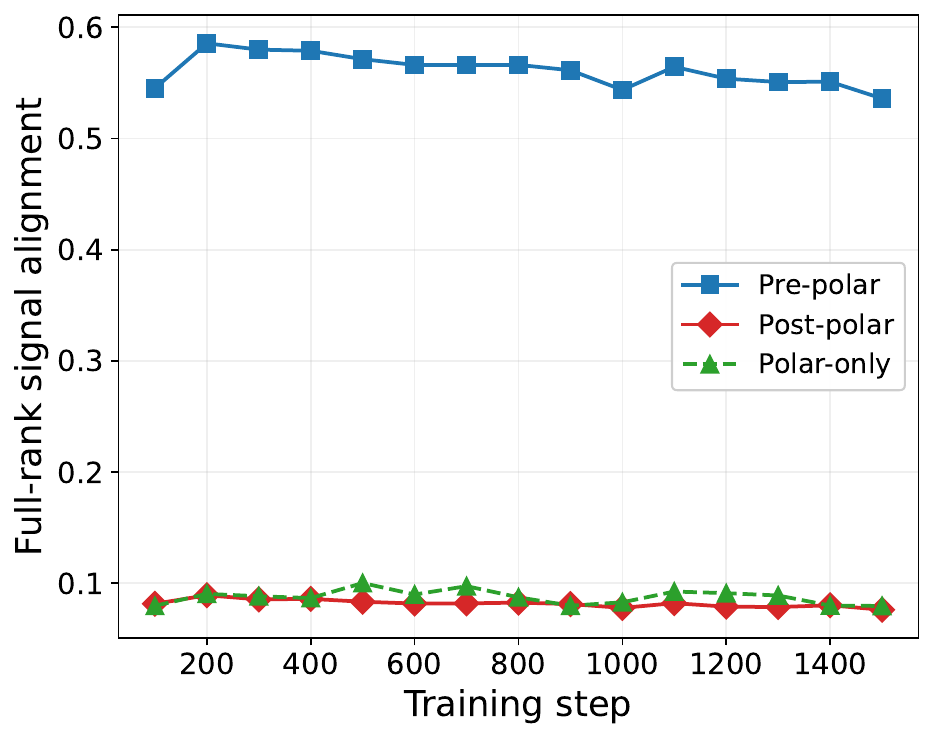}
  \caption{CIFAR-10 trajectory ordering history on \texttt{layer2.0.conv1} at $\beta=0.95$ (trajectory buffer $K=100$, analysis interval $I=100$, 15 checkpoints over training step 100--1500). Curve and reference conventions follow \cref{fig:frozen-thm2}.}
  \label{fig:app-cifar-online-history}
\end{figure}

\paragraph{NanoGPT Trajectory Probes Across Layers and Checkpoints.}
\Cref{fig:app-nanogpt-all-layer} extends \cref{fig:online-thm2-a} to all 24 projection targets of the GPT-2 backbone (12 attention + 12 MLP output projections) at the step-3000 checkpoint, three trajectory seeds at $K=50$, $\beta=0.95$. Every one of the 24 $\times$ 3 = 72 (layer, seed) triples satisfies Pre-polar $>$ Post-polar and Pre-polar $>$ Polar-only, with MLP output projections reaching a slightly larger absolute gap than attention output projections.
\Cref{tab:online-all-layer} reports the corresponding mean trajectory metrics aggregated over all 24 projection targets and three seeds at four common checkpoints. The Pre-polar vs. Post-polar full-rank signal alignment gap is essentially flat across the tracked checkpoints.
\Cref{fig:app-nanogpt-online-checkpoints} tracks the trajectory ordering across training on \texttt{h.0}, \texttt{h.5}, \texttt{h.11} at $K=50$, $\beta=0.95$, three-seed mean. Pre-polar $>$ Post-polar and Pre-polar $>$ Polar-only at every checkpoint and every layer. The mean trajectory gap is 0.528 on \texttt{h.0}, 0.563 on \texttt{h.5}, and 0.566 on \texttt{h.11}.

\begin{figure}[htbp]
  \centering
  \includegraphics[width=0.95\linewidth]{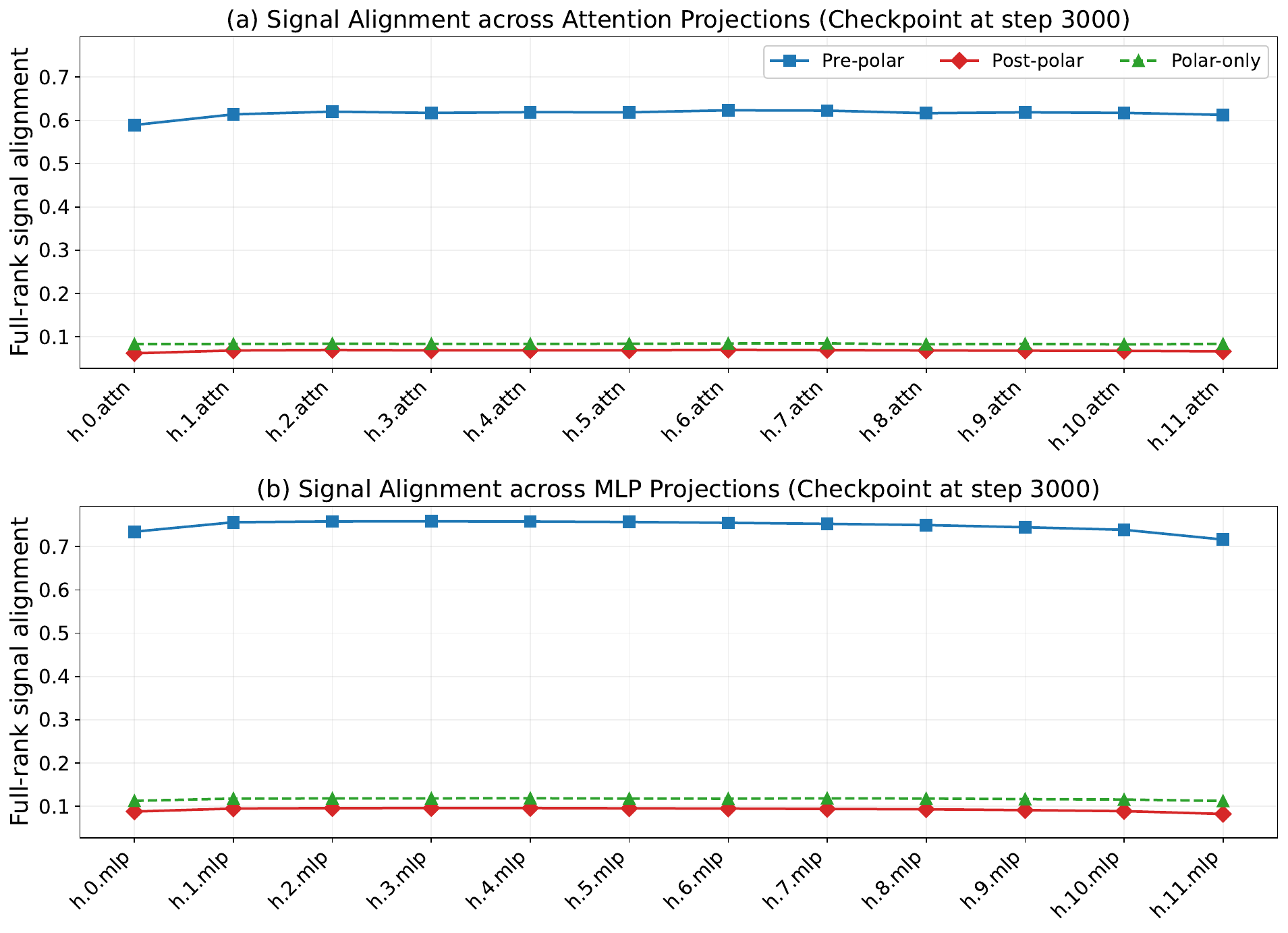}
  \caption{All-layer NanoGPT trajectory full-rank signal alignment across every attention and MLP output projection at training step 3000, $K=50$, $\beta=0.95$, aggregated over three seeds (1337, 1338, 1339). Curve and reference conventions follow \cref{fig:frozen-thm2}.}
  \label{fig:app-nanogpt-all-layer}
\end{figure}

\begin{table}[htbp]
    \centering\small
    \caption{Mean full-rank Pre-polar / Post-polar / Polar-only signal alignments against $\cO(\bar G)$ and their pairwise gaps, averaged over all 24 projection targets and three seeds (1337, 1338, 1339) at $\beta=0.95$, $K=50$, at four trajectory checkpoints.}
    \label{tab:online-all-layer}
    \begin{tabular}{@{}cccccc@{}}
        \toprule
        Step & Pre-polar align & Post-polar align & Polar-only align & Pre vs. Post & Pre vs. Polar-only \\
        \midrule
        100  & 0.780 & 0.105 & 0.125 & 0.675 & 0.655 \\
        1000 & 0.681 & 0.074 & 0.103 & 0.607 & 0.578 \\
        2000 & 0.683 & 0.078 & 0.094 & 0.606 & 0.589 \\
        3000 & 0.682 & 0.080 & 0.100 & 0.602 & 0.582 \\
        \bottomrule
    \end{tabular}
\end{table}

\begin{figure}[htbp]
  \centering
  \includegraphics[width=0.9\linewidth]{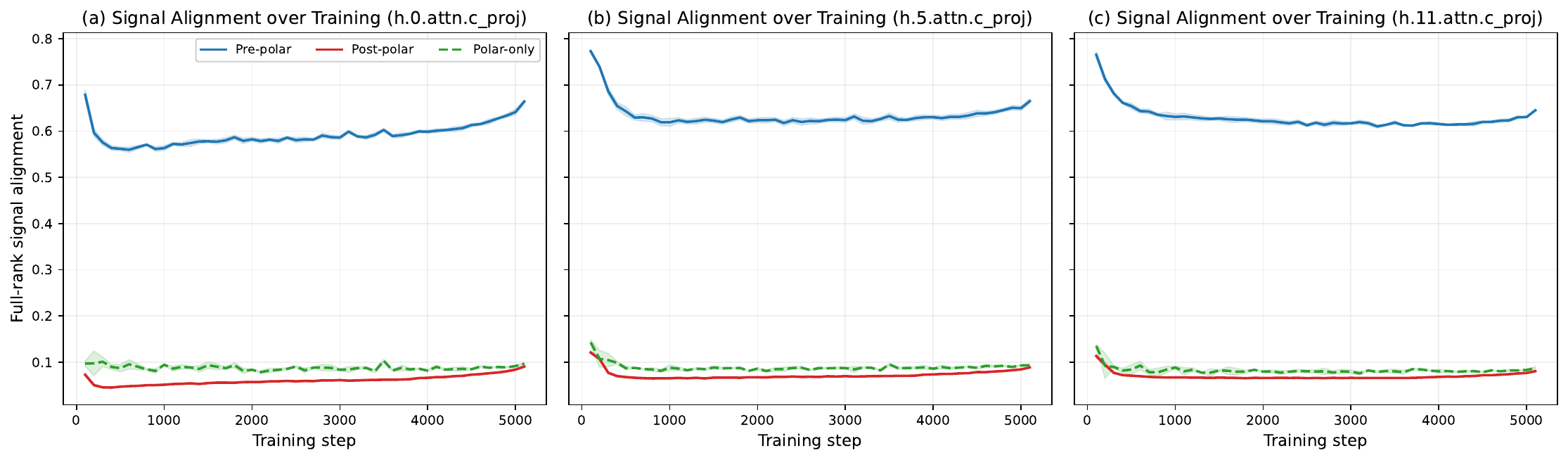}
  \caption{NanoGPT trajectory signal alignment over training on attention output projections \texttt{h.0}, \texttt{h.5}, and \texttt{h.11} at $K=\Kheadline$, $\beta=0.95$, three-seed mean with sample standard deviation bands across seeds. Curve and reference conventions follow \cref{fig:frozen-thm2}.}
  \label{fig:app-nanogpt-online-checkpoints}
\end{figure}


\section{Hyperparameter Sensitivity in Signal Strength}
\label{app:signal-strength}

This appendix sweeps the signal-to-noise ratio across three data sources at fixed momentum coefficient $\beta=0.95$ (full settings in \Cref{app:exp-tasks-signal-strength}). Synthetic $\lambda$ acts directly as the leading signal singular value. On CIFAR-10 and NanoGPT, we use the mini-batch size as a proxy for SNR: at a fixed checkpoint, the stochastic mini-batch gradient is an unbiased estimator of the full-batch gradient, and its sampling noise decreases as the mini-batch size increases. Thus, increasing the mini-batch size lowers the per-step gradient noise and raises the effective gradient SNR~\citep{smith2020generalization}. All panels follow the curve, marker, and reference conventions of \cref{fig:frozen-thm2}. Across all panels, the Pre-polar signal alignment dominates the signal alignments of both non-denoised pipelines: as SNR decreases (smaller $\lambda$ or smaller batch), the Post-polar and Polar-only signal alignments drop sharply while Pre-polar retains high signal alignment.

\paragraph{Synthetic $\lambda$ Sweep}
\Cref{fig:app-synth-signal-strength} extends \Cref{thm:recovery} on the rank-3 spiked model shared with \cref{fig:app-synth-ordering} to a sweep over the leading signal singular value $\lambda$. The Pre-polar signal alignment dominates the signal alignments of both non-denoised pipelines at every $\lambda$.

\begin{figure}[htbp]
  \centering
  \includegraphics[width=0.5\linewidth]{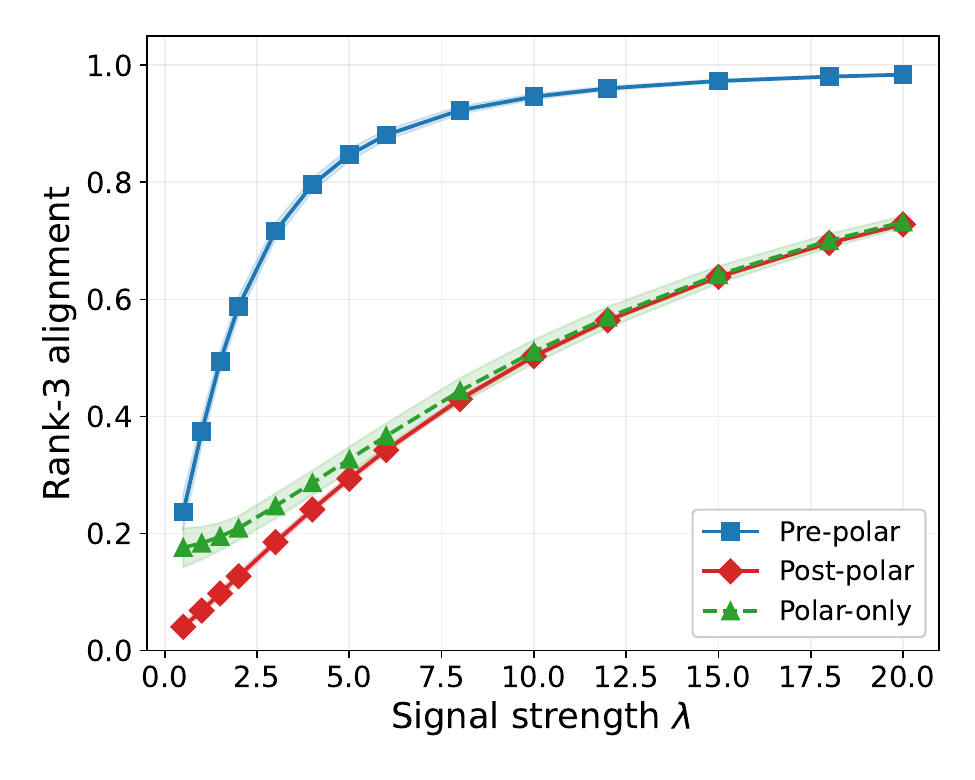}
  \caption{Synthetic rank-3 subspace alignment $\|U_3^\top A V_3\|_F / \sqrt{3}$ against the planted $(U_{\mathrm{true}},V_{\mathrm{true}})$ versus signal strength $\lambda$, on the rank-3 spiked model shared with \cref{fig:app-synth-ordering}. Pre-polar (blue squares, $\cO(M_K^{(\beta)})$), Post-polar (red diamonds, $\widetilde M_K^{(\beta)}$), Polar-only (green dashed, $\cO(G_K)$). Bands are the standard deviation across $10$ random-seed trials at $\beta=0.95$.}
  \label{fig:app-synth-signal-strength}
\end{figure}

\paragraph{CIFAR-10 Batch Sweep}
\Cref{fig:app-cifar-batch} extends \Cref{thm:recovery} to six independent CIFAR-10 stationary probes at \texttt{layer2.0.conv1} of a ResNet-18 warmed to step 500, sweeping the mini-batch size with $K=200$ collected gradients per probe. As the mini-batch size shrinks and the per-step gradient noise grows, Post-polar and Polar-only alignment drop sharply while Pre-polar retains a high alignment, because the momentum buffer averages out the per-step noise before the polar factor is applied. The Pre-polar signal alignment therefore dominates the signal alignments of both non-denoised pipelines in every panel and at every mini-batch size.

\begin{figure}[htbp]
  \centering
  \includegraphics[width=0.75\linewidth]{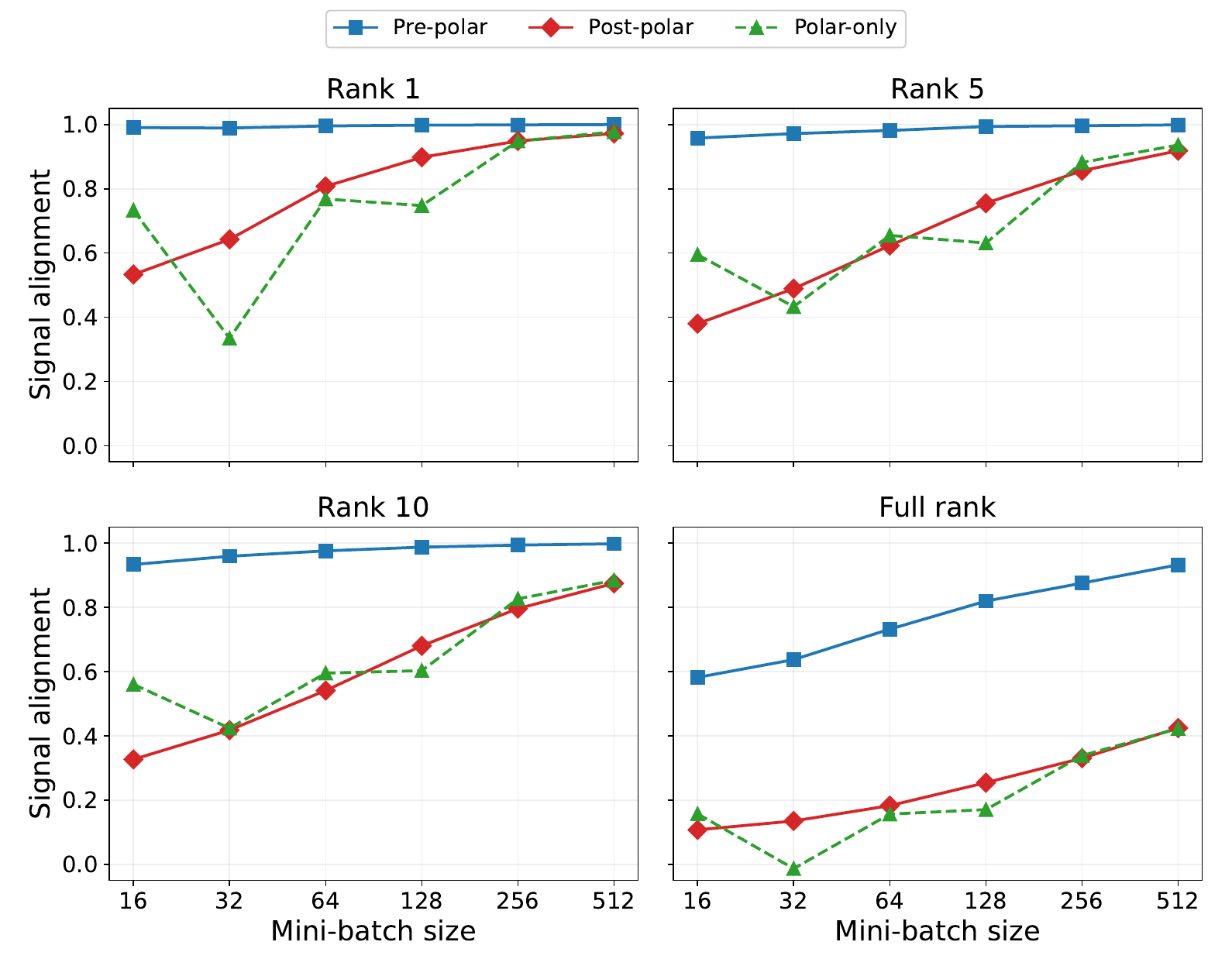}
  \caption{CIFAR-10 stationary signal alignment versus mini-batch size on \texttt{layer2.0.conv1} (128 $\times$ 576) of a ResNet-18, warmup step 500, $K=200$ per probe, $\beta=0.95$. Pre-polar (blue squares, $\cO(M_K^{(\beta)})$), Post-polar (red diamonds, $\widetilde M_K^{(\beta)}$), Polar-only (green dashed, $\cO(G_K)$). Panels are rank-1, rank-5, rank-10, and full-rank alignment against $\bar G$.}
  \label{fig:app-cifar-batch}
\end{figure}

\paragraph{NanoGPT Batch Sweep}
\Cref{fig:app-nanogpt-batch} extends \Cref{thm:recovery} to six independent NanoGPT stationary probes at \texttt{h.0.attn.c\_proj} (768 $\times$ 768) of the step-3000 checkpoint, sweeping the mini-batch size with $K=500$ collected gradients per probe experiment. The same phenomenon appears: as the mini-batch size shrinks and the per-step gradient noise grows, Post-polar and Polar-only alignment drop while Pre-polar retains a higher alignment because of the momentum filtering. The Pre-polar signal alignment dominates the signal alignments of both non-denoised pipelines in every panel and at every mini-batch size.

\begin{figure}[htbp]
  \centering
  \includegraphics[width=0.75\linewidth]{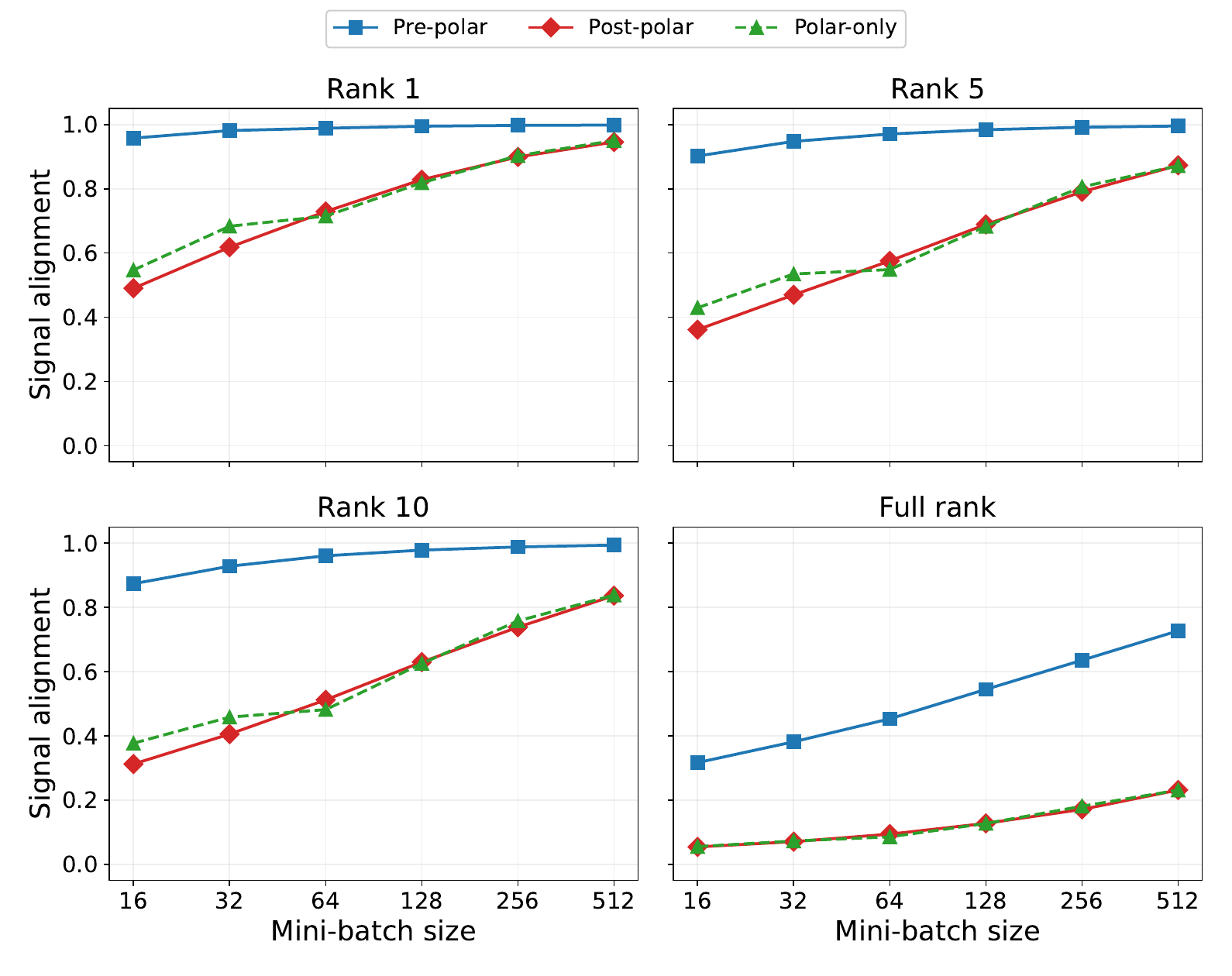}
  \caption{NanoGPT stationary signal alignment versus mini-batch size on \texttt{h.0.attn.c\_proj} (768 $\times$ 768) at the step-3000 checkpoint, $K=500$ per probe, $\beta=0.95$. Pre-polar (blue squares, $\cO(M_K^{(\beta)})$), Post-polar (red diamonds, $\widetilde M_K^{(\beta)}$), Polar-only (green dashed, $\cO(G_K)$). Panels are rank-1, rank-5, rank-10, and full-rank alignment against $\bar G$.}
  \label{fig:app-nanogpt-batch}
\end{figure}


\section{Hyperparameter Sensitivity in Buffer Size}
\label{app:k-sweep}

This appendix tests robustness of Pre-polar advantage to the trajectory buffer size. The two sweeps share $\beta=0.95$ ($T=20$) and vary $K$ from $2.5T$ to $10T$. The Pre-polar signal alignment dominance over the signal alignments of Post-polar and Polar-only holds at every $K$, and Pre-polar to Polar-only alignment ratio stays approximately constant at $6\times$ on both CIFAR-10 and NanoGPT. Absolute alignment values decrease with $K$ as $\bar G$ smooths over more training steps. The trajectory headline $K$ choices in CIFAR-10 and NanoGPT are robust within this range.

\paragraph{CIFAR-10 $K$-Sweep}
\Cref{tab:cifar-online-ksweep} reports four buffer sizes $K\in$ \{50,100,150,200\} at $\beta=0.95$ ($T=20$) on the CIFAR-10 trajectory probe of \texttt{layer2.0.conv1} at training step 1500. As $K$ grows from 50 to 200, the Pre-polar -- Post-polar full-rank signal alignment gap falls from 0.65 to 0.33, with positive Pre-polar advantage at every row. Pre-polar to Polar-only alignment ratio stays approximately constant at 6$\times$ across the sweep. The CIFAR-10 trajectory probe (\Cref{fig:app-cifar-online-history}) uses $K=100$.

\begin{table}[htbp]
    \centering\small
    \caption{CIFAR-10 trajectory $K$-sweep on \texttt{layer2.0.conv1} at training step 1500, $\beta=0.95$. Rank-5 columns are $\mathrm{Align}_5$ and full-align columns are $\mathrm{Align}_{\mathrm{full}}$, both against $\bar G$ (\Cref{app:measurements}). The Pre/Polar ratio is Pre-polar full-rank alignment divided by Polar-only full-rank alignment.}
    \label{tab:cifar-online-ksweep}
    \begin{tabular}{@{}cccccccc@{}}
        \toprule
        $K$ & Pre r5 & Post r5 & Gap r5 & Pre full-align & Post full-align & Polar-only full-align & Pre/Polar ratio \\
        \midrule
        50 & 0.8817 & 0.1354 & 0.7463 & 0.7479 & 0.1007 & 0.1192 & 6.27 \\
        100 & 0.6703 & 0.0947 & 0.5756 & 0.5420 & 0.0770 & 0.0827 & 6.55 \\
        150 & 0.5781 & 0.0855 & 0.4926 & 0.4452 & 0.0659 & 0.0676 & 6.59 \\
        200 & 0.5222 & 0.0732 & 0.4490 & 0.3837 & 0.0566 & 0.0629 & 6.10 \\
        \bottomrule
    \end{tabular}
\end{table}

\paragraph{NanoGPT $K$-Sweep}
\cref{tab:nanogpt-online-ksweep} reports four buffer sizes $K\in$ \{50,100,150,200\} at $\beta=0.95$ ($T=20$) on the NanoGPT trajectory probe of \texttt{h.0.attn.c\_proj}, single seed 1337. As $K$ grows from 50 to 200, the Pre-polar -- Post-polar full-rank signal alignment gap falls monotonically from 0.57 to 0.28, with positive Pre-polar advantage at every row. Pre-polar to Polar-only alignment ratio stays approximately constant at 6$\times$ across the sweep. The NanoGPT trajectory probe uses $K=\Kheadline$, matching the all-layer view of \Cref{app:thm2-supplement}.

\begin{table}[htbp]
    \centering\small
    \caption{NanoGPT trajectory $K$-sweep on \texttt{h.0.attn.c\_proj} at the final training step, single seed 1337, $\beta=0.95$. Settings beyond $K$ and seed match \Cref{app:exp-setup}. Column conventions follow \cref{tab:cifar-online-ksweep}.}
    \label{tab:nanogpt-online-ksweep}
    \begin{tabular}{@{}cccccccc@{}}
        \toprule
        $K$ & Pre r5 & Post r5 & Gap r5 & Pre full-align & Post full-align & Polar-only full-align & Pre/Polar ratio \\
        \midrule
        50 & 0.8586 & 0.1603 & 0.6983 & 0.6621 & 0.0911 & 0.1001 & 6.61 \\
        100 & 0.7037 & 0.1278 & 0.5758 & 0.4673 & 0.0705 & 0.0710 & 6.58 \\
        150 & 0.6382 & 0.1103 & 0.5279 & 0.3784 & 0.0591 & 0.0602 & 6.29 \\
        200 & 0.5901 & 0.1018 & 0.4883 & 0.3289 & 0.0524 & 0.0544 & 6.05 \\
        \bottomrule
    \end{tabular}
\end{table}



\end{document}